\documentclass[11pt, oneside]{article}
\usepackage{etoolbox}
\newcommand{\arxiv}[1]{\iftoggle{colt}{}{#1}}
\newcommand{\colt}[1]{\iftoggle{colt}{#1}{}}
\newtoggle{colt}
\global\togglefalse{colt}

\colt{\usepackage[subtle]{savetrees}}

\arxiv{\usepackage{geometry}
\geometry{letterpaper}}
\arxiv{
\usepackage{graphicx}	
\usepackage{amssymb}
\usepackage{amsfonts,latexsym,amsthm,amssymb,amsmath,amscd,euscript}
}
\colt{\usepackage{amsthm}}

\usepackage{bbm}
\usepackage{framed}
\arxiv{\usepackage{fullpage}}
\usepackage{mathrsfs}
\usepackage{hyperref}

\colt{
\usepackage{titlesec}
}

\hypersetup{
  colorlinks=true,
  linkcolor=blue,
  filecolor=blue,
  citecolor = black,      
  urlcolor=cyan,
}
\arxiv{\usepackage{accents}}
\usepackage{setspace}
\hypersetup{colorlinks=true,citecolor=blue,urlcolor=black,linkbordercolor={1 0 0}}
\usepackage{tikz-cd}
\newcount\Comments  %
\ifnum\Comments=1
\usepackage[inline]{showlabels}

\fi
\usepackage{turnstile}
\usepackage{mathpartir}
\usepackage{tikz}
\usepackage{xspace}

\usepackage{algorithm}
\usepackage{verbatim}
\usepackage[noend]{algpseudocode}

\newcommand{\nc}{\newcommand}

\usepackage[nameinlink,capitalize]{cleveref}
\crefformat{equation}{#2(#1)#3}
\Crefformat{equation}{#2(#1)#3}

\Crefformat{figure}{#2Figure #1#3}
\Crefname{assumption}{Assumption}{Assumptions}
\Crefformat{assumption}{#2Assumption #1#3}
\Crefname{subsubsection}{Section}{Sections}
\crefformat{subsubsection}{#2Section #1#3}
\Crefformat{subsubsection}{#2Section #1#3}

\nc{\comp}{\mathrm{c}}
\nc{\sups}[1]{^{\scriptscriptstyle{#1}}}
\nc{\subs}[1]{_{\scriptscriptstyle{#1}}}

\newcommand{\Kpairs}[1][K]{\MP_2({#1})}
\newcommand{\discpt}{\tau}
\newcommand{\discmg}{\gamma}
\newcommand{\disc}[1][\discmg]{\MD_{#1}}
\newcommand{\Menucls}[1][2]{\mathscr{M}_{#1}}
\newcommand{\hmenu}[1][h]{{#1}^{\mathsf{menu}}}
\newcommand{\hthr}[1][h]{{#1}^{\mathsf{thr}}}
\newcommand{\Thrcls}[1][\alpha]{\mathscr{T}_{#1}}
\newcommand{\merge}{\mathsf{merge}}

\newcommand{\wb}{\widebar}

\newcommand{\Wcomp}{\MW^{\mathrm{c}}}
\newcommand{\Hmc}[1][N,q]{\MH^{\mathsf{mc}}_{#1}}

\newcommand{\outdeg}{\mathsf{outdeg}}
\newcommand{\EstimatePotential}{\texttt{EstimatePotential}\xspace}

\newcommand{\WeakRealizable}{\texttt{WeakRealizable}\xspace}
\newcommand{\AgnosticPartial}{\texttt{AgnosticPartial}\xspace}
\newcommand{\RealizablePartial}{\texttt{RealizablePartial}\xspace}

\newcommand{\Adaboost}{\texttt{Adaboost}\xspace}
\newcommand{\SampleERMBinary}{\texttt{SampleERM.Binary}\xspace}
\newcommand{\SampleERMReal}{\texttt{SampleERM.Real}\xspace}
\newcommand{\SampleConReal}{\texttt{SampleCon.Real}\xspace}
\newcommand{\MulticlassRealizable}{\texttt{MulticlassRealizable}\xspace}
\newcommand{\MulticlassAgnostic}{\texttt{MulticlassAgnostic}\xspace}
\newcommand{\RegRealizable}{\texttt{RegRealizable}\xspace}
\newcommand{\RegAgnostic}{\texttt{RegAgnostic}\xspace}

\newcommand{\Alg}{\mathsf{Alg}}
\newcommand{\AlgR}{\mathsf{Alg}^{\mathsf{R}}}
\newcommand{\AlgA}{\mathsf{Alg}^{\mathsf{A}}}

\newcommand{\True}{\mathsf{True}}
\newcommand{\False}{\mathsf{False}}

\nc{\Actor}{\texttt{Actor}\xspace}
\nc{\EstFeature}{\texttt{EstFeature}\xspace}
\nc{\ExpFTPL}{\texttt{ExpFTPL}\xspace}
\nc{\dist}{\mathrm{dist}}
\nc{\Bquad}{B^{\mathsf{quad}}}
\nc{\thetar}{\theta^{\mathsf{r}}}
\nc{\sign}{\mathrm{sign}}
\nc{\eerr}{\widehat{\mathrm{er}}}
\nc{\err}{\mathrm{er}}
\nc{\Ber}{\mathrm{Ber}}
\nc{\SU}{\mathscr{U}}
\nc{\Decmc}[1][]{\mathsf{Dec}_{#1}^{\mathsf{mc}}}

\arxiv{
\makeatletter
\newtheorem*{rep@theorem}{\rep@title}
\newcommand{\newreptheorem}[2]{%
\newenvironment{rep#1}[1]{%
 \def\rep@title{#2 \ref{##1}}%
 \begin{rep@theorem}}%
 {\end{rep@theorem}}}
\makeatother
}

\makeatletter
\newcommand\xlabel[2][]{\phantomsection\def\@currentlabelname{#1}\label{#2}}
\makeatother

{
\theoremstyle{plain}
\newtheorem{theorem}{Theorem}
\newtheorem{lemma}[theorem]{Lemma}
\newtheorem{corollary}[theorem]{Corollary}

\newtheorem{proposition}[theorem]{Proposition}

\theoremstyle{definition}
\newtheorem{definition}{Definition}

\numberwithin{theorem}{section}
\numberwithin{definition}{section}
}

\nc{\DMO}{\DeclareMathOperator}

\DeclareMathOperator*{\argmin}{arg\,min} %

\DMO{\prox}{prox}
\DMO{\UCB}{UCB}
\DMO{\LCB}{LCB}
\nc{\phidiff}{\phi\sups{\Delta}}
\nc{\pexp}{q_{\mathrm{exp}}}
\nc{\nn}{\nonumber}
\nc{\rk}{\mathrm{rk}}
\nc{\brk}[3]{{\rm br}_{#1}^{#2}({#3})}
\nc{\co}{{\rm co}}
\nc{\br}[2]{{\rm br}^{#1}({#2})}
\nc{\depth}[1]{{\rm d}({#1})}
\nc{\tA}{\textsc{A}}
\nc{\child}[2]{{\rm ch}_{#1}({#2})}
\nc{\parent}[1]{{\rm pa}({#1})}
\nc{\dg}{\dagger}
\nc{\bB}{\mathbf{B}}
\nc{\be}{\mathbf{e}}
\nc{\Span}{{\rm Span}}
\nc{\unif}{\mathsf{unif}}
\nc{\indsig}[2]{\mathcal{I}_{#1}({#2})}
\nc{\total}{{\rm fin}}
\nc{\early}{{\rm pre}}
\nc{\zsink}{z_{\rm sink}}
\nc{\lowv}{{\rm low}}
\nc{\ol}{\overline}
\nc{\ul}{\underline}
\nc{\madec}[3]{\texttt{ma-dec}_{#1}({#2}, {#3})}
\nc{\madeco}[1]{\texttt{ma-dec}_{#1}}
\nc{\madecd}[3]{\texttt{ma-dec}^{\texttt{d}}_{#1}({#2}, {#3})}
\nc{\SF}{\mathscr{F}}
\nc{\SP}{\mathscr{P}}
\nc{\SPc}{\wb{\mathscr{P}}}
\nc{\SB}{\mathscr{B}}
\nc{\SC}{\mathscr{C}}
\nc{\BS}{\mathbb{S}}
\nc{\PiMarkov}{\Pi^{\mathsf{M}}}
\nc{\PiM}{\PiMarkov}
\nc{\piopt}{\pi^{\mathsf{opt}}}
\nc{\Critic}{\texttt{Critic}\xspace}
\nc{\trunc}[2]{\mathsf{trunc}_{#2}({#1})}
\nc{\sbl}{of strong Bellman type\xspace}
\nc{\inormal}[1][\Phi, u,v]{\til{N}_{{#1}}}

\nc{\gamvec}{\gamma}
\nc{\til}{\widetilde}
\nc{\td}{\tilde}
\nc{\wh}{\widehat}
\arxiv{\nc{\todo}[1]{\ifnum\Comments=1 {\color{red}  [TODO: #1]}\fi}}
\nc{\old}[1]{\ifnum\Comments=1 {\color{brown}  [OLD: #1]}\fi}
\nc{\noah}[1]{\ifnum\Comments=1 {\color{purple} [ng: #1]}\fi}
\nc{\dhruv}[1]{\ifnum\Comments=1 {\color{magenta} [dr: #1]}\fi}
\nc{\BP}{\mathbb{P}}
\nc{\BI}{\mathbb{I}}
\nc{\midpoint}[1][\Phi,\phi_1,\phi_2]{\mu^{\star}_{{#1}}}

\nc{\fools}[3]{\MF_{#3}({#1}, {#2})}
\nc{\fool}[2]{\MF({#1},{#2})}
\nc{\clip}[2]{{\rm clip}\left[ \left. {#1} \right| {#2} \right]}
\nc{\imax}{\omega}
\DMO{\conv}{conv}
\nc{\MH}{\mathcal{H}}
\nc{\MV}{\mathcal{V}}
\nc{\MC}{\mathcal{C}}
\nc{\MI}{\mathcal{I}}
\nc{\st}{\star}
\nc{\lng}{\langle}
\nc{\rng}{\rangle}
\DMO{\OOPT}{opt}
\nc{\dopt}[2]{\ell_{\OOPT}({#1},{#2})}
\nc{\grad}{\nabla}
\nc{\MG}{\mathcal{G}}
\nc{\MP}{\mathcal{P}}
\nc{\PP}{\mathbb{P}}
\nc{\TT}{\mathbb{T}}
\nc{\TTmax}{\TT_{\max}}
\DMO{\REG}{Reg}
\DMO{\WREG}{wReg}
\nc{\reg}[2]{{\Delta}_{{#1}}({#2})}
\nc{\wreg}[2]{{\Delta}^{\rm w}_{{#1}}({#2})}
\nc{\Reg}[2]{{\REG}_{{#1}}({#2})}
\nc{\wReg}[2]{{\WREG}_{{#1}}({#2})}
\DMO{\Ham}{Ham}
\DMO{\Gap}{Gap}
\DMO{\GD}{GD}
\DMO{\GDA}{GDA}
\DMO{\EG}{EG}
\nc{\TE}{\til{\E}}
\nc{\Var}{\mathbb{V}}
\DMO{\Cov}{Cov}
\DMO{\OGDA}{OGDA}
\DMO{\Unif}{Unif}
\DMO{\Tr}{Tr}
\nc{\Qu}{\ul{Q}}
\nc{\Qo}{\ol{Q}}
\nc{\Ro}{\ol{R}}
\nc{\Vu}{\ul{V}}
\nc{\Vo}{\ol{V}}
\nc{\RanQ}{\Delta Q}
\nc{\RanV}{\Delta V}
\nc{\clipQ}{\Delta \breve{Q}}
\nc{\frzQ}{\Delta \mathring{Q}}
\nc{\clipV}{\Delta \breve{V}}
\nc{\clipdelta}{\breve{\delta}}
\nc{\cliptheta}{\breve{\theta}}
\nc{\delmin}{\Delta_{{\rm min}}}
\nc{\delmins}[1]{\Delta_{{\rm min},{#1}}}
\nc{\gapfinal}[1]{\max \left\{ \frac{\frzQ_{{#1}}^{k^\st}(x,a)}{2H}, \frac{\delmin}{4H} \right\}}
\nc{\post}[2]{R({#1}; {#2})}
\nc{\posts}[3]{R_{#3}({#1}; {#2})}
\nc{\VCdim}{d_{\mathsf{VC}}}
\nc{\Natdim}{d_{\mathsf{N}}}
\nc{\fatdim}[1][\discmg]{d_{\mathsf{fat},{#1}}}
\nc{\DSdim}{d_{\mathsf{DS}}}
\nc{\OIGdim}[1][\gamma]{d_{\mathsf{OIG},{#1}}}

\nc{\algnst}[1]{\begin{align*}#1\end{align*}}
\nc{\algn}[1]{\begin{align}#1\end{align}}
\nc{\matx}[1]{\left(\begin{matrix}#1\end{matrix}\right)}
\renewcommand{\^}[1]{^{(#1)}}

\nc{\nuu}{\nu}

\nc{\bel}[1]{\mathbf{b}({#1})}
\nc{\nbel}[1]{\bar{\mathbf{b}}({#1})}
\nc{\sbel}[2]{\mathbf{b}'_{#1}({#2})}
\nc{\nsbel}[2]{\bar{\mathbf{b}}'_{#1}({#2})}

\nc{\bv}{\mathbf{v}}
\nc{\bone}{\mathbf{1}}
\nc{\bX}{\mathbf{X}}
\nc{\bY}{\mathbf{Y}}
\nc{\bG}{\mathbf{G}}
\nc{\bz}{\mathbf{z}}
\nc{\bw}{\mathbf{w}}
\nc{\bA}{\mathbf{A}}
\nc{\bJ}{\mathbf{J}}
\nc{\bK}{\mathbf{K}}
\nc{\bb}{\mathbf{b}}
\nc{\ba}{\mathbf{a}}
\nc{\bc}{\mathbf{c}}
\nc{\bC}{\mathbf{C}}
\nc{\BR}{\mathbb R}
\nc{\BA}{\mathbb{A}}
\nc{\BC}{\mathbb C}
\nc{\bx}{\mathbf{x}}
\nc{\bS}{\mathbf{S}}
\nc{\bM}{\mathbf{M}}
\nc{\bR}{\mathbf{R}}
\nc{\bN}{\mathbf{N}}
\nc{\NN}{\mathbb{N}}
\nc{\by}{\mathbf{y}}
\nc{\sy}{y}
\nc{\sx}{x}

\nc{\MO}{\mathcal O}
\nc{\MU}{\mathcal{U}}
\nc{\ME}{\mathcal{E}}
\nc{\MN}{\mathcal{N}}
\nc{\MK}{\mathcal{K}}
\nc{\MM}{\mathcal{M}}
\nc{\MS}{\mathcal{S}}
\nc{\MT}{\mathcal{T}}
\nc{\BF}{\mathbb F}
\nc{\BQ}{\mathbb Q}
\nc{\MX}{\mathcal{X}}
\nc{\MA}{\mathcal{A}}
\nc{\MD}{\mathcal{D}}
\nc{\MB}{\mathcal{B}}
\nc{\MZ}{\mathcal{Z}}
\nc{\MJ}{\mathcal{J}}
\nc{\MW}{\mathcal{W}}
\nc{\MR}{\mathcal{R}}
\nc{\MY}{\mathcal{Y}}
\nc{\BZ}{\mathbb Z}
\nc{\BN}{\mathbb N}
\nc{\ep}{\epsilon}
\nc{\epbe}{\varepsilon_{\mathsf{BE}}}
\nc{\epout}{\varepsilon_{\mathsf{outlier}}}
\nc{\bellc}[1][h]{\MT_{#1}^\circ}
\nc{\vep}{\varepsilon}
\nc{\gapfn}[1]{\varepsilon_{#1}}
\nc{\ggapfn}[2]{\varphi_{#1}({#2})}
\nc{\epsahk}{\gapfn{0}}
\nc{\BH}{\mathbb H}
\nc{\BG}{\mathbb{G}}
\nc{\D}{\Delta}
\nc{\MF}{\mathcal{F}}
\nc{\One}[1]{\mathbbm{1}\{{#1}\}}
\nc{\bOne}{\mathbf{1}}
\nc{\Aopt}{\mathcal{A}^{\rm opt}}
\nc{\Amul}{\mathcal{A}^{\rm mul}}

\nc{\SQ}{\mathsf Q}

\nc{\DO}{\accentset{\circ}{\D}}
\nc{\mf}{\mathfrak}
\nc{\mfp}{\mathfrak{p}}
\nc{\mfq}{\mf{q}}
\nc{\Sp}{\mbox{Spec}}
\nc{\Spm}{\mbox{Specm}}
\nc{\hookuparrow}{\mathrel{\rotatebox[origin=c]{90}{$\hookrightarrow$}}}
\nc{\hookdownarrow}{\mathrel{\rotatebox[origin=c]{-90}{$\hookrightarrow$}}}
\nc{\hra}{\hookrightarrow}
\nc{\tra}{\twoheadrightarrow}
\nc{\sgn}{{\rm sgn}}
\nc{\aut}{{\rm Aut}}
\nc{\Hom}{{\rm Hom}}
\nc{\img}{{\rm Im}}
\DMO{\id}{Id}
\DMO{\supp}{supp}
\DMO{\KL}{KL}
\nc{\kld}[2]{\KL({#1}||{#2})}
\nc{\ren}[2]{D_2({#1}||{#2})}
\nc{\chisq}[2]{\chi^2({#1}||{#2})}
\nc{\tvd}[2]{D_{\mathsf{TV}}({#1}, {#2})}
\nc{\hell}[2]{D_{\mathsf{H}}^2({#1}, {#2})}
\nc{\dbi}[3][\pi]{D_{\mathsf{bi}}^{#1}({#2} \| {#3})}
\DMO{\BSS}{BSS}
\DMO{\BES}{BES}
\DMO{\BGS}{BGS}
\DMO{\poly}{poly}
\nc{\indep}{\perp}
\DMO{\sink}{sink}
\nc{\fp}[1]{\MP_1({#1})}
\nc{\BO}{\mathbb{O}}
\nc{\BT}{\mathbb{T}}

\nc{\RR}{\mathbb{R}}
\nc{\Gradient}{\nabla}
\DMO{\diag}{diag}
\nc{\norm}[1]{\left \lVert #1 \right \rVert}
\nc{\EE}{\mathop{\mathbb{E}}}
\nc{\MQ}{\mathcal{Q}}
\nc{\ML}{\mathcal{L}}
\nc{\cPhi}{\bar \Phi}
\nc{\SA}{\mathscr{A}}

\DMO{\PR}{Pr}
\renewcommand{\Pr}{\PR}
\nc{\E}{\mathbb{E}}
\nc{\ra}{\rightarrow}

\nc{\hc}{\{0,1\}^n}
\nc{\pmhc}[1]{\{-1,1\}^{#1}}
\nc{\Dbnd}{D}
\nc{\Bbnd}{B}
\nc{\brew}{b_{\mathsf{rew}}}
\nc{\binit}{b_{\mathsf{init}}}
\nc{\Cvg}{\mathscr{C}}
\nc{\Ocon}{\mathcal{O}^{\mathsf{con,w}}}
\nc{\Oconmenu}{\mathcal{O}^{\mathsf{con,menu}}}
\nc{\Oerm}{\mathcal{O}^{\mathsf{erm,w}}}
\nc{\Oerms}{\mathcal{O}^{\mathsf{erm,s}}}
\nc{\Oconst}{\mathcal{O}^{\mathsf{con,s}}}
\nc{\Orange}{\mathcal{O}^{\mathsf{range}}}
\nc{\lbin}{\ell^{\mathsf{bin}}}
\nc{\lmc}{\ell^{\mathsf{mc}}}
\nc{\labs}{\ell^{\mathsf{abs}}}

\nc{\factoring}{\textsf{FACTORING}\xspace}
\nc{\Hprime}{\MH^{\mathsf{primes}}}

\colt{\renewcommand{\cite}{\citep}}

\colt{
  \title[PAC Learning with Weak Oracles]{Is Efficient PAC Learning Possible with an \\ Oracle That Responds ``Yes'' or ``No''?}
  \usepackage{times}
  \coltauthor{
    \Name{Constantinos Daskalakis} \Email{costis@csail.mit.edu}\\ \addr MIT CSAIL \& Archimedes AI \\
    \Name{Noah Golowich} \Email{nzg@mit.edu} \\ \addr MIT CSAIL
    }
  }

\arxiv{
\title{Is Efficient PAC Learning Possible with an \\ Oracle That Responds ``Yes'' or ``No''?}
\author{Constantinos Daskalakis \\ {\small MIT CSAIL \& Archimedes AI} \\ \url{costis@csail.mit.edu} \and Noah Golowich \\ {\small MIT CSAIL} \\ \url{nzg@mit.edu}}
\date{June 17, 2024}
}

\begin{document}
\maketitle

\begin{abstract}
  The \emph{empirical risk minimization (ERM)} principle has been highly impactful in machine learning, leading both to near-optimal theoretical guarantees for ERM-based learning algorithms as well as driving many of the recent empirical successes in deep learning. In this paper, we investigate  the question of whether the ability to perform ERM, which computes a hypothesis minimizing empirical risk on a given dataset, is necessary for efficient learning: in particular, is there a weaker oracle than ERM  which can nevertheless enable learnability? We answer this question affirmatively, showing that in the realizable setting of PAC learning for binary classification, a concept class can be learned using an oracle which only returns a \emph{single bit} indicating whether a given dataset is realizable by some concept in the class. The sample complexity and oracle complexity of our algorithm depend polynomially on the VC dimension of the hypothesis class, thus showing that there is only a polynomial price to pay for use of our weaker oracle. Our results extend to the agnostic learning setting with a slight strengthening of the oracle, as well as to the partial concept, multiclass and real-valued learning settings. In the setting of partial concept classes, prior to our work no oracle-efficient algorithms were known, even with a standard ERM oracle. Thus, our results address a question of \cite{alon2021theory} who asked whether there are algorithmic principles which enable efficient learnability in this setting. 
\end{abstract}

\colt{
  }

\section{Introduction}
Many of the successful techniques in modern machine learning proceed by specifying a large function class $\MH$, such as a class of neural networks, and optimizing over $\MH$ to find a minimizer of a loss function on a finite dataset. %
This approach, known as \emph{empirical risk minimization} (ERM), has long been known to lead to near-optimal PAC learning guarantees in fundamental settings such as binary classification and regression \cite{vapnik1968algorithms,vapnik1974theory,blumer1989learnability,bartlett1998prediction,alon1997scale}. Due to the ability of heuristics such as gradient descent to approximately implement ERM for neural network function classes, the ERM principle also lies behind numerous empirical successes in supervised learning \cite{krizhevsky2012imagenet}. Inspired by these successes, various works have also investigated to what extent an \emph{oracle} which can implement ERM for a given function class is useful for learning problems beyond the PAC setting, including online learning \cite{block2022smoothed,haghtalab2022oracle,assos2023online}, bandits \cite{simchi2022bypassing}, and reinforcement learning \cite{agarwal2020flambe,mhammedi2023representation}.

In this paper, we return to the basics and ask: is ERM necessary? Or can we efficiently perform learning tasks with a \emph{weaker} oracle than an ERM oracle? For the foundational problem of realizable PAC learning, it is known that a \emph{consistency oracle}, which returns a hypothesis $\hat h$ in the class $\MH$ which is consistent with a given dataset (and fails if one does not exist), is still sufficient for efficient learning. Perhaps the most drastic way to further weaken such a consistency oracle is as follows: suppose that the oracle does not return $\hat h$, and only returns a \emph{single bit} indicating whether such a $\hat h\in \MH$ exists which fits the data. We refer to such an oracle as a \emph{weak consistency oracle}. Is there a PAC learning algorithm which learns efficiently with respect to this oracle? Perhaps surprisingly, we find that the answer is \emph{yes}. Moreover, this positive answer extends to the agnostic PAC setting, if we slightly generalize the weak consistency oracle to return the \emph{value} of the empirical risk minimizer $\hat h$ on the dataset, but not $\hat h$ itself; we call such an oracle a \emph{weak ERM oracle}. 

\paragraph{Motivation.} We discuss several possible sources of motivation behind a weakening of an ERM (or consistency) oracle. First, note that a weak consistency oracle, which only returns a single bit indicating whether a consistent hypothesis exists, corresponds to a \emph{decision problem} on $\MH$, whereas a standard consistency oracle corresponds to a \emph{search problem}. For many natural classes of problems (e.g., see \cite{agrawal2004primes,reith2003optimal,khuller1991planar}), the decision variant is known to be computationally cheaper than the search variant.\footnote{More generally, such separations can emerge for non-self-reducible problems in NP.} Based off of such a separation, in \cref{prop:hprimes}, we provide a concrete example of a class for which implementing a weak consistency oracle is possible in polynomial time but implementing a standard consistency oracle is not, under standard computational assumptions. Thus, in such a case, our approach, via a weak consistency oracle, will lead to improved computational guarantees over the standard approach which calls a consistency oracle.

From a more theoretical perspective, the use of weak consistency and weak ERM oracles yields PAC learning bounds that do not rely on uniform convergence, in contrast to some prior analyses of ERM. A notable setting where PAC learning is known to be statistically feasible but uniform convergence fails is that of learning with \emph{partial concept classes} \cite{alon2021theory,long2001on}, which in turn has numerous applications including to regression \cite{long2001on,bartlett1998prediction}, learning with fairness constraints such as multicalibration \cite{hu2023comparative}, adversarially robust learning \cite{attias2022characterization}, and others (see \cref{sec:extensions}). Our results provide the first (ERM) oracle-efficient learning algorithm for partial concept classes, which addresses a question asked in \cite{alon2021theory}. We emphasize that even with a \emph{standard} ERM oracle, no efficient algorithm was known, whereas our guarantees for partial concept classes hold with a \emph{weak} ERM oracle. 

\subsection{Overview of results}
First, we consider the setting of PAC learning of \emph{partial concept classes} \cite{alon2021theory,long2001on}, which are classes $\MH \subset \{0,1,* \}^\MX$ for some domain space $\MX$. Hypotheses $h \in \MH$ should be thought of as undefined at points $x \in \MX$ for which $h(x) = *$ (see \cref{sec:prelim} for a formal definition).\footnote{The unfamiliar reader can simply consider the special case of those $\MH$ whose hypotheses never take the value $*$, which corresponds to standard binary classification.} Our main results for this setting are as follows:
\begin{itemize}
\item In the realizable setting of PAC learning, any partial concept class $\MH$ of VC dimension at most $\VCdim$ can be learned by an algorithm that makes polynomially many calls to a \emph{weak consistency oracle} $\Ocon$, which receives as input a dataset $S = \{ (x_i, y_i) \}_{i \in [n]}$ and returns $\True$ if there is $h \in \MH$ satisfying $h(x_i) = y_i \in \{0,1\}$ for all $i$, and $\False$ otherwise. The sample complexity scales as $\tilde O(\VCdim^3)$ (see first part of \cref{thm:partial-main}). 
\item In the agnostic setting of PAC learning, the same guarantee holds, except with respect to a \emph{weak ERM oracle} $\Oerm$ which receives as input $S = \{ (x_i, y_i) \}_{i \in [n]}$ and returns the value $\min_{h \in \MH} \frac 1n \sum_{i=1}^n \One{h(x_i) \neq y_i \vee h(x_i) = * } \in \{0, 1/n, \ldots, 1\}$ (see second part of \cref{thm:partial-main}).
\end{itemize}
We proceed to generalize our upper bounds beyond the binary setting, namely the multiclass and real-valued (regression) settings:
\begin{itemize}
\item For $K \in \BN$, any multiclass concept class $\MH \subset [K]^\MX$ of Natarajan dimension at most $\Natdim$ can be PAC learned by an algorithm that uses $n = \tilde O(\Natdim^3 \log^4 K)$ samples and makes $K \cdot \poly(n)$ calls to a weak consistency oracle (or to a weak ERM oracle in the agnostic setting); see \cref{thm:multiclass-main}. 
\item Any real-valued class $\MH \subset [0,1]^\MX$ whose fat-shattering dimension at scales $\discmg\in (0,1)$ is at most $\fatdim$ can be agnostically PAC learned by an algorithm that uses $n$ samples as long as $n \gtrsim \fatdim^3$ for an appropriate chocie of $\discmg$, using $\poly(n)$ calls to a weak ERM oracle. Moreover, a similar guarantee holds for the realizable setting; see \cref{thm:reg-main}. 
\end{itemize}
In the setting of partial concept classes as well as the \emph{agnostic} setting of regression, VC dimension and fat-shattering dimension, respectively, are known to characterize learnability. Thus, our results above  show that there is at most a polynomial price to pay in terms of sample complexity if we require efficiency with respect to a weak ERM oracle. In contrast, in the multiclass and realizable regression settings, the optimal sample complexity is characterized by the \emph{Daniely-Schwartz dimension} \cite{daniely2014optimal,brukhim2022characterization} and the \emph{one-inclusion graph dimension} \cite{attias2023optimal}, respectively. These quantities can be arbitrarily smaller than our corresponding sample complexities above, namely $\Natdim \log K$ and $\fatdim$, respectively.

Can our bounds for the multiclass and realizable regression settings be improved to get near-optimal sample complexity while retaining oracle efficiency? Our final results show a negative answer to this question, even when the algorithm is given a standard ERM oracle: %
\begin{itemize}
\item Multiclass concept classes of Daniely-Schwartz dimension 1 are not PAC learnable with any finite number of ERM oracle queries; see \cref{thm:mc-lb}.
\item In the realizable setting of regression, real-valued classes with one-inclusion graph dimension 1 are not PAC learnable with any finite number of ERM oracle queries; see \cref{thm:reg-lb}. 
\end{itemize}

\paragraph{Techniques.} Our results rest on a technique to efficiently implement a randomized variant of the \emph{one-inclusion graph algorithm}, formalized in \cref{thm:weak-oig} (see also \cref{def:oig}). In particular, we show first that we can obtain a weak learner by constructing a random orientation of the one-inclusion graph with bounded out-degree, as follows. For each edge we wish to orient, we take a random walk starting from each of its endpoints and inspect the distribution of hitting times of the complement of the one-inclusion graph. The vertex whose random walk reaches the complement of the one-inclusion graph sooner should have the edge directed away from it. We then use standard boosting techniques to improve the weak learner to a strong learner. A detailed proof overview may be found in \cref{sec:weak-learner}.

\paragraph{Related work: learning by refuting.} Our work is closely related to \cite{vadhan2017learning,kothari2018improper}, which relate the PAC learnability of binary hypothesis classes with finite VC dimension to the existence of \emph{refutation algorithms} for them. In particular, by noting that a weak consistency oracle for $\MH$ RRHS-refutes $\MH$ in the sense of \cite[Definition 4]{vadhan2017learning}, Theorem 5 of \cite{vadhan2017learning} recovers the first part of \cref{thm:partial-main} in the case of a \emph{total} binary hypothesis class $\MH$. In a similar manner, by noting that a weak ERM oracle for $\MH$ allows one to implement a $\delta$-refutation algorithm for $\MH$ in the sense of \cite[Definition 2.1]{kothari2018improper}, Lemma 3.2 of \cite{kothari2018improper} establishes a quantitatively weaker version of the second part of \cref{thm:partial-main} in the case of a \emph{total} binary hypothesis class $\MH$. In particular, uniform convergence implies that, for any $\delta > 0$, a $\delta$-refutation algorithm for $\MH$ with sample complexity $m_\delta := \tilde O(\VCdim(\MH)/\delta^2)$ may be implemented using a single call to a weak ERM oracle, and using this bound with $\delta = \ep$ in \cite[Lemma 2.1]{kothari2018improper} (which gives sample complexity $O(m_\delta^3/\ep^2)$ to obtain error $\ep + \delta$) yields overall sample complexity of $\tilde O(\VCdim(\MH)^3/\ep^8)$. In contrast, \cref{thm:partial-main} obtains the better bound of $\tilde O(\VCdim(\MH)^3/\ep^2)$, which has near-optimal dependence on $\ep$.

We also remark that our techniques, which interpret the construction of a weak learner as a random walk on the one-inclusion graph of the hypothesis class, differ from \cite{vadhan2017learning,kothari2018improper}, which did not have such an interpretation.

\paragraph{Open questions.} Taken together, our results represent a comprehensive treatment of the weak oracle efficiency of PAC learning in the standard settings of partial, multiclass, and real-valued learning. One intruiging question that remains is closing the gap between our $\tilde O(\VCdim^3)$ sample complexity in the binary setting and the fact that only $O(\VCdim)$ samples are required when one is allowed access to a standard ERM oracle. In particular, is there a (polynomial-sized) cost in sample complexity to pay for using a weak oracle? Analogous questions can be asked in the multiclass and real-valued settings. Along different lines, it would be interesting to investigate the use of weaker notions of ERM oracles in more complex learning situations such as contextual  bandits, online learning, and reinforcement learning. 

\subsection{Broader perspectives}
Broadly speaking, our work connects to an extensive body of literature, both in empirical and theoretical communities, on  query-efficient learning. Spurred by the increasing prevelance of proprietary models and the availability of APIs to query inputs to these models at a small cost, recent empirical research has studied to what extent such API calls, which typically each reveal a small amount of information, can be used to reconstruct information such as the model's training data or an approximation to the model itself. For instance,  \cite{tramer2016stealing} showed that several types of models, including decision trees, SVMs, and neural networks can be reconstructed to high fidelity using a relatively small number of queries to evaluate the model at chosen inputs. Similar results have been shown for various specific domains, including sentence classification \cite{krishna2020thieves}, machine translation \cite{wallace2020imitation}, and sentence embedding encoders \cite{dziedzic2023sentence}.\footnote{See also \cite{dosovitskiy2015inverting,morris2023language} and references within for work on the related problem of inverting trained models.} These papers on ``model stealing'' roughly parallel an old line of work in learning theory on learning from queries \cite{angluin1988queries}, in which one can make several types of queries to the ground-truth hypothesis $h^\st$, such as a \emph{membership query} where one specifies $x$ and receives $h^\st(x)$. \cite{angluin1988queries} and many follow-up works study the question of how many such queries are sufficient for learning $h^\st$. 

The high-level implications of the works mentioned above parallel our own in that one often arrives at the conclusion that a surprisingly large amount of information be gleaned from a relatively small number of queries, each of which returns a relatively small number of bits. At a technical level, the above papers differ from our own in that the queries performed by the learning algorithm depend explicitly on the ground-truth hypothesis $h^\st$, whereas the oracle queries we consider are queries to the \emph{hypothesis class} $\MH$ without any mention of $h^\st$. Nevertheless, with the advent of methods such as in-context learning \cite{brown2020language}, which allows a fully trained large language model to simulate learning algorithms such as gradient descent, the distinction between these two settings is somewhat blurred. In particular, one could imagine a setting where a proprietary large language model itself serves as an ERM oracle for simpler classes, and thus our results demonstrating that queries which return only a few bits of information still permit learning could be informative. We leave it to future work to elucidate the question of whether the success of such weak oracles ultimately amounts to a feature (allowing learning without giving too much away) or a bug (giving enough away to allow reconstruction attacks). \noah{last sentence confusing?}

\section{Preliminaries}
\label{sec:prelim}
Consider a domain $\MX$, a label set $\MY$, and a \emph{concept class} $\MH \subset \MY^\MX$. Elements $h \in \MH$ are known as \emph{concepts} (or \emph{hypotheses}). In this paper, we consider the following different label sets $\MY$:
\begin{itemize}
\item If $\MY = \{0, 1, * \}$, then we say that $\MH$ is a \emph{partial (binary) concept class} \cite{alon2021theory}. A hypothesis which outputs a label of $*$ on some $x\in \MX$ should be interpreted as being undefined at $x$. In the special case that no hypothesis ever outputs $*$, a partial binary concept class is known as a \emph{total} binary concept class. We define the binary loss function $\lbin(y, y') := \One{y \neq y' \vee y = * \vee y' = *},$ for $y,y' \in \{0,1,* \}$. In words, we suffer a loss for true label $y$ when predicting $y'$ if we predict the wrong label \emph{or} either $y,y'$ is $*$.
\item If $\MY = [K]$, then $\MH$ is said to be a \emph{multiclass concept class}. We define the \emph{multiclass loss function} $\lmc(y, y') := \One{y \neq y'}$, for $y,y' \in [K]$.
\item If $\MY = [0,1]$, then $\MH$ is said to be a \emph{real-valued concept class}. We define the \emph{absolute loss function} $\labs(y,y') := |y-y'|$, for $y,y' \in [0,1]$. 
\end{itemize}
Throughout the paper, all concept classes $\MH \subset \MY^\MX$ will be understood as being either partial, multiclass, or real-valued concept classes.

 \subsection{PAC learning}
 Given a distribution $P \in \Delta(\MX \times \MY)$, a loss function $\ell : \MY \times \MY \to [0,1]$, and a hypothesis $h : \MX \to \MY$, we define $\err_{P, \ell}(h) := \E_{(x,y) \sim P}[\ell(h(x), y)]$. Given a sequence $S  \in (\MX \times \MY)^n$, which we refer to as a \emph{sample}, a loss function $\ell : \MY \times \MY \to [0,1]$, and $h : \MX \to \MY$, we define $\eerr_{S, \ell}(h) := \frac 1n \sum_{(x,y) \in S} \One{h(x) \neq y}$. When working with partial, multiclass, or real-valued concept classes, we will typically take $\ell$ to be the corresponding respective loss function among $\lbin, \lmc, \labs$. Thus, unless otherwise stated, in such situations we will write $\err_{P}(\cdot)$ in place of $\err_{P, \ell}(\cdot)$ and $\eerr_{S}(\cdot)$ in place of $\eerr_{S,\ell}(\cdot)$, where $\ell \in \{\lbin, \lmc,\labs\}$ is understood to be the appropriate choice. We denote samples using curly braces, but emphasize that samples should be interpreted as sequences of $n$ examples (in particular, examples can be repeated).

 Given a concept class $\MH\subset \MY^\MX$, a \emph{sample} $S = \{ (x_i, y_i) \}_{i\in [n]}  \subset (\MX \times \MY)^n$ is said to be \emph{$\MH$-realizable} if there is $h \in \MH$ so that $h(x_i) = y_i \neq * $ for each $i \in [n]$.  Moreover, a distribution $P \in \Delta(\MX \times \MY)$ is defined to be \emph{$\MH$-realizable} if the following holds: in the case that $\MH$ is a partial concept class, for any $n \in \BN$, then a sample $S \sim P^n$ is $\MH$-realizable with probability 1; in the case that $\MH$ is a multiclass or real-valued concept class, then $\inf_{h \in \MH} \err_P(h) = 0$. We remark that these two conditions coincide if $P$ has finite or countable support (see \cite[Lemma 33]{alon2021theory}).

 \paragraph{Realizable oracle-efficient PAC learning.} In the problem of \emph{realizable PAC learning} (or simply \emph{PAC learning}), for a concept class $\MH \subset \MY^\MX$ and a $\MH$-realizable distribution $P$, an algorithm $\Alg$ receives a sample $S \sim P^n$ and outputs a hypothesis $H : \MX \to \MY$. While often it is assumed that $\Alg$ has full knowledge of the class $\MH$, we are concerned with the setting in which $\Alg$'s only access to $\MH$ comes in the form of an oracle $\MO : (\MX \times \MY)^\st \times \{0,1\}^\st \to \{0,1\}^\st$, which takes as input a sequence of examples $(x,y) \in \MX \times \MY$ as well as a string of bits, and outputs a string of bits. While much prior work in the literature has focused on oracles, such as a (strong) ERM oracle (\cref{def:strong-erm}), which can return \emph{elements of $\MH$}, the oracles we consider are weaker in the sense that their only outputs are strings of bits, which will be quite short.\footnote{Of course, elements of $\MH$ can be represented with $\log |\MH|$ bits, but our oracles will always return strings of length $\poly(\log d, \log 1/\ep, \log\log 1/\delta)$, which can be infinitely smaller than $\log |\MH|$. Here $d$ denotes a dimension quantity (e.g., VC dimension) and $\ep,\delta$ are accuracy parameters.} For an input $(S, z) \in (\MX \times \MY)^\st \times \{0,1\}^\st$ to an oracle $\MO$, we let the \emph{size} of $(S,z)$ be $|S| + |z|$, namely the sum of the number of examples in $S$ and the number of bits in $z$. We say that an algorithm $\Alg$ has \emph{cumulative query cost $q$} if the sum of the sizes of the inputs for all oracle calls that $\Alg$ makes to $\MO$ is at most $q$. Note that the number of oracle calls made by $\Alg$ is bounded above by $q$. 
 \begin{definition}[Oracle-efficient PAC learning]
   \label{def:oracle-pac-learning}
   Let domain and label spaces $\MX, \MY$ be given. %
   Given $n \in \BN$, let $\Alg$ be an algorithm which takes as input a dataset $S \in (\MX \times \MY)^n$,  $x \in \MX$, and a string of uniformly random bits $R \in \{0,1\}^\st$, has cumulative query cost $q$ to an oracle $\MO : (\MX \times \MY)^\st \times \{0,1\}^\st \to \{0,1\}^\st$, and outputs some value $\Alg_R(S,x) \in \MY$, which is a deterministic function of $R,S,x$, and the results of the oracle calls. 

    Let $\MH \subset \MY^\MX$ be a hypothesis class and $\MO$ be an oracle as above.   Given $\ep, \delta \in (0,1)$, we say that the class $\MH$ is \emph{$(\MO; \ep, \delta)$-PAC learnable} by $\Alg$ with \emph{sample complexity $n$} and \emph{oracle complexity $q$} if %
   the following holds. Letting $H_{S,R}(x) := \Alg_R(S,x)$, we have \colt{$\Pr_{S \sim P^n, R} \left( \err_P(H_{S,R}) \leq \ep \right) \geq 1-\delta$.}
 \arxiv{  \begin{align}
\Pr_{S \sim P^n, R} \left( \err_P(H_{S,R}) \leq \ep \right) \geq 1-\delta\nonumber.
   \end{align}}
 \end{definition}
 We emphasize that in the above definition $\Alg$ has no knowledge of $\MH$ (apart from its calls to $\MO$); of course, the oracle $\MO$ will depend on $\MH$. Often we will slightly abuse terminology by stating that $\Alg$ ``outputs'' the hypothesis $H_{S,R}$. 

 \paragraph{Agnostic oracle-efficient PAC learning.} The setting of \emph{agnostic PAC learning} is similar to the setting of realizable PAC learning, except that the distribution $P \in \Delta(\MX \times \MY)$ is no longer required to be realizable. As such, we measure the performance of the output hypothesis of an algorithm by comparing to the best hypothesis in the class $\MH$. In the case that $\MH$ is a multiclass or real-valued class, the error of the best-performing concept in $\MH$ on $P$ is defined as $\err_P(\MH) := \inf_{h \in \MH} \err_P(h)$. In the case that $\MH$ is a partial (binary) concept class, we instead define $\err_P(\MH) := \lim_{n \to \infty} \E_{S \sim P^n} \left[ \min_{h \in \MH} \eerr_S(h) \right]$. \cite[Lemma 39]{alon2021theory} shows that the limit exists, and that when $\MH$ is a total (binary) class, the two notions of $\err_P(\MH)$ coincide.\footnote{See Footnote 12 of \cite{alon2021theory} for discussion on why the particular choice of $\err_P(\MH)$ is made.}
 \begin{definition}[Oracle-efficient agnostic PAC learning]
   \label{def:oracle-agnostic-learning}
   Using the setup and terminology of \cref{def:oracle-pac-learning}, the class $\MH$ is said to be \emph{$(\MO; \ep, \delta)$}-agnostically PAC learnable by $\Alg$ with sample complexity $n$ and oracle complexity $q$ if, for $H_{S,R}(x) := \Alg_R(S,x)$, we have
   \begin{align}
\Pr_{S \sim P^n, R} \left( \err_P(H_{S,R}) \leq \err_P(\MH) + \ep \right) \geq 1-\delta\nonumber.
   \end{align}
 \end{definition}

 \subsection{Oracles}
 \label{sec:oracles}
 In this section, we formally introduce the oracles that our algorithms will use. 
 We begin with a \emph{weak consistency oracle}, which will be used by our realizable PAC learning algorithms for partial and multiclass concept classes.
 \begin{definition}[Weak consistency oracle]
   \label{def:weak-con}
Given a concept class $\MH\subset \MY^\MX$, a \emph{(weak) consistency oracle} $\Ocon$ for $\MH$ is defined as follows: it takes as input a sample $S  \in (\MX \times \MY)^n$, and $\Ocon(S)$ outputs $\mathsf{True}$ if $S$ is $\MH$-realizable and $\mathsf{False}$ otherwise. 
\end{definition}

For real-valued learning, a weak consistency oracle is not sufficient for learning, due to the fact that labels in $\MY = [0,1]$ can take infinitely many values. Therefore, for realizable PAC learning in the real-valued setting, we make use of a \emph{range consistency oracle}, which is a natural generalization of a weak consistency oracle when one allows some margin of error in label space:
\begin{definition}[Range consistency oracle]
  \label{def:range-con}
  Given a real-valued concept class $\MH \subset [0,1]^\MX$, a \emph{range consistency oracle} $\Orange$ for $\MH$ is defined as follows: it takes as input a sample $S = \{ (x_i, \ell_i, u_i) \}_{i \in [n]} \in (\MX \times [0,1]^2)^n$, and outputs $\mathsf{True}$ if there is some $h \in \MH$ so that $\ell_i \leq h(x_i) \leq u_i$ for all $i \in[ n]$, and $\mathsf{False}$ otherwise. 
\end{definition}

The consistency oracles defined above are not sufficient for oracle-efficient agnostic PAC learning: %
the challenge is that even approximating the \emph{value} of the empirical risk minimizer $\min_{h \in \MH} \eerr_S(h) \in [0,1]$ on a sample $S \in (\MX \times \{0,1\})^n$ can require many weak consistency queries to $\MH$. %
Rather surprisingly, it turns out that an oracle which returns only the value of the empirical risk minimizer on a sample, as defined formally below, is sufficient for efficient agnostic PAC learnability.

\begin{definition}[Weak ERM oracle]
  \label{def:weak-erm}
Consider a concept class $\MH \subset \MY^\MX$, and a real-valued loss function $\ell : \MY \times \MY \to [0,1]$. A \emph{weak ERM oracle} $\Oerm$ is a mapping which takes as input a dataset $S \in (\MX \times \MY)^n$ and outputs the value $\min_{h \in \MH}  \eerr_{S,\ell}(h) \in [0,1]$.\footnote{When $\ell \in \{\lbin, \lmc\}$ is binary-valued, $\min_{h \in \MH}\eerr_{S,\ell}(h) \in \{0, 1/n, \ldots, 1 \}$ can be represented with $O(\log n)$ bits. In the real-valued setting, while this is no longer the case, one can assume that $\Oerm$ returns only the $\log(1/\ep)$ most significant bits of the empirical risk, at the cost of an $O(\ep)$ error that propagates through the PAC bounds. For simplicity, we ignore such considerations relating to arithmetic precision.}
\end{definition}

Finally, for reference, we introduce the standard notion of ERM oracle, which returns a hypothesis that minimizes the empirical risk on a sample; to contrast with a weak ERM oracle, we call such an oracle a strong ERM oracle.
\begin{definition}[Strong ERM oracle]
  \label{def:strong-erm}
  Consider a concept class $\MH \subset \MY^\MX$ and a real-valued loss function $\ell : \MY \times \MY \to [0,1]$. A \emph{strong ERM oracle} $\Oerms$ is a mapping which takes as input a dataset $S \in (\MX \times \MY)^n$ and outputs some concept in $\argmin_{h \in \MH} \eerr_{S,\ell}(h) \in \MH$. 
\end{definition}

\arxiv{\colt{
  \section{Additional Preliminaries}
In this section, we give some additional preliminaries which will be useful in the proofs. 
  }
\arxiv{\subsection{Sample compression schemes}}
\label{sec:scs}
Our techniques will involve the use of sample compression schemes, which we proceed to define.
\begin{definition}[Sample compression scheme; \cite{littlestone2003relating,david2016supervised}]
Fix a domain $\MX$ and a label set $\MY$.  A \emph{compression scheme} for the tuple $(\MX, \MY)$ is a pair $(\kappa, \rho)$, consisting of a compression function $\kappa : (\MX \times \MY)^\st \to (\MX \times \MY)^\st \times \{0,1\}^\st$ and a \emph{reconstruction function} $\rho : (\MX \times \MY)^\st \times \{0,1\}^\st \to \MY^\MX$, satisfying the following property. For any sequence $S \in (\MX \times \MY)^\st$, $\kappa(S)$ evaluates to some tuple $(S', B) \in (\MX \times \MY)^\st \times \{0,1\}^\st$, where $S'$ is a sequence of elements of $S$.

  For $S \in (\MX \times \MY)^m$, writing $(S', B) := \kappa(S)$, define $|\kappa(S)| := |S'| + |B|$, i.e., to denote the sum of the number of samples in $S'$ and the length of $B$.  The \emph{size} of the compression scheme $(\kappa, \rho)$ for \emph{$m$-sample datasets} is $|\kappa| := \max_{S \in (\MX \times \MY)^{\leq m}} |\kappa(S)|$.

  For a (partial, multiclass, or real-valued) concept class $\MH$, a \emph{sample compression scheme for $\MH$} is a compression scheme $(\kappa, \rho)$, so that for every $\MH$-realizable sequence $S \in (\MX \times \MY)^m$, $\rho(\kappa(S))$ correctly classifies every point in $S$, i.e., $\frac 1n \sum_{(x,y) \in S} \ell(\rho(\kappa(S)), y) = 0$, where $\ell \in \{ \lbin, \lmc, \labs \}$ is the appropriate loss function corresponding to $\MH$. 
\end{definition}

\cref{lem:gen-compression} below shows that compression schemes of bounded size generalize. 
\begin{lemma}[Generalization-by-compression; Theorem 2.1 of \cite{david2016supervised}]
  \label{lem:gen-compression}
  There is a constant $C > 0$ so that the following holds. Consider any domain $\MX$ and label set $\MY$, together with a loss function $\ell : \MY \times \MY \to [0,1]$. For any compression scheme $(\kappa, \rho)$, for any $n \in \BN$, and $\delta \in (0,1)$, for any distribution $P \in \Delta(\MX \times \{0,1\})$, the following holds with probability $1-\delta$ over $S \sim P^n$:
  \begin{align}
    & | \err_{P,\ell}(\rho(\kappa(S))) - \eerr_{S, \ell}(\rho(\kappa(S))) | \nonumber\\
    \leq & C \sqrt{\eerr_{S, \ell}(\rho(\kappa(S))) \cdot \frac 1n \left(|\kappa(S)|\log(n) + \log \frac{1}{\delta}\right)} + C \cdot \frac 1n \left(|\kappa(S)|\log(n) + \log \frac{1}{\delta}\right)\nonumber.
  \end{align}
In particular, if $\eerr_{S, \ell}(\rho(\kappa(S))) = 0$, then
  \begin{align}
\err_{P, \ell}(\rho(\kappa(S))) \leq & C \cdot \frac 1n \left(|\kappa(S)|\log(n) + \log \frac{1}{\delta}\right)\nonumber.
  \end{align}
\end{lemma}

}

\subsection{The one-inclusion graph}
\label{sec:misc-prelim}
In this section, we introduce the one-inclusion graph \cite{haussler1988predicting}, which plays a fundamental role in many PAC learning results. %
For a (partial, multiclass, or real-valued) concept class $\MH \subset \MY^\MX$ and $X = (x_1, \ldots, x_m) \in \MX^m$, we define $\MH|_X := \{ y \in (\MY \backslash \{ * \})^m \ : \ \exists h \in \MH \mbox{ s.t. } h(x_i) = y_i \ \forall i \in [m] \}$. Note that, in the special case of partial concept classes, $\MH|_X \subset \{0,1\}^m$ and in particular does \emph{not} include the $*$ symbol. For a partial binary concept class $\MH$, its \emph{VC dimension}, denoted $\VCdim(\MH)$, is the largest positive integer $d$ so that there is some $X = (x_1, \ldots, x_d) \in \MX^d$ so that $\MH|_X = \{0,1\}^d$. It is known that the VC dimension tightly characterizes statistical learnability of (partial) concept classes \cite{alon2021theory}. 

For $v \in \{0,1\}^n$ and $i \in [n]$, we write $v^{\oplus i}$ to denote $v$ with coordinate $i$ flipped, i.e., $v^{\oplus i}_j = v_j$ for all $j \neq i$ and $v^{\oplus i}_i = 1-v_i$. For $n \in \BN$, let $G_n = (V_n, E_n)$ be the $n$-dimensional hypercube graph, so that $V_n = \{0,1\}^n$ and $E_n = \{( (v_{-i}, 0), (v_{-i},1)) \ : \ i \in [n], v_{-i} \in \{0,1\}^{n-1} \}$.
\begin{definition}[One-inclusion graph]
  \label{def:oig}
  Consider a set $\MW \subset \{0,1\}^n$. The \emph{one-inclusion graph} $G(\MW) = (V,E)$ induced by $\MW$ is defined as the following graph.  The vertex set $V$ is equal to $\MW$. The edge set $E$ is the subgraph of $G_n$ induced by $\MW$, namely:
  \begin{align}
E := \{ (v, v^{\oplus i}) \ : \ i \in [n],\ v,v^{\oplus i} \in \MW \}.\nonumber
  \end{align}
  For any $i \in [n]$ and $h \in \{0,1\}^n$, we will occassionally write $e_{i,h}$ to refer to the edge $(h, h^{\oplus i})$. For a partial concept class $\MH \subset \{0,1,*\}^\MX$ and $X \in \MX^n$, we refer to the the one-inclusion graph induced by $\MH|_X \subset \{0,1\}^n$ as the one-inclusion graph of $\MH$ induced by $X$. 
\end{definition}
For a set $\MW \subset \{0,1\}^n$, $\MW^{\mathrm{c}}$ denotes its complement in $\{0,1\}^n$.

 \paragraph{Orientations.} Given a graph $G = (V,E)$, a \emph{random orientation} of $G$ is a mapping $\sigma : E \to \Delta(V)$, where, for all $e \in E$, $\supp(\sigma(e)) \subseteq e$ (i.e., $\sigma(e)$ is supported on the 2 vertices of $e$). $\sigma$ is called an \emph{orientation} if $\sigma(e)$ is supported on a single vertex $v$, in which case we will write $v = \sigma(e)$. Given a function $F : V \to [0,1]$ and $\lambda \in [0,1]$, we consider a random orientation $\sigma_{F,\lambda}$ \emph{induced by $F$}, defined as follows: for an edge $e =(v,v')$, we set
\begin{align}
\sigma_{F,\lambda}(e)(v) = \frac{1 + \lambda \cdot (F(v') - F(v))}{2}, \qquad \sigma_F(e)(v') = \frac{1 + \lambda \cdot (F(v) - F(v'))}{2}\label{eq:define-orientation}.
\end{align}
Given a random orientation $\sigma$ and a vertex $v \in V$, we define the \emph{out-degree} of $\sigma$ at $v$ to be \colt{$\outdeg(v;\sigma) := \sum_{e \ni v} \left( 1 - \sigma(e)(v)\right)$,}
\arxiv{\begin{align}
\outdeg(v;\sigma) := \sum_{e \ni v} \left( 1 - \sigma(e)(v)\right)\nonumber,
\end{align}}
and the out-degree of $\sigma$ is $\outdeg(\sigma) := \max_{v \in V} \outdeg(v;\sigma)$. 

\section{Learning partial concept classes with a weak oracle}
\label{sec:partial-weak}
In this section, we give an algorithm for realizable PAC learning with low oracle complexity for a weak consistency oracle, and an algorithm for agnostic PAC learning with low oracle complexity for a weak ERM oracle.
\begin{theorem}[Oracle-efficient partial concept class learning]
  \label{thm:partial-main}
  For any $\ep, \delta \in (0,1)$ and $\VCdim \in \BN$, the following statements hold:
  \begin{enumerate}
  \item There is an algorithm $\AlgR$  so that for any class $\MH \subset \{0,1,* \}^\MX$ satisfying $\VCdim(\MH) \leq \VCdim$ and any weak consistency oracle $\Ocon$ for $\MH$, the class $\MH$ is $(\Ocon; \ep, \delta)$-PAC learnable by $\AlgR$ with sample complexity $ n = \tilde O \left( \frac{\VCdim^3 \log(1/\delta)}{\ep} \right)$ and oracle complexity $\poly(n)$.
  \item There is an algorithm $\AlgA$  so that for any class $\MH \subset \{0,1,* \}^\MX$ satisfying $\VCdim(\MH) \leq \VCdim$ and any weak ERM oracle $\Oerm$ for $\MH$, the class $\MH$ is $(\Oerm; \ep, \delta)$-PAC learnable by $\AlgA$ with sample complexity $ n = \tilde O \left( \frac{\VCdim^3 \log(1/\delta)}{\ep^2} \right)$ and oracle complexity $\poly(n)$.
  \end{enumerate}
\end{theorem}
In the theorem statement above, $\tilde O(\cdot)$ hides factors which are polynomial in $\log (\VCdim), \log( 1/\ep), \log\log(1/\delta)$. The proof of the realizable case of \cref{thm:partial-main} proceeds by first constructing an algorithm (\WeakRealizable; \cref{alg:weak-oig}) which is a \emph{weak learner} for any class of VC dimension at most $\VCdim$ in the realizable setting and makes polynomially many oracle calls to $\Ocon$. %
We then use a standard boosting algorithm (namely, \Adaboost; \cref{alg:adaboost}) to boost the performance of the weak learner so as to obtain a learner which has error at most $\ep$ with high probability. To analyze the generalization error of this approach, we use a technique involving sample compression schemes \cite{david2016supervised,schapire2012boosting}. Finally, to handle the agnostic case of \cref{thm:partial-main}, we reduce to the realizable setting by showing that a weak ERM oracle can be used in an efficient manner to determine, given any sample $S \in (\MX \times \{0,1\})^n$, a subsample of maximum size which is $\MH$-realizable (\cref{lem:weak-to-sample-binary}). 
In the remainder of the section, we introduce our weak learner; the remaining ingredients are (mostly) standard and are presented in \cref{sec:boosting,sec:main-partial-proof}.

\subsection{An oracle-efficient weak learner}
\label{sec:weak-learner}
Our goal is to construct an oracle-efficient weak learner, namely one that improves upon random guessing in expectation over its dataset by a small margin $\eta > 0$: 
\begin{definition}[Weak learner]
  \label{def:weak-learner}
  For $m \in \BN$ and $\eta \in (0,1)$, a randomized learning algorithm $\SA : (\MX \times \{0,1\})^m \times \MX \to \{0,1\}$ is defined to be a \emph{$m$-sample weak learner with margin $\eta$} for the concept class $\MH$ if the following holds. For any $\MH$-realizable distribution $P \in \Delta(\MX \times \{0,1\})$, $\SA$ takes as input an i.i.d.~sample $S \sim P^m$ and $x \in \MX$ and outputs a (random) bit $\SA(S, x)$, so that %
  \begin{align}
\E_{S \sim P^m} \E_{(x,y) \sim P} \E_{\SA} \left[ \lbin(\SA(S, x), y) \right] \leq \frac 12-  \eta \label{eq:weak-learning}.
  \end{align}
\end{definition}
We construct an oracle-efficient weak learner using polynomially many calls to a weak consistency oracle $\Ocon$ by simulating a random walk on the one-inclusion graph of $\MH|_X$ for an appropriate choice of $X \in \MX^m$. This procedure is formalized in the \WeakRealizable algorithm (\cref{alg:weak-oig}), whose main guarantee is shown below: 
\begin{theorem}[Weak learning guarantee]
  \label{thm:weak-oig}
  There are constants $C_1, C_2$ so that the following holds. Consider a partial concept class $\MH$ of VC dimension $d$, $\delta \in (0,1)$, and suppose $m \geq C_1 d \log d$. 
  For an $\MH$-realizable sample $S \in (\MX \times \{0,1\})^{m-1}$ and $x \in \MX$, let $\SA(S, x) \in \{0,1\}$ be the output of $\WeakRealizable(S, x, 1-\frac{1}{C_1 m \log m}, 1,C_1 m^2 \log^3 m, \Ocon)$ (\cref{alg:weak-oig}), which is a random variable.
  Then for any $\MH$-realizable sample $S = \{ (x_i, y_i) \}_{i \in [m]} \in (\MX \times \{0,1\})^m$, it holds that 
  \begin{align}
\frac{1}{m} \sum_{i=1}^{m} \E \left[ \lbin( \SA(S_{-i}, x_i),  y_i) \right] \leq \frac{1}{2} - \frac{1}{C_2 m \log m}\label{eq:weak-transductive}
  \end{align}
   where the expectation is taken over the randomness in the runs of $\SA(S_{-i}, x_i)$. Moreover, \WeakRealizable makes at most $\tilde O(m^3)$ calls to $\Ocon$, each with a dataset of size $m-1$.  %
 \end{theorem}
 \begin{algorithm}
  \caption{Weak oracle-efficient OIG learner}
  \label{alg:weak-oig}
\begin{algorithmic}[1]\onehalfspacing
  \Require A partial concept class $\MH$, an $\MH$-realizable sample $S = \{ (x_i, y_i) \}_{i \in [m-1]} \in (\MX \times \{0,1\})^{m-1}$, query point $x \in \MX$, consistency oracle $\Ocon$, parameters $\lambda, \gamma \in (0,1)$, $U \in \BN$. %
  \Function{\WeakRealizable}{$S, x, \gamma,\lambda, U, \Ocon$}
  \State Set $X \gets (x_1, \ldots, x_{m-1}, x) \in \MX^{m}$.
  \State Set $y^0 \gets (y_1, \ldots, y_{m-1}, 0) \in \{0,1\}^{m}$ and $y^1 \gets (y_1, \ldots, y_{m-1}, 1) \in \{0,1\}^{m}$.
  \If{$\Ocon(\{ (X_j, y^b_j) \}_{j \in [m]}) = \False$ for some $b \in \{0,1\}$}
  \State \Return $1-b$. \label{line:oig-easy-case}
  \EndIf
  \State For $b \in \{0,1\}$, set $\hat F(y^b) \gets \EstimatePotential(X, y^b, K, \gamma, \Ocon)$. 
  \State \Return a sample from $\mathrm{Ber}(\hat \sigma)$, where $\hat \sigma := \frac{1 + \lambda \cdot (\hat F(y^0) - \hat F(y^1))}{2}$.\label{line:sample-oig}
  \EndFunction

  \Require $U, \gamma, \Ocon$ as above, and $X \in \MX^m, y \in \{0,1\}^m$. 
  \Function{\EstimatePotential}{$X$, $y, U, \gamma,\Ocon$}
  \Comment{\emph{$y$ represents a vertex of the OIG induced by $\MH|_X$, and $U$ is the number of trials}}
  \For{$1 \leq u \leq U$}
  \State Set $Y\^0 \gets y$ and $T_u \gets  \frac{\log(32e/(1-\gamma))}{\log(1/\gamma)}$.
  \For{$0 \leq t \leq  \frac{\log(32e/(1-\gamma))}{\log(1/\gamma)}$}
  \If{$\Ocon(\{(X_j, (Y\^t)_j) \}_{j \in [m]}) = \False$}
  \State Set $T_u \gets t$, and \textbf{break} (out of the inner \textbf{for} loop).
  \Else
  \State Choose $i \sim \Unif([m])$, and set $Y\^{t+1} \gets (Y\^t)^{\oplus i}$.
  \EndIf
  \EndFor
  \EndFor
  \State \Return the quantity $\frac{1}{U} \sum_{u=1}^U \gamma^{T_u}$.
  \EndFunction
  \end{algorithmic}
  
\end{algorithm}

 The guarantee \cref{eq:weak-transductive}, in which an arbitrary realizable dataset $S$ is fixed and the algorithm's performance is measured on all leave-one-out configurations of $S$, is known as a \emph{transductive learning} guarantee. A standard exchangeability argument (see \cref{lem:loo}) shows that \cref{eq:weak-transductive} implies an in-expectation error guarantee under any realizable distribution $P \in \Delta(\MX \times \{0,1\})$, and thus \cref{thm:weak-oig} implies that \WeakRealizable is an $m$-sample weak learner with margin $\eta = \Theta(1/(m \log m))$ for $\MH$. In the remainder of the section we focus on the proof of \cref{thm:weak-oig}.

 \paragraph{Analyzing \WeakRealizable.} Given a dataset $S = \{ (x_i, y_i) \}_{i \in [m-1]}$ together with a ``query point'' $x \in \MX$, \WeakRealizable considers the two vertices $y^0 = (y_1, \ldots, y_{m-1}, 0), y^1 = (y_1, \ldots, y_{m-1}, 1)$ of the one-inclusion graph $G(\MH|_X)$ induced by $\MH$ on the sequence $X = (x_1, \ldots, x_{m-1}, x)$. (If $y^b$ is not a vertex of $G(\MH|_X)$ for some $b \in \{0,1\}$, then, by realizability, the correct prediction on $x$ must be $1-b$ -- see \cref{line:oig-easy-case} of \cref{alg:weak-oig}.) \WeakRealizable then calls \EstimatePotential on each of the vertices $y^0, y^1$, which returns estimates $\hat F(y^0), \hat F(y^1)$ of a certain potential function on vertices of $G$. These potentials are used to randomly return an output bit in \cref{line:sample-oig}.

 The proof that \WeakRealizable satisfies \cref{eq:weak-transductive} proceeds by considering the following perspective: the value $\hat \sigma = \frac{1 + \lambda(\hat F(y^0) - \hat F(y^1))}{2}$ computed in \cref{line:sample-oig} can be viewed as a decision to randomly orient the edge $(y^0,y^1)$ of the one-inclusion graph $G(\MH|_X)$ by putting mass $\hat \sigma$ on $y^1$ and mass $1-\hat \sigma$ on $y^0$ (see \cref{sec:misc-prelim}). To minimize loss, we hope that this orientation puts as much mass as possible on whichever of $y^0, y^1$ corresponds to the ground-truth hypothesis, i.e., we want the edge $(y^0, y^1)$ to not contribute much to the out-degree of the ground-truth. 

 Translated into this language of orientations, the transductive error guarantee \cref{eq:weak-transductive} of \cref{thm:weak-oig} is therefore equivalent to the following statement: fix $m \in \BN$, consider any $\MH$-realizable dataset $S = \{ (x_i, y_i) \}_{i \in [m]}$, and let $X = (x_1, \ldots, x_m)$. Then the random orientations of the $m$ edges adjacent to $y \in G(\MH|_X)$ induced by running \WeakRealizable with inputs $(S_{-i}, x_i)$, for each $i \in [m]$, lead the \emph{out-degree} of $y$ to be bounded above by $m \cdot \left(1/2 - \Omega(1/(m\log m))\right)$. As a sanity check, it is trivial to achieve out-degree $m/2$ by orienting each edge to each of its vertices with probability $1/2$; thus, the quantity of interest is the decrease of $-m \cdot \Omega(1/(m\log m))$ in the out-degree.

 To explain how we achieve such an out-degree bound, consider the $m$-dimensional hypercube graph $G_m = (V_m,E_m)$. Note that the one-inclusion graph $G(\MH|_X)$ is the subgraph of $G_m$ induced by $\MH|_X$. Given $v \in V_m$, we consider the (lazy) random walk on $G_m$ started at $v$. In particular, it is the sequence $Z_v\^0, Z_v\^1, Z_v\^2, \ldots \in V_m$ of random variables with $Z_v\^0 = v$, and with $Z_v\^t$ defined as follows, for $t \geq 0$: given $Z_v\^t \in V_m$, the value of $Z_v\^{t+1}$ is defined by selecting uniformly at  random an edge $e$ of $G_m$ containing $Z_v\^t$, and then letting $Z_v\^{t+1}$ to be a uniformly random vertex of $e$. Given a subset $\MS \subset V_m$ and a vertex $v \in V_m$, the \emph{hitting time} for $\MS$ starting at $v$ is the random variable
\begin{align}
\tau_{\MS,v} := \min \left\{ t \geq 0 \ : \ Z_v\^t \in \MS \right\}\label{eq:stop-at-s}.
\end{align}
Moreover, the generating function $M_{\MS,v}(\gamma)$, for $\gamma \in (0,1)$, is defined for $v \in V_m$ by $M_{\MS, v}(\gamma) := \E[\gamma^{\tau_{\MS,v}}]$.\footnote{For $v \in \MS$, we have $\tau_{\MS,v} = 0$ and hence $M_{\MS,v}(\gamma) = 1$.}
The definition of the random walk yields the following recursive formula for $M_{\MS, v}(\gamma)$ (see \cref{lem:recursion} for a formal statement): for all $v \in \MS^\comp$, 
\begin{align}
M_{\MS, v}(\gamma) = \frac{\gamma}{(2-\gamma) m} \sum_{i \in[m]} M_{\MS, v^{\oplus i}}(\gamma)\label{eq:w-recursion}.
\end{align}
Given $\MH$ and $X \in \MX^m$ as above, we now choose $\MS := (\MH|_X)^\comp$, $\gamma := 1-\Theta(1/(m \log m))$, and define $F(v) = M_{\MS, v}(\gamma)$. We may consider the orientation $\sigma_{F,1}$ induced by $F$ (see \cref{eq:define-orientation}). It is a simple consequence of \cref{eq:w-recursion} (see \cref{lem:outdeg-m-bound} that
\begin{align}
\outdeg(\sigma_{F,1}) \leq \frac m2 - (1-\gamma) m \cdot \min_{v \in V_m} F(v)\label{eq:outdeg-F-bound}.
\end{align}
Finally, we can show (in \cref{lem:min-m-bound}) that as long as $m \geq \Omega(d \log d)$ (where $d$ is an upper bound on $\VCdim(\MH)$), we have $\min_{v \in V_m} F(v) \geq \Omega(1)$. This statement is a consequence of the Sauer-Shelah lemma, which bounds $|\MS^\comp| = |\MH|_X| \leq (em)^d$. In particular, since $\MS^\comp$ is relatively ``small'', the hitting time $\tau_{\MS,v}$ cannot get too large for any vertex $v$, meaning that $\gamma^{\tau_{\MS,v}}$ cannot become too small. To summarize, we thus obtain from \cref{eq:outdeg-F-bound} that $\outdeg(\sigma_{F,1}) \leq m \cdot \left( \frac{1}{2} - \Omega(1-\gamma) \right) = m \cdot \left( \frac 12 - \Omega\left( \frac{1}{m\log m} \right) \right)$.

Since the generating function $F(v)$ is not known exactly, \WeakRealizable cannot compute the orientation $\sigma_{F,1}$ exactly. Instead, it computes estimates $\hat F(y^0), \hat F(y^1)$ of $F(y^0), F(y^1)$ respectively (using \EstimatePotential), via random rollouts. Crucially, doing so is possible using only a weak consistency oracle $\Ocon$: we only need to be able to check, at each step, whether the random walk has hit $\MS = (\MH|_X)^\comp$, which is exactly what is accomplished by $\Ocon$.  By standard concentration arguments, we can show that the induced orientation $\sigma_{\hat F,1}$ is sufficiently close to $\sigma_{F,1}$ to enjoy the same outdegree bounds, thus establishing \cref{thm:weak-oig}. The full details of the proof of \cref{thm:weak-oig} may be found in \cref{sec:weak-oig-proof}.

\section{Extensions to multiclass and real-valued classes}
\label{sec:extensions}
We next extend the guarantee of \cref{thm:partial-main} to the settings of \emph{multiclass classification} and \emph{regression}. The proofs for both of these settings proceed via a reduction to the case of partial concept classes.
\subsection{Multiclass concept classes}
\label{sec:ext-mc}
Our upper bounds for the multiclass setting are phrased in terms of \emph{Natarajan dimension}: for a multiclass concept class $\MH \subset [K]^\MX$, its \emph{Natarajan dimension}, denoted $\Natdim(\MH)$, is the smallest $d \in \BN$ so that there is some $X = (x_1, \ldots, x_d) \in \MX^d$ together with vectors $a,b \in [K]^d$ with $a_i \neq b_i$ for all $i \in [d]$ so that $\MH|_X \supseteq \{ a_1, b_1 \} \times \cdots \times \{ a_d, b_d \}$. It is known that the algorithm the algorithm which returns an empirical risk minimizer of $\MH$ on an i.i.d.~sample, which requires access to a \emph{strong} ERM oracle $\Oerms$, enjoys sample complexity for PAC learning of $\tilde O(\Natdim(\MH) \log(K)/\ep)$ in the realizable setting and of $\tilde O(\Natdim(\MH) \log(K)/\ep^2)$ in the agnostic setting \cite{daniely2011multiclass}. \cref{thm:multiclass-main} shows that we can extend this result to the setting where we only have a \emph{weak} ERM oracle, as long as the oracle complexity is allowed to grow linearly with $K$.
\begin{theorem}[Oracle-efficient multiclass learning]
  \label{thm:multiclass-main}
  For any $\ep, \delta \in (0,1)$ and $\Natdim \in \BN$, the following statements hold:
  \begin{enumerate}
  \item There is an algorithm $\AlgR$ so that for any class $\MH \subset [K]^\MX$ satisfying $\Natdim(\MH) \leq \Natdim$ and any weak consistency oracle $\Ocon$ for $\MH$, the class $\MH$ is $(\Ocon; \ep, \delta)$-PAC learnable by $\AlgR$ with sample complexity $n = \tilde O \left( \frac{\Natdim^3 \log^4(K/\delta)}{\ep} \right)$ and oracle complexity $ K \cdot \poly(n)$.
  \item There is an algorithm $\AlgA$ so that for any class $\MH \subset [K]^\MX$ satisfying $\Natdim(\MH) \leq \Natdim$ and any weak ERM oracle $\Oerm$ for $\MH$, the class $\MH$ is $(\Oerm; \ep, \delta)$-agnostically PAC learnable with sample complexity $n = \tilde O \left( \frac{\Natdim^3 \log^4(K/\delta)}{\ep^2} \right)$ and oracle complexity $K \cdot \poly(n)$.
  \end{enumerate}
\end{theorem}
The $\tilde O(\cdot)$ above hides factors that are polynomial in $\log( 1/\ep), \log( \Natdim), \log (K/\delta)$. It is straightforward to show that oracle complexity growing linearly in $K$ is necessary if one only uses a weak ERM or consistency oracle, by considering the case where $\MH$ is a class that consists of a single unknown hypothesis on a large domain $\MX$, and where the covariates are uniformly distributed on $\MX$.

It is known that for any class $\MH \subset [K]^\MX$, the sample complexity of PAC learning $\MH$ is %
always within a polynomial factor of the DS dimension of $\MH$, denoted $\DSdim(\MH)$ \cite{brukhim2022characterization}, and is in particular bounded below by $\Omega(\DSdim(\MH))$ (see \cref{sec:lb-mc} for a definition of the DS dimension). Moreover, we always have $\Natdim(\MH) \leq \DSdim(\MH) \leq O(\Natdim(\MH) \cdot \log K)$.  Thus, the sample complexity obtained by the oracle-efficient algorithms $\AlgR, \AlgA$ of \cref{thm:multiclass-main} comes within a $\poly \log K$ factor of the optimal sample complexity. While the $\log K$ factor is unlikely to be large in many applications, it is nevertheless of theoretical interest to wonder if there is an oracle-efficient algorithm with sample complexity $\poly (\DSdim(\MH))$, even if one allows a \emph{strong} ERM oracle. We show in \cref{thm:mc-lb} (\cref{sec:lb}) that no such algorithm exists, even if we restrict $\DSdim(\MH) = 1$.

\subsection{Real-valued concept classes}
\label{sec:ext-real}
Our bounds for the regression setting are phrased in terms of \emph{fat-shattering dimension}: for a real-valued concept class $\MH \subset [0,1]^\MX$ and $\discmg \in (0,1)$, its \emph{fat-shattering dimension at scale $\discmg$}, denoted $\fatdim(\MH)$, is the largest positive integer $d$ so that there exist $x_1, \ldots, x_d \in \MX$ and $s_1, \ldots, s_d \in [0,1]$ so that, for all $b \in \{0,1\}^d$, there is some $h \in \MH$ so that $h(x_i) \geq s_i + \discmg$ if $b_i = 1$ and $h(x_i) \leq s_i - \discmg$ if $b_i = 0$. It is known that finiteness of the fat-shattering dimension at all scales $\discmg$ is a sufficient condition for learnability in both the realizable and agnostic settings, and that a sample complexity scaling nearly linearly with the fat-shattering dimension at an appropriate scale  can be obtained by outputting an empirical risk minimizer of $\MH$ on an i.i.d.~sample (which requires access to a strong ERM oracle) \cite{long2001on,bartlett1998prediction,alon1997scale}. \cref{thm:reg-main} shows that we can extend this result to the setting where we only have a weak ERM oracle, with a polynomial cost in the sample complexity.
\begin{theorem}[Oracle-efficient regression]
  \label{thm:reg-main}
  For any $\delta \in (0,1)$, $n \in \BN$, and function $\discmg \mapsto \fatdim \in \BN$ (for $\discmg \in (0,1)$), the following statements hold:
  \begin{enumerate}
  \item There is an algorithm $\AlgR$ so that for any class $\MH \subset [0,1]^\MX$ satisfying $\fatdim(\MH) \leq \fatdim$ for all $\discmg$ and any weak range oracle $\Orange$ for $\MH$, the class $\MH$ is $(\Orange; \ep, \delta)$-PAC learnable with sample complexity $n$ and oracle complexity $\poly(n)$, for $\ep =   \inf_{\discmg \in [0,1]} \left\{O( \discmg) + \tilde O \left(\frac{\fatdim^3 \cdot \log(1/\delta)}{n} \right) \right\}$.
  \item There is an algorithm $\AlgA$ so that for any class $\MH \subset [0,1]^\MX$ satisfying $\fatdim(\MH) \leq \fatdim$ for all $\discmg$ and any weak ERM oracle $\Oerm$ for $\MH$, the class $\MH$ is $(\Oerm; \ep, \delta)$-agnostically PAC learnable with sample complexity $n$ and oracle complexity $\poly(n)$, for \colt{\newline} $ \ep =   \inf_{\discmg \in [0,1]} \left\{ O(\discmg) + \tilde O \left(\sqrt{\frac{\fatdim^3 \cdot \log(1/\delta)}{n}}\right)\right\}.$
  \end{enumerate}
\end{theorem}
The $\tilde O(\cdot)$ above hides factors that are polynomial in $\log (n), \log( \fatdim), \log\log  (1/\delta)$. In the agnostic setting the fat-shattering dimension is known to characterize PAC learnability, and thus \cref{thm:reg-main} shows that the price to pay for oracle-efficiency with respect to $\Oerm$ is only a polynomial (assuming reasonable growth of $\fatdim$). In contrast, in the realizable setting, the sample complexity is characterized by a different quantity known as \emph{the one-inclusion graph (OIG) dimension}  \cite{attias2023optimal}, which can be smaller than the fat-shattering dimension by an arbitrarily large factor. We show in \cref{thm:reg-lb} (\cref{sec:reg-lbs}) that, even with a strong ERM oracle, it is impossible to obtain an oracle-efficient algorithm even for classes whose OIG dimension is a constant. 

\section*{Acknowledgements}
CD is supported by NSF Awards CCF-1901292, DMS-2022448, and DMS2134108, a
Simons Investigator Award, and the Simons Collaboration on the Theory of Algorithmic Fairness. NG is supported by a Fannie \& John Hertz Foundation Fellowship and an NSF Graduate Fellowship.
\arxiv{\bibliographystyle{alpha}}
\newpage
\bibliography{oig.bib}
\appendix
\newpage
\colt{}

\section{Helpful lemmas}
In this section we collect various probabilistic lemmas which are used throughout the proofs. 
Fix $n \in \BN$, and consider the hypercube $V_n = \{0,1\}^n$. For some $v \in V$, we consider the lazy random walk on $V_n$, denoted $Z_v\^0, Z_v\^1, \ldots $, where $Z_v\^0$, and $Z_v\^t$ is generated from $Z_v\^{t-1}$ by picking $i \in [n]$ uniformly at random and flipping the $i$th coordinate of $Z_v\^{t-1}$ with probability $1/2$. 
\begin{lemma}[Mixing time of the hypercube; \cite{levin2006markov}]
  \label{lem:hypercube-mixing}
  Consider $n \in \BN$, $v \in V_n$, and let $Z_v\^0, Z_v\^1, \ldots$ denote the lazy random walk on the hypercube $V_n$. Let $U$ be a uniformly distributed random variable on $V_n$. Then for any $\ep \in (0,1)$ and $t \geq n \log n + n \log(1/\ep)$, it holds that
$ 
\tvd{Z_v\^t}{U} \leq  \ep.
$
\end{lemma}

\begin{lemma}[Sauer-Shelah; \cite{shalev2014understanding}]
  \label{lem:ss}
If $\MH \subset \{0,1, * \}^\MX$ is a partial concept class with VC dimension $d$ and $X \in \MX^m$, then $| \MH|_X | \leq (em/d)^d$. 
\end{lemma}

The following result is a corollary of Freedman's inequality.
\begin{lemma}[Lemma A.3 of \cite{foster2021statistical}]
  \label{lem:freedman}
  Let $(X_t)_{t \in [T]}$ be a sequence of random variables adapted to a filtration $(\MF_t)_{t \in [T]}$. If $0 \leq X_t \leq R$ almost surely for all $t \in [T]$, then with probability at least $1-\delta$,
  \begin{align}
\sum_{t=1}^T \E[X_t \mid \MF_{t-1}]\leq 2 \sum_{t=1}^T X_t + 8 R \log(2/\delta)\nonumber.
  \end{align}
\end{lemma}

\section{Manipulating a weak ERM oracle}
\label{sec:oracle-implications}
In this section, we prove some lemmas showing that a weak ERM oracle can be used to implement a slightly stronger oracle which, given a dataset $S = \{ (x_i, y_i) \}_{i \in [n]} \in (\MX \times \MY)^n$, gives the values of $h^\st(x_i)$ for $i \in [n]$, where $h^\st$ is a risk-minimizing element of $\MH$. 

\subsection{Weak ERM oracle: binary-valued labels}
\begin{algorithm}
  \caption{Finding the ERM minimizer on a sample from a weak ERM oracle}
  \label{alg:sample-erm-binary}
\begin{algorithmic}[1]\onehalfspacing
\Require Concept class $\MH \subset \MY^\MX$, weak ERM oracle $\Oerm$, binary-valued loss function $\ell : \MY \times \MY \to \{0,1\}$, sample $S = \{ (x_i, y_i) \}_{i=1}^n \in (\MX \times \MY)^n$. 
\Function{\SampleERMBinary}{$S, \ell, \Oerm$}
\State Set $\MI \gets [n]$.
\While{There is $i \in \MI$ so that $\Oerm(\{ x_j, y_j') \}_{j \in \MI} > \Oerm(\{ (x_j, y_j') \}_{j \in \MI \backslash \{ i \}})$}
\State Remove such $i$ from $\MI$.
\EndWhile
\State For each $i \in [n]$, set $z_i \gets \One{i \not \in \MI}$.
\State Return $(z_1, \ldots, z_n)$. 
\EndFunction
\end{algorithmic}

\end{algorithm}

We begin with the case of binary-valued loss functions.
\begin{lemma}
  \label{lem:weak-to-sample-binary}
Consider a concept class $\MH \subset \MY^\MX$ and a binary-valued loss function $\ell : \MY \times \MY \to \{0,1\}$. Let $\Oerm$ be a weak ERM oracle for the class $\MH$ and loss function $\ell$ (\cref{def:weak-erm}). Then for any dataset $S = \{ (x_i, y_i) \}_{i \in [n]} \in (\MX \times \MY)^n$, the algorithm $\SampleERMBinary(S, \ell, \Oerm)$ (\cref{alg:sample-erm-binary}) makes $O(n^2)$ calls to $\Oerm$ and outputs a vector $(z_1, \ldots, z_n) \in \{0,1\}^n$ so that, for some empirical minimizer $h^\st = \argmin_{h \in \MH} \sum_{i=1}^n \ell(h(x_i), y_i)$, we have $z_i = \ell(h^\st(x_i), y_i)$ for all $i \in [n]$.
\end{lemma}
\begin{proof}
  Fix $\MH, \ell, \Oerm$, and $S$. %
  Let the number of iterations of the while loop be denoted $N$.
  For $0 \leq t \leq N$, let $\MI_t$ denote the value of the set $\MI$ in $\SampleERMBinary(S, \ell, \Oerm)$ directly after round $t$ (so that, in particular, $\MI_0 = [n]$). Note that each iteration of the while loop, $\MI$ decreases in size by 1. Moreover, on round $t$ of the while loop, letting $i_t$ denote the chosen $i \in \MI_{t-1}$, we have
  \begin{align}
\min_{h \in \MH} \sum_{j \in \MI_{t-1}} \ell(h(x_j), y_j) > \min_{h \in \MH} \sum_{j \in \MI_{t-1} \backslash i_t} \ell(h(x_j), y_j) = \min_{h \in \MH} \sum_{j \in \MI_t} \ell(h(x_j), y_j) \geq \min_{h \in \MH} \sum_{j \in \MI_{t-1}} \ell(h(x_j), y_j) - 1\nonumber.
  \end{align}
  It follows that for $1 \leq t \leq N$, $\min_{h \in \MH} \sum_{j \in \MI_{t}} \ell(h(x_j), y_j) = \min_{h \in \MH} \sum_{j \in [n]} \ell(h(x_j), y_j) -t$.

  Note also that we must have $\min_{h \in \MH} \sum_{j \in \MI_N} \ell(h(x_j), y_j) = 0$, as otherwise we could remove some $i$ from $\MI_N$ and decrease the empirical loss. Thus, $N = \min_{h \in \MH} \sum_{j \in [n]} \ell(h(x_j), y_j)$.
Moreover,  there is some $h^\st \in \MH$ so that $\ell(h(x_j), y_j) = 0$ for each $j \in \MI_N$.

If any $i \in [n] \backslash \MI_N$ satisfies $\ell(h^\st(x_i), y_i) = 0$, then we would have $\sum_{j \in [n]} \ell(h(x_j), y_j) < n-|\MI_N| = N$, which is a contradiction. Thus, $\ell(h^\st(x_i), y_i) = \One{i \not\in \MI_N}$ for all $i \in [n]$, as desired.

The total number of oracle calls made in \cref{alg:sample-erm-binary} is at most $2n^2$: we certainly have $N \leq n$, and each round of the while loop requires at most $|\MI| + 1 \leq 2n$ calls to $\Oerm$. 
\end{proof}

\subsection{Weak ERM oracle: real-valued labels}
Next, we prove an analogue of \cref{lem:weak-to-sample-binary} for real-valued loss functions; to keep the oracle complexity bounded, we need to tolerate some approximation error.
\begin{lemma}
  \label{lem:weak-to-sample-real}
Consider a concept class $\MH \subset [0,1]^\MX$. Let $\Oerm$ be a weak ERM oracle for the class $\MH$ and absolute loss function $\labs$ (\cref{def:weak-erm}). Then for any dataset $S = \{ (x_i, y_i) \}_{i \in [n]} \in ( \MX \times [0,1])^n$ and $\discmg \in (0,1)$, the algorithm $\SampleERMReal(S, \ell, \Oerm, \discmg)$ (\cref{alg:sample-erm-real}) makes $O(n/\discmg)$ calls to $\Oerm$ of length $O(n)$ and outputs a vector $(z_1, \ldots, z_n) \in [0,1]^n$ so that, for some $h^\st \in \MH$ satisfying $\sum_{i=1}^n \labs(h^\st(x_i), y_i) \leq \inf_{h \in \MH} \sum_{i=1}^n \labs(h(x_i), y_i)$, we have $|z_i - \ell(h^\st(x_i), y_i)| \leq \discmg$ for all $i \in [n]$. 
\end{lemma}
\begin{proof}
  Fix $\MH, \Oerm, S, \discmg$. %
  For $0 \leq i \leq n$, let $\tilde S\^i$ denote the value of $\tilde S$ directly after round $i$, so that, in particular, $\tilde S\^0= \emptyset$, and $\tilde S\^i$ consists of $n \cdot \lceil 1/\discmg \rceil$ copies of $(x_j, y_j')$, for each $j \in [i]$. For $1 \leq i \leq n$, let $\Delta\^i$ denote the value of $\Delta$ defined on round $i$. We show the following claim: %
  \begin{lemma}
    \label{lem:inductive-hi-tool}
    For each $0 \leq i \leq n$, the following properties hold for any empirical risk minimizer $h\^i \in \argmin_{h \in \MH} \sum_{(x,y) \in S\^i \cup \{ (x_j, y_j) \}_{j \in [n]}} \labs(h(x), y)$: 
    \begin{enumerate}
    \item For each $j \leq i$, $h\^i(x_j) \in [y_j', y_j' + \discmg]$.
    \item $h\^i$ is an empirical risk minimizer for $S\^{i-1} \cup \{ (x_j, y_j) \}_{j \in [n]}$. 
    \end{enumerate}
  \end{lemma}
  \begin{proof}[Proof of \cref{lem:inductive-hi-tool}]
    We prove the claim by induction on $i$, %
    noting that there is nothing else to establish for the base case $i=0$. To establish the inductive step, suppose that both claims hold at steps $j < i$, for some $i \in [n]$. Let us write $V_{i-1} := \sum_{(x,y) \in \tilde S\^{i-1} \cup \{ (x_j, y_j) \}_{j \in[n]}} \labs(h\^{i-1}(x), y)$. Taking $\ell = \lfloor h\^{i-1}(x_i) / \discmg \rfloor$ yields
    \begin{align}
V_{i, \ell} \leq \sum_{(x,y) \in \tilde S\^{i-1} \cup \{ (x_j, y_j) \}_{j \in[n]} \cup \{ (x_i, \discmg \ell), (x_i, \discmg(\ell+1)) \}} \labs(h\^{i-1}(x), y) \leq  V_{i-1} +  \discmg\nonumber,
    \end{align}
    which yields that $V_{i, \ell_i^\st} \leq V_{i-1} + \discmg$. On the other hand, since any function $h$ satisfies $\sum_{(x,y) \in   \{ (x_i, \discmg \ell), (x_i, \discmg(\ell+1)) \}} \labs(h(x), y) \geq \discmg$, we must have that $V_{i,\ell} \geq V_{i-1} + \discmg$ for each $\ell$. It follows that $V_i = V_{i,\ell_i^\st} = V_{i-1} +  \discmg$, and that any empirical risk minimizer $h\^i$ for $\tilde S\^i \cup \{ (x_j, y_j) \}_{j \in [n]}$ satisfies the following two properties:
    \begin{itemize}
    \item $h\^i(x_i) \in [\discmg \ell_i^\st, \discmg (\ell_i^\st+1)] = [y_i', y_i' + \discmg]$.
    \item $h\^i$ is an empirical risk minimizer on $\tilde S\^{i-1} \cup \{ (x_j, y_j) \}_{j \in [n]}$. 
    \end{itemize}
    Thus the second claim of the lemma statement holds at step $i$. Moreover, using the inductive hypothesis together with the first item above, we see that $h\^i(x_j) \in [y_j', y_j' + \discmg]$ for all $j < i$. 
  \end{proof}
  By  \cref{lem:inductive-hi-tool}, any empirical risk minimizer $h\^n$ for $S\^n \cup \{ (x_j, y_j) \}_{j \in [n]}$ satisfies $h\^n(x_i) \in [y_i', y_i' + \discmg]$ for each $i \in [n]$ and moreover is also an empirical risk minimizer for $S$. This establishes the claim of \cref{lem:weak-to-sample-real}. 
\end{proof}

\begin{algorithm}
  \caption{Finding the ERM minimizer on a sample from a weak ERM oracle}
  \label{alg:sample-erm-real}
\begin{algorithmic}[1]\onehalfspacing
\Require Concept class $\MH \subset [0,1]^\MX$, weak ERM oracle $\Oerm$, sample $S = \{ (x_i, y_i) \}_{i=1}^n \in (\MX \times [0,1])^n$, accuracy parameter $\discmg$ so that $1/\discmg \in \BN$. 
\Function{\SampleERMReal}{$S, \discmg, \Oerm$}
\State Initialize $\tilde S \gets \emptyset$. 
\For{$1 \leq i \leq n$}
\For{$0 \leq \ell \leq 1/\alpha -1 $}
\State Set $V_{i,\ell} \gets \Oerm(\tilde S \cup \{ (x_j, y_j) \}_{j \in [n]} \cup \{(x_i, \alpha \ell), (x_i, \alpha (\ell+1))\})$. %
\EndFor
\State Define $\ell_i^\st := \argmin_{0 \leq \ell \leq 1/\alpha - 1l } \{ V_{i, \ell} \}$, and $y_i' := \alpha \cdot \ell_i^\st$. 
\State Add $(x_i, y_i')$ and $(y_i' + \alpha)$ to $\tilde S$. 
\EndFor
\State \Return the vector $(y_1', \ldots, y_n')$. 
\EndFunction
\end{algorithmic}
\end{algorithm}

\begin{algorithm}
  \caption{Implementing a real-valued consistency oracle with range queries}
  \label{alg:sample-con-real}
  \begin{algorithmic}[1]\onehalfspacing
\Require Concept class $\MH \subset [0,1]^\MX$, weak ERM oracle $\Oerm$, sample $S = \{ (x_i, \ell_i, u_i) \}_{i \in [n]} \in (\MX \times [0,1]^2)^n$, $i \in [n]$.
\Function{\SampleConReal}{$(S, \Oerm)$}
\State Set $S' := \bigcup_{i \in [n]} \{ (x_i, \ell_i), (x_i, u_i) \}$. 
\State Set $V \gets \Oerm(S')$.
\State \Return $\True$ if $V \leq \sum_{i =1}^n (u_i - \ell_i)$, else $\False$.
  \EndFunction
  \end{algorithmic}
\end{algorithm}

\begin{lemma}
  \label{lem:sample-con-real}
Consider a concept class $\MH \subset [0,1]^\MX$ equipped with a weak ERM oracle $\Oerm$. Then for any $S = \{ (x_i, \ell_i, u_i) \}_{i \in [n]} \in (\MX \times [0,1]^2)^n$ with $\ell_i \leq u_i$ for all $i$, the algorithm $\SampleConReal(S, \Oerm)$ (\cref{alg:sample-con-real}) outputs $\True$ if and only if there is some $h \in \MH$ satisfying $\ell_i \leq h(x) \leq u_i$ for all $i \in [n]$.
\end{lemma}
\begin{proof}
  The lemma statement is immediate from the fact that there is $h \in \MH$ satisfying $\ell_i \leq h(x) \leq u_i$ if and only if
  \begin{align}
\inf_{h \in \MH} \sum_{i=1}^n |h(x_i) - \ell_i| + |h(x_i) - u_i| = \sum_{i=1}^n (u_i - \ell_i)\nonumber.
  \end{align}
\end{proof}

\section{Proof of \cref{thm:weak-oig}}
\label{sec:weak-oig-proof}

\subsection{Properties of the generating function}
\label{sec:genfun-props}
Given $m, \MS \subset V_m = \{0,1\}^m$, and $v \in V_m$, recall the definition of the hitting time $\tau_{\MS,v}$ in \cref{eq:stop-at-s}. We begin by proving the following basic recursive property of the generating function $M_{\MS, v}(\gamma)$ of the random walk on the hypercube graph $G_m$ defined in \cref{sec:weak-learner}. 
\begin{lemma}
  \label{lem:recursion}
  Suppose $\MW \subset \{0,1\}^m$ is given, and consider the $m$-dimensional hypercube graph $G_m = (V,E)$. Then %
  the following holds for all $v \in \MW$:
  \begin{align}
M_{\Wcomp, v}(\gamma) = \frac{\gamma}{(2-\gamma)m} \sum_{i \in [m]} M_{\Wcomp, v^{\oplus i}}(\gamma)\nonumber.
  \end{align}
\end{lemma}
\begin{proof}
For any $v \in \MW$ and $t > 0$, we have
  \begin{align}
\Pr(\tau_{\Wcomp,v} = t) = \frac 12 \cdot \Pr(\tau_{\Wcomp,v} = t-1) + \frac 12 \sum_{i=1}^m \frac 1m \cdot \Pr(\tau_{\Wcomp, v^{\oplus i}} = t-1)\nonumber,
  \end{align}
  where we have used the fact that for $v \in \MW$, each of the $m$ edges containing $v$, indexed by $i \in [m]$ has two vertices, namely $v$ and $v^{\oplus i}$. Moreover, for $v \in \MW$, we have that $\Pr(\tau_{\Wcomp, v} = 0) = 0$. Thus, for $\gamma \in (0,1)$, we have
  \begin{align}
    M_{\Wcomp,v}(\gamma) =&  \sum_{t \geq 1} \gamma^t \cdot \left( \frac 12 \cdot \Pr(\tau_{\Wcomp,v} = t-1) + \frac 12 \sum_{i=1}^m \frac 1m \cdot \Pr(\tau_{\Wcomp, v^{\oplus i}} = t-1) \right)\nonumber\\
    =& \frac{\gamma}{2} \cdot M_{\Wcomp, v}(\gamma) + \frac{\gamma}{2m} \sum_{i=1}^m \sum_{t \geq 1} \gamma^{t-1} \cdot \Pr(\tau_{\Wcomp, v^{\oplus i}} = t-1) \nonumber\\
    =& \frac{\gamma}{2} \cdot M_{\Wcomp,v}(\gamma) + \frac{\gamma}{2m} \sum_{i=1}^m M_{\Wcomp, v^{\oplus i}}(\gamma)\nonumber.
  \end{align}
  Rearranging, we see that
  \begin{align}
M_{\Wcomp,v}(\gamma) = \frac{\gamma}{(2-\gamma)m} \sum_{i \in [m]} M_{\Wcomp, v^{\oplus i}}(\gamma)\nonumber,
  \end{align}
  as desired.
\end{proof}

\cref{lem:outdeg-m-bound} establishes an upper bound on the outdegree of the orientation $\sigma_{F,\lambda}$ (defined in  \cref{eq:define-orientation}) induced by the function $F(v) := M_{\Wcomp, v}(\gamma)$. 
\begin{lemma}
  \label{lem:outdeg-m-bound}
  Given $m \in \BN$, $\gamma \in (0,1)$, $\lambda \in [0,1]$, and $\MW \subset \{0,1\}^m$, %
  write $F(v) := M_{\Wcomp, v}(\gamma)$ for $v \in \{0,1\}^m$. Then the induced orientation $\sigma_{F,\lambda}$ satisfies
  \begin{align}\outdeg(\sigma_{F,\lambda}) \leq \frac{m}{2} -  (1-\gamma) \lambda m \cdot \min_{v \in V} F(v) .\nonumber\end{align}
\end{lemma}
\begin{proof}
Consider any $v \in \MW$. As a consequence of \cref{lem:recursion}, we have
  \begin{align}
m \cdot F(v) - \sum_{i \in [m]} F(v^{\oplus i}) = -m \cdot \frac{2(1-\gamma)}{\gamma} \cdot F(v)\label{eq:used-recursion-lemma}. 
  \end{align}

  We compute
  \begin{align}
    \outdeg(v; \sigma_{F,\lambda}) =& \sum_{i \in [m]} (1 - \sigma_{F,\lambda}((i, e_{i,v}))(v))\nonumber\\
    =& \sum_{i \in [m]} \frac{1 - \lambda \cdot (F(v^{\oplus i}) - F(v))}{2}\nonumber\\
    =& \frac{m}{2} + \frac{\lambda m}{2} \cdot F(v) - \frac{\lambda}{2}  \sum_{i \in [m]} F(v^{\oplus i})\nonumber\\
    =& \frac{m}{2} - \frac{(1-\gamma)\lambda m}{\gamma} \cdot F(v) \leq \frac m2 - (1-\gamma) \lambda m \cdot F(v)\label{eq:not-in-d-bnd},
  \end{align}
  where the first equality uses that if $v^{\oplus i} \in \Wcomp$, then $1-\sigma_{F,\lambda}((i, e_{i,v}))(v)  = 0\leq \frac{1-\lambda (1-F(v))}{2}$, and the final equality uses \cref{eq:used-recursion-lemma}.
\end{proof}

Next, \cref{lem:min-m-bound} lower bounds the function $M_{\Wcomp, v}(\gamma)$, which is needed to apply \cref{lem:outdeg-m-bound}. 
\begin{lemma}
  \label{lem:min-m-bound}
  Let $m \geq 4$ be an integer, $\gamma \in (0,1)$, and $\MW \subset \{0,1\}^m$ be given. Then, if $\frac{1}{1-\gamma} \geq 4m \log m$ and  $|\MW| \leq \frac{1}{m} 2^{m-2}$, 
  \begin{align}
\min_{v \in \MW} M_{\Wcomp, v}(\gamma) \geq \frac{1}{4e}\nonumber.
  \end{align}
\end{lemma}
\begin{proof}
 Suppose for the purpose of contradiction that there is some $v \in \MW$ for which $M_{\Wcomp,v}(\gamma) < 1/(4e)$. Let $X_v\^0, X_v\^1, \ldots \in V$ denote the lazy random walk on the $m$-dimensional hypercube, $G_m$, started at $v$. Since $\gamma^{1/(1-\gamma)} \geq 1/e$ for all $\gamma < 1$, we have that
  \begin{align}
\frac{1}{e} \Pr(\tau_{\Wcomp,v} \leq \lfloor 1/(1-\gamma) \rfloor) = \sum_{0 \leq t \leq \lfloor 1/(1-\gamma) \rfloor} \gamma^t \cdot \Pr(\tau_{\Wcomp,v} = t) \leq M_{\Wcomp,v}(\gamma)\label{eq:md-e-bound}.
  \end{align}
  Let us write $L := \lfloor 1/(1-\gamma) \rfloor$ and $\tau = \tau_{\Wcomp,v}$. Note that the distribution of $X_v\^0, \ldots, X_v\^\tau$ is exactly the distribution of a lazy random walk $Y_v\^0, \ldots, Y_v\^\tau$ on the hypercube $\{0,1\}^m$, up to the stopping time $\tau$. Let $U$ denote a uniformly distributed random variable on $\{0,1\}^m$. By \cref{lem:hypercube-mixing} together with the fact that $L \geq 1/(2(1-\gamma)) \geq 2m \log(m) \geq m \log m + m \log(4)$, we have that $\tvd{Y_v\^L}{U} \leq 1/4$. Let $\bar X_v\^t := X_v\^{t \wedge \tau}$ denote the stopped random walk, with respect to the stopping time $\tau$.

  Consider a coupling between the distributions $\{ \bar X_v\^t \}_{t \geq 0}$ and $\{ Y_v\^t \}_{t \geq 0}$ so that $\bar X_v\^t = Y_v\^t$ for all $t \leq \tau$, almost surely. Since $\tau > L$ with probability at least $1-e \cdot M_{\Wcomp,v}(\gamma)$ by \cref{eq:md-e-bound}, we have that $\Pr(\bar X_v\^L = Y_v\^L) \geq \Pr(\tau \geq L) \geq 1-e\cdot M_{\Wcomp,v}(\gamma)$, where the probability is with respect to the coupling. It follows that $\tvd{\bar X_v\^L}{Y_v\^L} \leq e \cdot M_{\Wcomp,v}(\gamma) < 1/4$. By the triangle inequality, we have
  \begin{align}
\tvd{\bar X_v\^L}{U} \leq  e \cdot M_{\Wcomp,v}(\gamma) + 1/4 < 1/2.\nonumber
  \end{align}
Let $\bar N(\MW) := \{ v \in \{0,1\}^m \ : \ v \in \MW \mbox{ or } \exists i \mbox{ s.t. } v^{\oplus i} \in \MW \}$ denote the union of $\MW$ and its neighborhood.  Thus, we must have $\supp(\bar X_v\^L) > 2^{m-1}$, which contradicts $|\MW| \leq \frac 1m 2^{m-2}$ since $\supp(\bar X_v\^L) \subset \bar N(\MW)$, and $\bar N(\MW) \leq 2m | \MW| \leq 2^{m-1}$. 
  
\end{proof}

\subsection{Transductive learning guarantee}

\begin{proof}[Proof of \cref{thm:weak-oig}]
Let us write $\gamma := 1- \frac{1}{C_1 m \log m}$. Set $\ep = \frac{1-\gamma}{16e}$ and $L := \lceil \log(2/\ep)/\log(1/\gamma) \rceil = \lceil \log(32e/(1-\gamma))/\log(1/\gamma) \rceil \leq O(\log(1/(1-\gamma))/(1-\gamma))$.  Moreover, write $\delta = \ep/2$ and $U = C_1 m^2 \log^3 m$; note that $U = \Theta(\frac{\log(1/\delta)}{\ep^2})$.

  Let us write $X = (x_1, \ldots, x_m) \in \MX^{m}$, $y = (y_1, \ldots, y_m) \in \{0,1\}^m$. Let $\MW := \MH|_X$ be the projection of $\MH$ onto $\MX$ and $G_m = (V,E)$ denote the $m$-dimensional hypercube graph. For $y' \in \MW$, recall the definition of the stopping time $\tau_{\Wcomp, y'}$ in \cref{eq:stop-at-s}. Note that, for any $y' \in V$ and each $u \in U$, the random variable $T_u$ constructed in $\EstimatePotential(X, y', U, \gamma, \Ocon)$ is distributed exactly according to $L \wedge \tau_{\Wcomp, y'}$. Thus, for any $\delta, U$ satisfying $U \geq C \log(1/\delta)/\ep^2$ for a sufficiently large constant $C$, we have from Hoeffding's inequality that with probability $1-\delta$, 
  \begin{align}
\left| \frac{1}{U} \sum_{u=1}^U \gamma^{T_u} - \E[\gamma^{L \wedge \tau_{\Wcomp, y'}}] \right| \leq & \ep/2\label{eq:estpot-1}.
  \end{align}
  Moreover, by our choice of $L$, we have that, almost surely,
  \begin{align}
\left| \gamma^{L \wedge \tau_{\Wcomp, y'}} - \gamma^{\tau_{\Wcomp,y'}} \right| \leq \gamma^L \leq \ep/2.\label{eq:estpot-2}
  \end{align}
  Thus, combining \cref{eq:estpot-1,eq:estpot-2}, with probability at least $1-\delta$, the output of $\EstimatePotential(X, y', U, \gamma, \Ocon)$ satisfies
  \begin{align}
\left| \frac{1}{U} \sum_{u=1}^U \gamma^{T_u} - M_{\Wcomp, y'}(\gamma) \right|  = \left| \frac{1}{U} \sum_{u=1}^U \gamma^{T_u} - \E[\gamma^{\tau_{\Wcomp, y'}}]\right|\leq \ep\label{eq:estpot-3}.
  \end{align}

For $y' \in \MW$, define $F(y' ) := M_{\Wcomp, y'}(\gamma)$.  For each $i \in [m]$, write $y^{i,0} = (y_{-i}, 0)$ and $y^{i,1} = (y_{-i}, 1)$. %
Now consider $i \in [m]$ for which $y^{i,0}, y^{i,1} \in \MW$. Note that $\WeakRealizable(S_{-i}, x_i, \gamma, \lambda, \Ocon)$ calls $\EstimatePotential(X, y^0, U, \gamma, \Ocon)$ and $\EstimatePotential(X, y^{1}, U, \gamma, \Ocon)$. These calls return values $\hat F(y^{i,0}), \hat F(y^{i,1}) \in [0,1]$ respectively.     Since can ensure, by our choices of the values $U, \delta$ above, that $U \geq C \log(1/\delta)/\ep^2$ (by making $C_1$ sufficiently large), it follows by \cref{eq:estpot-3} and a union bound that with probability at least $1-2\delta$, for each $b \in \{0,1\}$,
  \begin{align}
\left| \hat F(y^{i,b}) - F(y^{i,b}) \right| = \left| \hat F(y^{i,b}) - M_{\Wcomp, y^{i,b}}(\gamma) \right| \leq \ep\nonumber.
  \end{align}

  Thus, with probability at least $1-2\delta$, the output $\hat y_i := \frac{1 +  (\hat F(y^{i,0}) - \hat F(y^{i,1}))}{2}$ of $\WeakRealizable(S_{-i}, x_i, \gamma, 1, \Ocon)$ satisfies
  \begin{align}
    \left| \hat y_i - \sigma_{F,1}( e_{i,y})(y^{i,1}) \right| =& \left| \frac{1 +   (\hat F(y^{i,0}) - \hat F(y^{i,1}))}{2} - \frac{1 +  (F(y^{i,0}) - F(y^{i,1}))}{2} \right|\nonumber\\
    \leq & \frac{1}{2} \cdot |\hat F(y^{i,0}) - F(y^{i,0})| + \frac{1}{2} \cdot |\hat F(y^{i,1}) - F(y^{i,1})| \leq \ep\nonumber,
  \end{align}
  and thus, using that $y = y^{i, y_i}$,
  \begin{align}
| \hat y_i - y_i | = |\hat y_i- \sigma_{F,1}(e_{i,y})(y^{i,1}) |  + |\sigma_{F,1}( e_{i,y})(y^{i,1}) - y_i| \leq \ep + (1 - \sigma_{F,1}( e_{i,y})(y))\nonumber.
  \end{align}
  By our choice of $\delta = {\ep/2}$, it follows that, for each $i \in [m]$,
  \begin{align}
\E \left[ \lbin(\SA(S_{-i}; x_i),  y_i) \right] =& \E\left[ |\hat y_i - y_i | \right] \leq 2\delta + \ep + (1-\sigma_{F,1}( e_{i,y})(y)) = 2\ep + (1-\sigma_{F,1}( e_{i,y})(y))\label{eq:sasi-yi-bnd}.
  \end{align}
  
  Next, the Sauer-Shelah lemma (\cref{lem:ss}) gives that $|\MW| = |\MH|_X| \leq (em/d)^d \leq \frac 1m 2^{m-2}$, since we have chosen $m \geq C_1d \log d$ for a sufficiently large constant $C_1$. 
  Thus, by our choice of $\gamma = 1-\frac{1}{C_1 m\log m}$  and \cref{lem:min-m-bound} we have that $\min_{v \in \MW} F(v) = \min_{v \in \MW} M_{\Wcomp, v}(\gamma) \geq 1/(4e)$.  We may now compute
\begin{align}
\sum_{i \in [m]} \E[\lbin(\SA(S_{-i}; x_i) , y_i)] \leq &2\ep m +  \sum_{i \in [m]} (1 - \sigma_{F,1}( e_{i,y})(y))   \nonumber\\
  =&2\ep m + \outdeg(y; \sigma_{F,1})\nonumber\\
  \leq &  2\ep m + \frac{m}{2} - \frac{(1-\gamma)  m}{4e} \leq \frac{m}{2} - \frac{(1-\gamma)m}{16e}\nonumber,
\end{align}
where the first inequality uses \cref{eq:sasi-yi-bnd}, the second inequality uses \cref{lem:outdeg-m-bound}, and the final inequality uses the choice of $\ep = \frac{1-\gamma}{16e}$. 
\end{proof}

Finally, for use in applying \cref{thm:weak-oig}, we state the following standard lemma, which relates the transductive error of a learning algorithm $\SA$ to its expected error with respect to any realizable distribution. 
\begin{lemma}[Leave-one-out]
  \label{lem:loo}
  Let $P \in \Delta(\MX \times \{0,1\})$ be $\MH$-realizable and $m \in \BN$ be given. Furthermore, let $\SA(\cdot, \cdot) : (\MX \times \{0,1\})^{m-1} \times \MX \to \{0,1\}$ be a (possibly randomized) mapping which takes as input a dataset of size $m-1$ and a point in $\MX$, and outputs a real number. Then
  \begin{align}
\EE_{S' \sim P^{m-1}} \EE_{(X,Y) \sim P} \EE[\lbin(\SA(S', X), Y)] = \EE_{\substack{S \sim P^m \\ S = \{ (x_i, y_i) \}_{i\in [m]}}}\left[ \frac 1m \sum_{i=1}^m \E[\lbin(\SA(S_{-i}, x_i),  y_i )] \right]\nonumber,
  \end{align}
  where the inner expectation is over the randomness in $\SA$. 
\end{lemma}
\begin{proof}
  Let $(x_1, y_1), \ldots, (x_m, y_m), (X,Y)$ denote i.i.d.~samples from $P$, and write $S' = \{ (x_i, y_i) \}_{i \in [m-1]}, S = \{ (x_i, y_i) \}_{i \in[m]}$. By exchangeability of these samples and linearity of expectation, we have
  \begin{align}
    \E_{S'} \E_{(X,Y) \sim P} \E[\lbin(\SA(S', X), Y)] =& \E_{S'} \E_{(x_m, y_m) \sim P} \E[\lbin(\SA(S', x_m), y_m)]\nonumber\\
    =& \frac 1m \sum_{i=1}^m \E_S \E[\lbin(\SA(S_{-i}, x_i), y_i)]\nonumber\\
    =& \E_S \left[ \frac 1m \sum_{i=1}^n \E[\lbin(\SA(S_{-i}, x_i), y_i)] \right]\nonumber.
  \end{align}
\end{proof}
\section{Boosting}
\label{sec:boosting}
In this section, we discuss the technique of boosting, which is used to upgrade a weak learner (in the sense of \cref{def:weak-learner}) to a strong learner, i.e., one which achieves error at most an arbitrary threshold $\ep \in (0,1)$ with high probability. 
Notice that we allow a weak learner to only possess its guarantee \cref{eq:weak-learning} \emph{in expectation}, rather than with high probability. To account for this weaker assumption, it is necessary to slightly modify standard boosting results \cite[Chapters 3 \& 4]{schapire2012boosting}, as stated in \cref{lem:adaboost-sample} below. The main difference in the proofs is the use of an appropriate martingale concentration inequality (namely, \cref{lem:freedman}) to deal with the deviations in errors exhibited by the individual calls to the weak learner by the boosting algorithm. \cref{lem:adaboost-sample} gives a bound on the training error of the \Adaboost algorithm (\cref{alg:adaboost}).

\begin{algorithm}
  \caption{\Adaboost (Algorithm 1.1 of \cite{schapire2012boosting})}
  \label{alg:adaboost}
\begin{algorithmic}[1]\onehalfspacing
  \Require Input dataset $\{ (x_i, y_i) \}_{i \in [n]} \subset (\MX \times \{0,1\})^n$, randomized weak learner $\SA$, number of time steps $T$. 
  \State Initialize $D_1 := \Unif([n]) \in \Delta([n])$.
  \For{$1 \leq t \leq T$}
  \State Sample a fresh string of uniform bits $R_t$ for use in $\SA$ and to sample $S_t$ in \cref{line:sample-st} below.
  \State \label{line:sample-st} Sample an i.i.d.~dataset $S_t$ of size $m$ from the distribution of $x_i,\ i \sim D_t$, so that $S_t \in (\MX \times \{ 0, 1 \})^m$.
  \State Let $h_t : \MX \to \{ 0,1 \}$ be the output of $\SA_{R_t}(S_t, \cdot)$.
  \State \label{line:def-eps-alpha} Define $\ep_t := \Pr_{i \sim D_t}(h_t(x_i) \neq y_i)$, and $\alpha_t := \frac 12 \ln \left( \frac{1-\ep_t}{\ep_t} \right)$.
  \State For $i \in [n]$, define
  \begin{align}
D_{t+1}(i) := \frac{D_t(i) \cdot \exp(-\alpha_t \cdot (2y_i-1) \cdot (2 h_t(x_i)-1))}{Z_t},\label{eq:dtp1-define}
  \end{align}
  where $Z_t := \sum_{j \in [n]} D_t(j) \cdot \exp(-\alpha_t \cdot (2y_j-1) \cdot (2 h_t(x_j)-1))$.
  \EndFor
  \State \Return the hypothesis $H : \MX \to \{0,1\}$, where $H(x) := \frac 12 + \frac 12 \sign \left( \sum_{t=1}^T \alpha_t \cdot (2h_t(x)-1) \right)$. 
\end{algorithmic}
\end{algorithm}

\begin{lemma}[Training error of \Adaboost]
  \label{lem:adaboost-sample}
  Let $m,n \in \BN$ and $\eta \in (0,1)$ be given, and suppose algorithm $\SA$ is an $m$-sample weak learner with margin $\eta$ for the class $\MH$. Let $\bar S \in (\MX \times \{0,1\})^n$ be an $\MH$-realizable sample. Then if \cref{alg:adaboost} is run for $T \geq \lceil 16\log(2n/\delta)/\eta^2 \rceil$ rounds on $\bar S$, the output hypothesis $H(x)$ satisfies $\eerr_S(H) = 0$ with probability $1-\delta$.

  Moreover, for $x \in \MX$, to compute $H(x)$, one must only call $\SA(S_t, x)$ for $T$ different choices of datasets $S_t \in (\MX \times \{0,1\})^m$. 
\end{lemma}
\begin{proof}
  The proof follows closely to that in, e.g., \cite[Chapter 3]{schapire2012boosting}, with minor modifications. We use the notation in \cref{alg:adaboost}. Define $F(x) := \sum_{t=1}^T \alpha_t h_t(x)$. From the definition of $Z_t$ in \cref{alg:adaboost} we have, for $i \in [m]$, 
  \begin{align}
D_{T+1}(i) =& \frac{D_1(i) \cdot \exp(-y_i \cdot F(x_i))}{Z_1 \cdots Z_T}\nonumber.
  \end{align}
  For all $(x,y) \in \MX \times \{\pm 1 \}$, we have $\One{H(x) \neq y} \leq e^{-F(x) \cdot y}$. We can thus bound the training error over the dataset $\{ (x_i, y_i) \}_{i \in [n]}$ as follows:
  \begin{align}
    \frac 1n \sum_{i=1}^n \One{H(x_i) \neq y_i} \leq  \sum_{i=1}^n D_1(i) \cdot e^{-F(x_i) \cdot y_i} = \sum_{i=1}^n D_{T+1}(i) \cdot (Z_1 \cdots Z_T) = Z_1 \cdots Z_T\label{eq:training-err-bound}.
  \end{align}
  By the choice of $\alpha_t = \frac 12 \ln \left( \frac{1-\ep_t}{\ep_t} \right)$ in \cref{line:def-eps-alpha} of \cref{alg:adaboost}, we have
  \begin{align}
Z_t = \sum_{j=1}^n D_t(j) \cdot e^{-\alpha_t \cdot y_j h_t(x_j)} = (1-\ep_t) \cdot e^{-\alpha_t} + \ep_t \cdot e^{\alpha_t} = 2 \sqrt{\ep_t (1-\ep_t)} = \sqrt{1-4\gamma_t^2}\nonumber,
  \end{align}
  where the second equality uses the fact that $\sum_{j : y_j h_t(x_j) = 1} D_t(j) = 1-\ep_t$ and $\sum_{j : y_j h_t(x_j) = -1} D_t(j) = \ep_t$, and we have written $\gamma_t := 1/2 - \ep_t$ for the final equality.

  For each $t \in [T]$, let $\MF_t$ denote the sigma-algebra generated by $\{ (S_s, h_s) \}_{1 \leq s \leq t}$. Note that $D_t$ is measurable with respect to $\MF_{t-1}$. Since the distribution over $(x_i, y_i)$, for $i \sim D_t$, is $\MH$-realizable, the fact that $\SA$ is an $m$-sample weak learner with margin $\eta$ yields that for each $t$,
  $ \E\left[ \ep_t \mid \MF_{t-1} \right] \leq \frac 12 - \eta$, i.e., $\E[\gamma_t \mid \MF_{t-1}] \geq \eta$. By Jensen's inequality it follows that $\E[\gamma_t^2 \mid \MF_{t-1}] \geq \eta^2$. Note that $\gamma_t^2 \in [0,1]$ for all $t \in [T]$. Then by \cref{lem:freedman} with $R = 1$, there is an event $\ME$ occurring with probability $1-\delta$ so that, under $\ME$, $\sum_{t=1}^T \gamma_t^2 \geq \frac 12 T \eta^2 - 4 \log(2/\delta)$. Thus, under the event $\ME$, we have
  \begin{align}
Z_1 \cdots Z_T = \left(\prod_{t=1}^T (1-4\gamma_t^2)\right)^{1/2} \leq e^{-2(\gamma_1^2 + \cdots + \gamma_T^2)} \leq e^{-T\eta^2 + 8 \log(2/\delta)} \leq e^{-T\eta^2/2} \leq \frac{1}{2n}\label{eq:z1t-bound},
  \end{align}
  where the second-to-last inequality above holds since we have chosen $T$ to satisfy $T \geq \frac{16 \log(2/\delta)}{\eta^2}$, and the final inequality holds since $T$ also satisfies $T \geq \frac{2 \log(2n)}{\eta^2}$. Combining the above display and \cref{eq:training-err-bound}, we obtain that, under $\ME$, $\frac{1}{n} \sum_{i=1}^n \One{H(x_i) \neq y_i} \leq 1/(2n)$, which implies that $H(x_i) = y_i$ for all $i \in [n]$.
\end{proof}

\subsection{Generalization error of \Adaboost}
\label{sec:adaboost-generror}
Next, using the technique of \emph{sample compression schemes}, and in particular their connection to generalization (\cref{lem:gen-compression}), we prove that the output hypothesis $H$ of \Adaboost generalizes well. \cref{prop:adaboost-gen} carries out this argument for the realizable setting, and \cref{prop:adaboost-agnostic} does so for the agnostic setting. Even in the agnostic setting, the input dataset for \Adaboost must still be realizable by $\MH$, as the weak learner's guarantee depends on realizability. Thus, in the statement of \cref{prop:adaboost-agnostic}, we assume that \Adaboost is passed a subsample of maximum size which is $\MH$-realizable. Our final algorithm which agnostically learns partial concept classes (\cref{alg:agnostic-partial}) will find such a subsample to pass to \Adaboost using a weak ERM oracle.

\begin{proposition}[Generalization error of Adaboost -- realizable setting]
  \label{prop:adaboost-gen}
  Let $m,n \in \BN$ and $\eta, \delta \in (0,1)$ be given, and suppose algorithm $\SA$ is an $m$-sample weak learner with margin $\eta$ for the class $\MH$. Let $P \in \Delta(\MX \times \{0,1\})$ be $\MH$-realizable. Then if \cref{alg:adaboost} is run for $T = \lceil 16\log(4n/\delta)/\eta^2 \rceil$ rounds on a dataset $\bar S \sim P^n$, the random output hypothesis $H \in \{0,1\}^\MX$ satisfies the following with probability $1-\delta$: %
  \begin{align}
    \err_P(H) \leq & O \left( \frac{\log^2(n/\delta) \cdot (m + \log n)}{\eta^2 n} \right)\nonumber.
  \end{align}
  In particular, the probability is over the draw of $\bar S$. %
\end{proposition}
\begin{proof}

  We use a compression-based argument, following \cite[Chapter 4.2]{schapire2012boosting}. Let us denote the input dataset to \cref{alg:adaboost} by $\bar S \in (\MX \times \{0,1\})^n$. We consider a distribution $Q$ over sample compression schemes on $n$-sample datasets, $(\kappa, \rho)$, defined as follows. Given a dataset $\bar S$, let $S_t, \alpha_t, R_t$, for $t \in [T]$, be the random variables generated in the course of the procedure in \cref{alg:adaboost}. Since the bits $R_t$ are used for the sampling step in \cref{line:sample-st} and by $\SA$ (where the portions of $R_t$ that are used for the two tasks are independent), given a dataset $\bar S$, the random variable $(S_t, \alpha_t)_{t \in [T]}$ is a deterministic function of $(R_t)_{t \in [T]}$ and $\bar S$. 
  Then we define $(\rho, \kappa)$ to be the distributed as the following (deterministic) function of $(R_t)_{t \in [T]}$:
  \begin{itemize}
  \item $\kappa$ maps a dataset $\bar S \in (\MX \times \{0,1\})^n$ to $\kappa(\bar S)= ((S_1, \ldots, S_T), (\alpha_1, \ldots, \alpha_T))$, where $S_t, \alpha_t$ are defined as a function of $(R_t)_{t \in [T]}$ and $\bar S$, as described above. 
  \item $\rho$ maps an input of the form $((S_1', \ldots, S_T'), (\alpha_1', \ldots, \alpha_T'))$ (where $(S_1', \ldots, S_T')$ is a sequence of examples in $\MX \times \{0,1\}$ of length $Tm$ and $(\alpha_1', \ldots, \alpha_T')$ is a sequence of real numbers, encoded in binary) to the hypothesis $x \mapsto \sign \left( \sum_{t=1}^T \alpha_t \cdot \SA_{R_t}(S_t, x) \right)$.
  \end{itemize}

  Since the values $\ep_t$ in \cref{alg:adaboost} lie in $\{ 0, 1/n, 2/n, \ldots, 1 \}$, each parameter $\alpha_t$ can be encoded with $O(\log n)$ bits. Thus, with probability 1 over $(\kappa, \rho) \sim Q$, we have that the size of $\kappa$ (for input samples $\bar S$ of size $n$) is $|\kappa| \leq O(T \cdot (m+\log(n)))$. Next, \cref{lem:adaboost-sample} establishes that, for any fixed $\MH$-realizable $\bar S$, there is a set $\ME$ of compression schemes satisfying $Q(\ME) \geq 1-\delta$ so that, for all $(\kappa, \rho) \in \ME$, $\eerr_{\bar S}(\rho(\kappa(\bar S))) = 0$. (Here we crucially use that the output hypothesis of $\rho$ depends on the same random bits $R_t$ used to generate the sequence $(S_t, \alpha_t)_{t \in [T]}$.) Moreover, by \cref{lem:gen-compression} and our bound on the size of $\kappa$ drawn from $Q$, there is a constant $C > 0$ so that the following holds for any fixed $(\kappa, \rho) \in \supp(Q)$: with probability $1-\delta$ over the draw of $\bar S \sim P^n$,
  \begin{align}
\err_P(\rho(\kappa(\bar S))) \leq \One{\eerr_{\bar S}(\rho(\kappa(\bar S))) = 0} + \frac{C}{n} \cdot \left(T \log(n) \cdot (m + \log n) + \log \frac{1}{\delta} \right)\nonumber.
  \end{align}
  By our choice of $T$ together with the fact that $Q(\ME) \geq 1-\delta$, it follows that with probability $1-2\delta$ over the draw of $\bar S \sim P^n$, we have %
  \begin{align}
\err_P(\rho(\kappa(\bar S))) \leq C \cdot \frac{ \log(2n^2) \log(n) \cdot \frac{1}{\eta^2}\cdot (m + \log n) + \log(1/\delta)}{n}\nonumber.
  \end{align}
Since, for fixed $\bar S$, the distribution of $\rho(\kappa(\bar S))$ is the distribution of the output hypothesis $H$ of \cref{alg:adaboost}, the claim of the proposition follows after rescaling $\delta$. 
\end{proof}

\begin{proposition}[Generalization error of Adaboost -- agnostic setting]
  \label{prop:adaboost-agnostic}
  Let $m, n \in \BN$ and $\eta, \delta \in (0,1)$ be given, and suppose algorithm $\SA$ is an $m$-sample weak learner with margin $\eta$ for the class $\MH$. Let $P \in \Delta( \MX \times \{0,1\})$ be an arbitrary distribution. Consider a procedure which samples a dataset $\bar S \sim P^n$, deterministically chooses a subsample $\tilde S \subset \bar S$ of maximum size which is realizable by $\MH$, and then runs \cref{alg:adaboost} for $T = \lceil 16 \log(4n/\delta)/\eta^2 \rceil$ rounds on $\tilde S$. Then the output hypothesis $H(x)$ satisfies the following with probability $1-\delta$:
  \begin{align}
\err_P(H) \leq \inf_{h \in \MH} \eerr_{\bar S}(h) + O \left( \sqrt{ \frac{\log^2(n/\delta) \cdot (m + \log n)}{\eta^2 n}} \right)\nonumber.
  \end{align}
\end{proposition}
\begin{proof}
  The proof closely follows that of \cref{prop:adaboost-gen}. In particular, we consider the same distribution $Q$ over sample compression schemes on datasets with (at most) $n$-samples. We again have that with probability 1 over $(\kappa, \rho) \sim Q$, $|\kappa| \leq O(T \cdot (m + \log n))$, and that, for any fixed $\MH$-realizable $\tilde S$, there is a set $\ME$ of compression schemes satisfying $Q(\ME) \geq 1-\delta$ so that, for all $(\kappa, \rho) \in \ME$, $\eerr_{\tilde S}(\rho(\kappa(\tilde S))) = 0$.

  Next, let the deterministic mapping from samples $\bar S \in (\MX \times \{0,1\})^n$ to $\tilde S$ be denoted by $\Sigma$. For any compression scheme $(\kappa, \rho) \in \supp(Q)$, note that $(\kappa \circ \Sigma, \rho)$ is a compression scheme of size $|\kappa \circ \Sigma| \leq |\kappa| \leq O(T \cdot (m + \log n))$. 
  Thus, by \cref{lem:gen-compression}, there is a constant $C > 0$ so that, for any fixed $(\kappa, \rho) \in \supp(Q)$, with probability $1-\delta$ over the draw of $\bar S \sim P^n$,
  \begin{align}
\err_P(\rho(\kappa(\Sigma(\bar S)))) \leq \eerr_{\bar S}(\rho(\kappa(\Sigma(\bar S)))) + C \sqrt{ \frac 1n \left( T \log(n) \cdot (m + \log n) + \log(1/\delta) \right)}\nonumber.
  \end{align}
  The choice of $\Sigma$ yields that, if $\bar S \in (\MX \times \{0,1\})^n$ and  $\tilde S = \Sigma(\bar S)$, then for any hypothesis $h'$ satisfying $\eerr_{\tilde S}(h') = 0$, we have
  \begin{align}
\eerr_{\bar S}(h') \leq \frac{n - |\tilde S|}{n} = \inf_{h \in \MH} \eerr_{\bar S}(h)\label{eq:bars-tildes}
  \end{align}
 Thus, by our choice of $T$, \cref{eq:bars-tildes}, and the fact that $Q(\ME) \geq 1-\delta$, it follows that with probability $1-2\delta$ over the draw of $\bar S \sim P^n$, %
  \begin{align}
\err_P(\rho(\kappa(\Sigma(\bar S)))) \leq \inf_{h \in \MH} \eerr_{\bar S}(h) + O \left( \sqrt{ \frac{\log^2(n/\delta) \cdot (m + \log n)}{\eta^2 n}} \right)\nonumber.
  \end{align}
  Since for fixed $\bar S$, the distribution of $\rho(\kappa(\Sigma(\bar S)))$ is the distribution of the output hypothesis $H$ of \cref{alg:adaboost}, the claim follows after rescaling $\delta$. 
\end{proof}

\section{Proof of \cref{thm:partial-main}}
\label{sec:main-partial-proof}
\begin{algorithm}
  \caption{Oracle-efficient PAC learner for partial concept classes}
  \label{alg:agnostic-partial}
\begin{algorithmic}[1]\onehalfspacing
  \Require Partial concept class $\MH$, $n \in \BN$, sample $S = \{ (x_i, y_i) \}_{i=1}^n \in (\MX \times \{ 0,1 \})^n$, weak learner $\SA$ for $\MH$ with margin $\eta \in (0,1)$, failure probability $\delta \in (0,1)$, weak ERM oracle $\Oerm$ for $\MH$, domain point $x \in \MX$.
  \Function{\RealizablePartial}{$S, x, \SA, \eta, \delta$}
  \State Call \Adaboost (\cref{alg:adaboost}) on the dataset $S$ using the weak learner $\SA$ with $T = \lceil 16 \log(4n/\delta)/\eta^2 \rceil$, and denote its output hypothesis by $H : \MX \to \{0,1\}$.
  \State \Return $H(x)$.
  \EndFunction
  \Function{\AgnosticPartial}{$S,x, \SA, \eta, \delta, \Oerm$}
  \State Let $z \in \{0,1\}^n$ be the output of $\SampleERMBinary(S,x,  \lbin, \Oerm)$.\Comment{\emph{(\cref{alg:sample-erm-binary})}}
  \State Define $\tilde S := \{ (x_i, y_i) \ : \ z_i = 0 \}$.\label{line:compute-tildes-partial}
  \State Call \Adaboost (\cref{alg:adaboost}) on the dataset $\tilde S$ using the weak learner $\SA$ with $T = \lceil 16 \log(6n/\delta)/\eta^2 \rceil$, and denote its output hypothesis by $H : \MX \to \{0,1\}$.
  \State \Return $H(x)$. 
  \EndFunction
\end{algorithmic}
\end{algorithm}
The guarantees of \cref{thm:partial-main} are established with the algorithms \RealizablePartial and \AgnosticPartial (\cref{alg:agnostic-partial}). These algorithms call \Adaboost on an appropriate $\MH$-realizable dataset and use \WeakRealizable as the weak learner. We remark that \AgnosticPartial uses the algorithm \SampleERMBinary (defined in \cref{sec:oracle-implications}) to find the largest subset of labels which can be realized by the class $\MH$. Note that in the algorithms' descriptions we have stated that \Adaboost returns a hypothesis $H : \MX \to \{0,1\}$. \RealizablePartial and  \AgnosticPartial never have to compute this entire hypothesis $H$, and instead only have to evaluate $H(x)$, which, as we shall show, can be done using few calls to the oracle $\Ocon$ or $\Oerm$, respectively. 
\begin{proof}[Proof of \cref{thm:partial-main}]
  Let $\VCdim \in \BN, \ep, \delta \in (0,1)$ be given. Let $C_1, C_2$ denote the constants of \cref{thm:weak-oig}, $m := C_1 \VCdim \log \VCdim$, and $\eta := \frac{1}{C_2 m \log m}$. Let $\SA$ denote the randomized mapping $\SA : (\MX \times \{0,1\})^m \times \MX \to \{0,1\}$ which, given as input $(S,x)$, returns $\WeakRealizable(S, x, \frac{1}{C_1 m\log m}, 1, C_1 m^2 \log^3 m, \Ocon)$ (see \cref{alg:weak-oig}). By \cref{thm:weak-oig,lem:loo}, $\SA$ is an $m$-sample weak learner with margin $\eta$ for the class $\MH$.

  \paragraph{Realizable setting.}  We take $n = \frac{\VCdim^3 \cdot \log(1/\delta)}{\ep} \cdot (c\log(\VCdim \log(1/\delta)/\ep))^c$, for a sufficiently large constant $c$ as specified below. The output hypothesis $H$ of \Adaboost in \RealizablePartial is a deterministic function of $S$ and the random bits $R$ used in \Adaboost (including in its calls to $\SA$). By \cref{lem:adaboost-sample}, for any $S$, with probability $1-\delta$ over the draw of $R$, we have $\eerr_S(H) = 0$. Moreover, by \cref{prop:adaboost-gen}, with probability $1-\delta$ over the draw of $S, R$, we have
  \begin{align}
\err_P(H) \leq O \left( \frac{\log^2(n/\delta) \cdot (m + \log n)}{\eta^2 n} \right)\nonumber.
  \end{align}
  Combining the above inequality with our choice of $n$ and rescaling $\ep, \delta$, we obtain that with probability $1-\delta$ over the draw of $S \sim P^n$, $\err_P(H) \leq \ep$. 

  Finally, we analyze the oracle complexity of \RealizablePartial: to compute the value of $H(x)$, we    need to compute the values $h_t(x)$ for the hypotheses $h_t$, $t \in [T]$, defined in \Adaboost. Given the value of $S_t$, each computation of $h_t(x)$, for any $x \in \MX$, requires a single run of $\SA(S_t, x)$, which requires $\tilde O(m^3)$ calls to $\Ocon$ with datasets of size $m$ (\cref{thm:weak-oig}). In turn, the datasets $S_t$ are computed inductively as follows: given $S_t$, we can compute $h_t(x_i)$ for each $i \in [n]$, which requires $\tilde O(nm^3)$ calls to $\Ocon$. This in turn allows us to compute $D_{t+1}$ (per \cref{eq:dtp1-define}), which then allows us to sample $S_{t+1}$. %
  Thus, overall, we need $\tilde O(nm^3 \cdot T) \leq \poly(n)$ calls to $\Ocon$ to compute $H(x)$, each of which uses a dataset of size at most $n$. Hence the cumulative oracle cost is $\poly(n)$. 
\paragraph{Agnostic setting.} We take $n = \frac{\VCdim^3 \cdot \log(1/\delta)}{\ep^2} \cdot (c\log(\VCdim \log(1/\delta)/\ep))^c$, for a sufficiently large constant $c$ as specified below. The output hypothesis $H$ of \Adaboost in \AgnosticPartial is a deterministic function of $S$ as well as the random bits $R$ used in \Adaboost (including in its calls to $\SA$). 
By \cref{lem:weak-to-sample-binary}, the set $\tilde S$ constructed from $S$ in \AgnosticPartial is a subset of $S$ of maximum size which is realizable by $\MH$. Thus, by \cref{lem:adaboost-sample}, for any $S$, with probability $1-\delta$ over the draw of $R$, the hypothesis $H$ satisfies $\eerr_{\tilde S}(H) = 0$.

Moreover, by \cref{prop:adaboost-agnostic}, the (random) output hypothesis $H$ of \Adaboost satisfies the below with probability $1-\delta$ over the draw of $S,R$:
\begin{align}
\err_P(H) \leq \inf_{h \in \MH} \eerr_S(h) +  O \left( \sqrt{ \frac{\log^2(n/\delta) \cdot (m + \log n)}{\eta^2 n}} \right)\label{eq:pc-agnostic-1}.
\end{align}
Since $n \geq \frac{C \log 1/\delta}{\ep^2}$ for a sufficiently large constant $C$, McDiarmid's inequality yields that with probability $1-\delta$ over the choice of $S \sim P^n$,
\begin{align}
\inf_{h \in \MH} \eerr_S(h) \leq \E_{S' \sim P^n} \left[ \inf_{h' \in \MH} \eerr_{S'}(h') \right] + \ep \leq \err_P(\MH) + \ep\label{eq:pc-agnostic-2}.
\end{align}
Combining \cref{eq:pc-agnostic-1,eq:pc-agnostic-2} with our choice of $n$ and rescaling $\ep, \delta$, we obtain that with probability $1-\delta$ over the draw of $S \sim P^n$, we have $\err_P(H) \leq \err_P(\MH) + \ep$.

Finally, note that, to compute the value $H(x)$, for any $x \in \MX$, we need a single call to $\Oerm$ to compute $\tilde S$ (\cref{line:compute-tildes-partial} of \cref{alg:agnostic-partial}), and then need to compute the values $h_t(x)$ for the hypotheses $h_t$, $t \in [T]$, defined in \Adaboost. The oracle complexity of these calls are analyzed in the same manner as in the realizable case, so we again obtain $\poly(n)$ cumulative oracle cost. %
\end{proof}
\section{Multiclass classification}
\label{sec:multiclass}
In this section, we generalize our results on binary classification to the setting of multiclass classification. We begin by establishing that a variant of $\WeakRealizable$ yields a weak learner in the multiclass setting, in an appropriate sense. It is less clear how to define a ``weak learner'' in the multiclass setting than in the binary setting, and the literature on multiclass boosting has identified several possible definitions (see \cite[Chapter 10]{schapire2012boosting} as well as many more recent works \cite{mukherjee2013theory,brukhim2023multiclass,brukhim2021multiclass,brukhim2022characterization,brukhim2023improper}). Our approach proceeds by defining a \emph{partial binary} concept class given any multiclass classification problem, in \cref{def:mc-weak-learner} below. We will then feed our weak learner for partial binary classes (\WeakRealizable; \cref{alg:weak-oig}) into \Adaboost, and finally show how to translate good performance of the boosted learner back to good performance for the original multiclass problem.

Our approach is similar to the one taken in \cite[Chapter 11]{schapire2012boosting} (which originally appeared in \cite{schapire1998improved}), where multiclass classification is reduced to boosting with rankings and the \texttt{Adaboost.MR} algorithm is used. However, our approach is syntactically different since the weak learner for a menu class (in the sense of \cref{def:mc-weak-learner}) takes as input \emph{two labels} together with the covariate $x$, and must determine which of them is the correct label, in contrast to \cite{schapire1998improved} where the weak learner takes as input a single label together with $x$ and outputs a scalar indicating how likely the label is to be correct. \noah{check}
\begin{definition}[Menu class]
  \label{def:mc-weak-learner}
  Consider a hypothesis class $\MH \subset [K]^\MX$, and let $\Kpairs$ denote the set of all length-2 vectors consisting of distinct elements of $[K]$. We refer to the elements of $\Kpairs$ as \emph{menus}.\footnote{This terminology is inspired by \cite{brukhim2022characterization,brukhim2023improper}, which considered such menus, though used them together with techniques distinct from ours.} Given $\mu \in \Kpairs, k \in [K]$, we write $k \in \mu$ to mean that $k$ is one of the two elements of $\mu$.

Given $h \in \MH$ and $(x, (\ell, k)) \in \MX \times \Kpairs$, we define
\begin{align}
  \hmenu(x, (\ell, k)) := \begin{cases}
    0 &: h(x) = \ell \\
    1 &: h(x) = k \\
    * &: \mbox{otherwise}.
  \end{cases}\nonumber
\end{align}
we let $\Menucls(\MH) \subset \{0,1,*\}^{\MX \times \Kpairs}$ denote the (binary) partial hypothesis class defined by
\begin{align}
\Menucls(\MH) := \left\{ (x, \mu) \mapsto \hmenu(x, \mu) \ : \ h \in \MH \right\}\nonumber.
\end{align}
We often refer to $\Menucls(\MH)$ as the \emph{menu class} of $\MH$.
\end{definition}

The below definition gives a procedure to map hypotheses in $\{0,1,* \}^{\MX \times \Kpairs}$ to multiclass hypotheses in $[K]^\MX$.
\begin{definition}[Multiclass decoder]
  \label{def:mc-decoder}
We define a map $\Decmc : \{0,1,*\}^{\MX \times \Kpairs} \to [K]^\MX$, as follows. For $J : \MX \times \Kpairs \to \{0,1,* \}$, we define $\Decmc(J) := H$, where $H(x)$  is defined to be the unique value $k \in [K]$ for which $J(x, (k,\ell)) = 0$ and $J(x, (\ell, k)) = 1$ for all $\ell \in [K]\backslash \{ k \}$, if such $k$ exists. Otherwise $H(x)$ is defined to be 1.
\end{definition}
Note that $\Decmc$ depends on $K, \MX$; for simplicity, we omit this dependence in our notation. \cref{lem:menucls-vc} bounds the VC dimension of the menu class of $\MH$ in terms of the Natarajan dimension of $\MH$.

\begin{lemma}
  \label{lem:menucls-vc}
For any $\MH \subset [K]^\MX$, it holds that $\VCdim(\Menucls(\MH)) \leq O( \Natdim(\MH) \log K)$.
\end{lemma}
\begin{proof}
  Let us write $d := \VCdim(\Menucls(\MH))$, and let $(x_1, \mu_1), \ldots, (x_d, \mu_d)$ be shattered by $\Menucls(\MH)$. Note that the values $x_1, \ldots, x_d$ are all distinct, since no hypothesis $\hmenu$ can shatter the points $(x, \mu), (x, \mu')$ for distinct menus $\mu, \mu' \in \Kpairs$. Thus, the number of vectors $y \in [K]^n$ for which there exists $h \in \MH$ so that $y_i = h(x_i)$, $i \in [n]$, is at least $2^d$. But the number of such $y$ may also be upper bounded by $(K^2 ed/\Natdim(\MH))^{\Natdim(\MH)}$, by \cite{haussler1995generalization}. %
  It follows that $2^d \leq (K^2 ed/\Natdim(\MH))^{\Natdim(\MH)}$, i.e., $d \leq \Natdim(\MH) \cdot O(\log(Kd/\Natdim(\MH)))$, which yields $d \leq O(\Natdim(\MH) \cdot \log(K))$. 
\end{proof}

\begin{algorithm}
  \caption{Multiclass learner}
  \label{alg:mc-boosting}
  \begin{algorithmic}[1]\onehalfspacing
    \Require Concept class $\MH \subset [K]^\MX$, sample $\{ (x_i, y_i) \}_{i \in[n]} \subset (\MX \times [K])^n$, domain point $x \in \MX$, weak learner $\SA$ for $\Menucls(\MH)$ with margin $\eta$, failure probability $\delta$, weak ERM oracle $\Oerm$. 
    \Function{\MulticlassRealizable}{$S,x, \SA, \eta, \delta$}
    \State Define a dataset $\tilde S \subset (\MX \times \Kpairs \times \{0,1,* \})^{2n(K-1)}$ as follows:
    \begin{align}
\tilde S = \left\{ (x_i, (y_i, \ell), 0) \ : \ i \in [n], \ell \in [K] \backslash \{ y_i \} \right\} \cup \left\{ (x_i, (\ell, y_i), 1) \ : \ i \in [n], \ell \in [K] \backslash \{ y_i \} \right\}\label{eq:define-stil-mc}.
    \end{align}
    \State \label{line:mc-call-adaboost} Call \texttt{Adaboost} (\cref{alg:adaboost}) on the dataset $\tilde S$ using the weak learner $\SA$ with $T = \lceil 16 \log(4nK/\delta)/\eta^2 \rceil$, and denote the resulting hypothesis by $J : \MX \times \Kpairs \to \{0,1\}$.
    \State \label{line:mc-define-H} Define the hypothesis $H := \Decmc(J) \in [K]^\MX$ (see \cref{def:mc-decoder}). %
    \State \Return $H(x)$. 
    \EndFunction

    \Function{\MulticlassAgnostic}{$S, \SA, \eta, \delta$}
    \State Let $z \in \{0,1\}^n$ denote the output of $\SampleERMBinary(S, \lmc, \Oerm)$. \Comment{\emph{(\cref{alg:sample-erm-binary})}}
    \State Define $\tilde S := \{ (x_i, y_i) \ : \ i \in [n], z_i = 0 \}$. 
    \State Set $H := \MulticlassRealizable(\tilde S, \SA, \eta, \delta)$.
    \State \Return $H(x)$. 
    \EndFunction
  \end{algorithmic}

  \end{algorithm}

Note that, given a weak consistency oracle $\Ocon$ for the class $\MH \subset [K]^\MX$, we immediately obtain a weak consistency oracle $\Oconmenu$ for the class $\Menucls(\MH)$: for a $\Menucls(\MH)$-realizable dataset $S = \{ (x_i, \mu_i, y_i) \}_{i \in [n]}$, the oracle $\Oconmenu(S)$ defines $y_i' = (\mu_i)_{y_i+1}$, and returns $\Ocon(\{ (x_i, y_i') \}_{i \in [n]})$. 
Using this observation, we may obtain a weak learner for the class $\Menucls(\MH)$ as a corollary of \cref{thm:weak-oig}:
\begin{corollary}
  \label{cor:mc-weak-learner}
  There are constants $C_1, C_2$ so that the following holds. Consider a multiclass concept class $\MH \subset [K]^\MX$ of Natarajan dimension $\Natdim$, and suppose $m \geq C_1 \Natdim\log K\cdot  \log( \Natdim\log K)$. Let $S = \{ ((x_i, \mu_i), y_i) \} \in (\MX \times \Kpairs \times \{0,1\})^m$ be $\Menucls(\MH)$-realizable.
  and let $\SA(S_{-i},( x_i, \mu_i))$ be the output of $\WeakRealizable(S_{-i}, (x_i, \mu_i), 1-\frac{1}{C_1 m \log m}, 1, C_1 m^2 \log^3 m, \Oconmenu)$ (\cref{alg:weak-oig}) for each $i \in [m]$. Then
  \begin{align}
\frac 1m \sum_{i=1}^m \E_{\SA} \left[ \lbin(\SA(S_{-i}, (x_i, \mu_i)), y_i) \right] \leq \frac 12 - \frac{1}{C_2 m \log m}\nonumber,
  \end{align}
  where the expectation is taken over the randomness in the runs of $\WeakRealizable$ as well as the sampling of $\hat y$ from its output. Moreover, \WeakRealizable makes at most $\tilde O(m^3)$ calls to $\Oconmenu$, each with a dataset of size $m-1$. 
\end{corollary}
\begin{proof}
The corollary follows immediately from \cref{thm:weak-oig} applied to the class $\Menucls(\MH)$, which has VC dimension bounded by $O(\Natdim \log K)$ by \cref{lem:menucls-vc}, together with the fact observed above that a weak consistency oracle for $\Menucls(\MH)$ can be implemented using a weak consistency oracle for $\MH$. 
\end{proof}

\begin{lemma}[Training error of \Adaboost in multiclass setting]
  \label{lem:adaboost-mc-training}
  Let $m,n \in \BN$ and $\eta,\delta \in (0,1)$ be given, and suppose that algorithm $\SA$ is an $m$-sample weak learner with margin $\eta$ for the class $\Menucls(\MH)$. Let $x \in \MX$ and $\bar S \in (\MX \times [K])^n$ be an $\MH$-realizable sample. Then the hypothesis $H$ defined in \cref{line:mc-define-H} of $\MulticlassRealizable(\bar S, x, \SA, \eta, \delta)$ (\cref{alg:mc-boosting}) %
  satisfies $\eerr_{\bar S}(H) = 0$ with probability $1-\delta$. 
\end{lemma}
Note that the domain point $x$ plays no role in \cref{lem:adaboost-mc-training}.
\begin{proof}[Proof of \cref{lem:adaboost-mc-training}]
  By \cref{lem:adaboost-sample} applied to the dataset $\tilde S$ (defined in \cref{eq:define-stil-mc}) together with the guarantee on the weak learner $\SA$, the output hypothesis $J : \MX \times \Kpairs \to \{0,1\}$ of \Adaboost satisfies the following with probability $1-\delta$:
  \begin{align}
\forall i \in [n], \ \ell \in [K] \backslash \{ y_i \}, \quad J(x, (y_i, \ell)) = 0, \ J(x,(\ell, y_i)) = 1\label{eq:j-boosted-guarantee}.
  \end{align}
  For each $(x,y) \in \MX \times [K]$, we have
  \begin{align}
\lmc(H(x), y) = \One{H(x) \neq y} \leq \sum_{\ell \in [K] \backslash \{ y\}} \One{J(x, (y, \ell)) = 1} + \One{J(x, (\ell, y)) = 0}\nonumber,
  \end{align}
  since if $H(x) \neq y$, then there must be some $\ell \neq y$ so that $J(x, (y, \ell)) = 1$ or $J(x, (\ell, y)) = 0$. The guarantee \cref{eq:j-boosted-guarantee} on $J$ yields that we must have $H(x_i) = y_i$ for all $i \in [n]$, with probability $1-\delta$. 
  
\end{proof}

\begin{proof}[Proof of \cref{thm:multiclass-main}]
Let $\Natdim \in \BN$ be given and consider a concept class $\MH \subset [K]^\MX$ with $\Natdim(\MH) \leq \Natdim$, together with a weak consistency oracle $\Ocon$ and a weak ERM oracle $\Oerm$ for $\MH$.  Let $C_1, C_2$ denote the constants in the statement of \cref{cor:mc-weak-learner}, and define $m := C_1 \Natdim \log K \cdot \log(\Natdim \log K)$ and $\eta : = \frac{1}{C_2 m \log m}$. Write $T = \lceil 16 \log(4nK/\delta)/\eta^2 \rceil$ as in \cref{alg:mc-boosting}.  Let $\SA : (\MX \times \Kpairs \times \{0,1\})^m \times (\MX \times \Kpairs) \to \{0,1\}$ denote the learner of \cref{cor:mc-weak-learner} (in particular, $\SA(S, (x, \mu))$ calls $\WeakRealizable$ with inputs $S, (x, \mu), \Ocon$, and an appropriate choice of the parameters). Then by \cref{cor:mc-weak-learner} and \cref{lem:loo}, $\SA$ is an $m$-sample weak learner for the partial concept class $\Menucls(\MH)$ with margin $\eta$. 

  \paragraph{Realizable case.} We begin by proving the first statement of the theorem. Given $\ep, \delta \in (0,1)$ and an $\MH$-realizable distribution $P \in \Delta(\MX \times [K])$, consider a sample $S \sim P^n$, where $n$ will be chosen below. We will show that the algorithm which, given $S$ and some $x \in \MX$ as input, returns the output $H(x)$ of  $\MulticlassRealizable(S, x, \SA, \eta, \delta)$ satisfies the requirements of the theorem. Let $H : \MX \to [K]$ denote the hypothesis constructed in \cref{line:mc-define-H} of \MulticlassRealizable; we wish to show that with probability $1-\delta$ over $S$ and the random bits of $\MulticlassRealizable$, $\err_P(H) \leq \ep$. To do so, we will define a distribution $Q$ over sample compression schemes $(\rho, \kappa)$ on $n$-sample datasets, so that the distribution of $\rho(\kappa(S))$ %
  is the same as the distribution of $H$. The call to $\Adaboost$ (\cref{alg:adaboost}) in \cref{line:mc-call-adaboost} of \cref{alg:mc-boosting} generates a sequence of i.i.d.~datasets $\tilde S_1, \ldots, \tilde S_T \in (\MX \times \Kpairs \times \{0,1\})^m$, each consisting of examples in $\tilde S$ (defined in \cref{eq:define-stil-mc}),  together with a sequence of parameters $\alpha_1, \ldots, \alpha_T \in \BR$ and a sequence of random bitstrings $R_1, \ldots, R_T$ for use in the weak learner $\SA$ and in the sampling step on \cref{line:sample-st} of \cref{alg:adaboost}. The values of $\tilde S_t, \alpha_t, R_t$ satisfy the following: the output hypothesis $J : \MX \times \Kpairs \to \{0,1\}$ of \Adaboost is given by
  \begin{align}
    \label{eq:j-compression}
    J(x, \mu) := \frac 12 + \frac 12 \sign \left( \sum_{t=1}^T \alpha_t \cdot (2\SA_{R_t}(\tilde S_t, x, \mu) - 1) \right).
  \end{align}
  For $j \in [m]$, the $j$th example in $\tilde S_t$ may be written as $(x_{i_{t,j}}, (\ell_{t,j}, k_{t,j}), b_{t,j})$, for some $i_{t,j} \in [n], \ell_{t,j} \in [K], k_{t,j} \in [K], b_{t,j} \in \{0,1\}$. We then define the (random) dataset $S_t := \{ (x_{i_{t,j}}, y_{i_{t,j}}) \}_{j \in [m]}$.

Note that, for fixed $S \in (\MX \times [K])^n$, the resulting random variables $(\tilde S_t, S_t, \alpha_t)_{t \in [T]}$ are a deterministic function of $S$ and $(R_t)_{t \in [T]}$.  We now define $(\rho, \kappa) \sim Q$ to be distributed as follows: %
  \begin{itemize}
  \item $\kappa$ maps the dataset $S \in (\MX \times [K])^n$ to $\kappa(S) = (S_1, \ldots, S_T), \left( \alpha_t, (\ell_{t,j}, k_{t,j}, b_{t,j})_{j \in [m]} \right)_{t \in [T]}$, where $S_t, \alpha_t, \ell_{t,j}, k_{t,j}, b_{t,j}$ are defined as a function of $(R_t)_{t \in [T]}$ and $S$ as described above. 
  \item $\rho$ proceeds as follows, given an input of the form $(S_1', \ldots, S_T'), \left( \alpha_t', (\ell_{t,j}', k_{t,j}', b_{t,j}')_{j \in [m]} \right)_{t \in [T]}$ (where $S_1', \ldots, S_T'$ is a sequence of examples in $\MX \times [K]$ of length $Tm$ and the supplemental information $\alpha_t' \in \BR$, $\ell_{t,j}', k_{t,j}' \in [K], b_{t,j}' \in \{0,1\}$ are encoded in binary). Denoting $S_t' = \{ (x_{t,j}', y_{t,j}') \}_{j \in [m]}$,  let us define $\tilde S_t' := \{(x_{t_j}', (\ell_{t,j}', k_{t,j}'), b_{t,j}')\}_{j \in [m]}$; then $\rho$ outputs the hypothesis $x \mapsto \Decmc(J')$, where $J'(x, \mu) := \left(\frac 12 + \frac 12 \sign \left( \sum_{t=1}^T \alpha_t' \cdot (2\SA_{R_t}(\tilde S_t', x, \mu) - 1) \right) \right)$. 
  \end{itemize}

  Since the values $\alpha_t$ defined in $\Adaboost$ can be encoded with $O(\log n)$ bits (as $\ep_t \in \{0, 1/n, \ldots, 1\}$) and the list $((\ell_{t,j}, k_{t,j}, b_{t,j}))_{j \in [m]}$ can be encoded with $O(m \log K)$ bits, with probability 1 over $(\kappa, \rho) \sim Q$, the size of $\kappa$ for inputs samples $S$ of size $n$ is bounded by $|\kappa| \leq O(T \cdot (m \log K + \log n))$. Next, \cref{lem:adaboost-mc-training} together with the definition of $Q$ above (and in particular the fact that $\rho$ uses the same bits $R_t$ as in the definition of $\kappa$),  gives that for any $\MH$-realizable dataset $S$, there is a subset $\ME$ of compression schemes for which $Q(\ME) \geq 1-\delta$  so that for all $(\kappa, \rho) \in \ME$, $\eerr_S(\rho(\kappa(S))) = 0$. By \cref{lem:gen-compression} applied to the multiclass loss function, there is a constant $C > 0$ so that for each $(\kappa, \rho) \in \supp(Q)$, with probability $1-\delta$ over the draw of $S \sim P^n$,
  \begin{align}
\err_P(\rho(\kappa(S))) \leq \One{\eerr_S(\rho(\kappa(S))) = 0} + C \cdot \frac{T \cdot (m \log K + \log n) \cdot \log n + \log(1/\delta)}{n}\nonumber,
  \end{align}
  for some sufficiently large constant $C$. 
  By our choice of $T$ together with the fact that $Q(\ME) \geq 1-\delta$, it follows that with probability $1-2\delta$ over the draw of $S \sim P^n$ and the draw of $(\rho, \kappa) \sim Q$, we have
  \begin{align}
\err_P(\rho(\kappa(S))) \leq C \cdot \frac{ \log(n) \log(nK/\delta) \cdot \frac{1}{\eta^2} \cdot (m \log K + \log n)}{n}\nonumber.
  \end{align}
  For fixed $S$, the distribution of $\rho(\kappa(S))$ is the same as the distribution of the output hypothesis $H$ of $\MulticlassRealizable$. 
  Thus, by our choices of $m, \eta$, after rescaling $\delta$, we can ensure that $\err_P(H) \leq \ep$ with probability $1-\delta$ as long as we take $n = \frac{\Natdim^3 \log^4(K/\delta)}{\ep} \cdot \poly(\log 1/\ep, \log \Natdim, \log \log K)$.

  Note that the dataset $\tilde S$ passed to \Adaboost in \MulticlassRealizable is of size $|\tilde S| = O(nK)$. Since the number of rounds of \Adaboost is $T \leq \poly(n)$ and the weak learner $\SA$ makes $\poly(m)$ calls to $\Ocon$, which can simulate $\Oconmenu$ (\cref{cor:mc-weak-learner}), the same argument as in the proof of \cref{thm:partial-main} shows that the cumulative query cost of \MulticlassRealizable for the oracle $\Ocon$ is $\tilde O(Kn m^3 \cdot T) \leq K \cdot \poly(n)$. 

  \paragraph{Agnostic case.} Let $P \in \Delta(\MX \times [K])$ be a distribution. For $n$ sufficiently large (as specified below), we will show that the algorithm which takes as input a sample $S \sim P^n$ and $x \in \MX$, and which returns the output $H(x)$ of $\MulticlassAgnostic(S, x,\SA, \eta, \delta)$ (\cref{alg:mc-boosting}) satisfies the requirements of the theorem. Let $H : \MX \to [K]$ denote the hypothesis returned by \MulticlassRealizable in \MulticlassAgnostic; we want to show that with probability $1-\delta$ over $S$ and the random bits of \MulticlassAgnostic, $\err_P(H) \leq \min_{h \in \MH} \err_P(h) + \ep$. To do so, we consider the same distribution $Q$ over sample compression schemes on datasets with (at most) $n$ samples. Let $\Sigma$ denote the mapping which takes as input $S \in (\MX \times [K])^n$ and outputs the dataset $\tilde S \in (\MX \times [K])^{\leq n}$ as defined in \MulticlassAgnostic. For any compression scheme $(\kappa, \rho) \in \supp(Q)$, $(\kappa \circ \Sigma, \rho)$ is a compression scheme of size $|\kappa \circ \Sigma| \leq |\kappa| \leq O(T \cdot (m \log K + \log n))$.  Thus, by \cref{lem:gen-compression}, there is a constant $C > 0$ so that, for any fixed $(\kappa, \rho) \in \supp(Q)$, with probability $1-\delta$ over the draw of $S \sim P^n$,
  \begin{align}
\err_P(\rho(\kappa(\Sigma( S)))) \leq \eerr_{ S}(\rho(\kappa(\Sigma( S)))) + C \sqrt{ \frac 1n \left( T \log(n) \cdot (m \log K + \log n) + \log(1/\delta) \right)}\label{eq:use-compression-agnostic-mc}.
  \end{align}

  By \cref{lem:weak-to-sample-binary} with loss function $\lmc$, for any $S \in (\MX \times [K])^n$, the dataset $\tilde S = \Sigma(S)$ satisfies $\inf_{h \in \MH} \eerr_S(h) = \frac{n - |\tilde S|}{n}$. The lemma also guarantees that $\tilde S$ is $\MH$-realizable, and thus there is a set $\ME$ of compression schemes satisfying $Q(\ME) \geq 1-\delta$ so that for all $(\kappa, \rho) \in \ME$, $\eerr_{\tilde S}(\rho(\kappa(\tilde S)))  = 0$.  Using this fact and \cref{eq:use-compression-agnostic-mc}, it follows that with probability $1-2\delta$ over the draw of $S \sim P^n$,and $(\rho, \kappa) \sim Q$,
  \begin{align}
    \err_P(\rho(\kappa(\Sigma(S)))) \leq & \frac{n - |\tilde S|}{n} + C \sqrt{ \frac 1n \left( T \log(n) \cdot (m \log K + \log n) + \log(1/\delta) \right)}\nonumber\\
    =& \inf_{h \in \MH} \eerr_{\bar S}(h) + C \sqrt{ \frac{\log(n) \log(nK/\delta) \cdot \frac{1}{\eta^2} \cdot (m \log K + \log n)}{n}}\label{eq:agnostic-mc-almost},
  \end{align}
  where the equality uses our choice of $T$. Moreover McDiarmid's inequality yields that with probability $1-\delta$ over the choice of $S \sim P^n$,
  \begin{align}
\inf_{h \in \MH} \eerr_S(h) \leq \E_{S' \sim P^n} \left[ \inf_{h' \in \MH} \eerr_{S'}(h') \right] + C \sqrt{\frac{\log 1/\delta}{n}} \leq \inf_{h \in \MH} \err_P(h) + C \sqrt{\frac{\log 1/\delta}{n}}\label{eq:mcdiarmid-agnostic-mc},
  \end{align}
  for a sufficiently large constant $C$. Combining \cref{eq:agnostic-mc-almost,eq:mcdiarmid-agnostic-mc}, using our choice of $m,\eta$, and choosing $n = \frac{\Natdim^3 \log^4(K/\delta)}{\ep^2} \cdot \poly(\log 1/\ep, \log \Natdim, \log \log K)$ yields the claimed statement of the theorem.

  To analyze the oracle complexity of \MulticlassAgnostic, we first note that the call to \SampleERMBinary makes $O(n^2)$ calls to $\Oerm$ (\cref{lem:weak-to-sample-binary}). The remainder of the analysis of oracle complexity follows the realizable case exactly. 
\end{proof}

\section{Regression}
\label{sec:regression}
In this section, we consider another generalization of our results on binary classification, namely to the setting of regression, in which hypotheses and labels take values in $[0,1]$. Similar to \cref{sec:multiclass}, our approach proceeds via reducing to oracle-efficient learning of partial function classes. A similar approach was used in \cite{bartlett1998prediction,adenali2023optimal}. %
We thus define a partial concept class associated to any real-valued concept class in \cref{def:real-weak-learner} below:
\begin{definition}[Threshold class]
  \label{def:real-weak-learner}
  Consider a hypothesis class $\MH \subset [0,1]^\MX$ and $\discmg \in (0,1)$. We let $\disc[\discmg] := \{ 0, \discmg, 2\discmg, \ldots, \lfloor 1/\discmg \rfloor \discmg \}$, Given $h \in \MH$ and $\discpt \in \disc$, we define
  \begin{align}
    \hthr_\discmg(x, \discpt) := \begin{cases}
      1 &: h(x) \geq \discpt + \discmg \\
      0 &: h(x) \leq \discpt - \discmg \\
      * & : |h(x) - \discpt| < \discmg.
    \end{cases}\nonumber
  \end{align}
  We then let $\Thrcls(\MH) \subset \{0,1,*\}^{\MX \times \disc}$ denote the (binary) partial hypothesis class defined by
  \begin{align}
\Thrcls(\MH) := \{ (x, \discpt) \mapsto \hthr_\discmg(x, \discpt) \ : \ h \in \MH \}\nonumber.
  \end{align}
  We will refer to $\Thrcls(\MH)$ as the \emph{threshold class} of $\MH$. 
\end{definition}

\begin{lemma}
  \label{lem:vc-thr-fat}
  For any $\MH \subset [0,1]^\MX$ and $\discmg \in (0,1)$, it holds that $\VCdim(\Thrcls(\MH)) \leq \fatdim(\MH)$.
\end{lemma}
\begin{proof}
  Let us write $d := \fatdim(\MH)$, and let $(x_1, \discpt_1), \ldots, (x_d, \discpt_d)$ be shattered by $\Thrcls(\MH)$, so that $x_i \in \MX$ and $\discpt_i \in \disc$ for each $i \in [d]$. Then by the definition of $\Thrcls(\MH)$ (\cref{def:real-weak-learner}), for each sequence $b \in \{0,1\}^d$, there is some $h \in \MH$ so that, for each $i \in [d]$, $h(x) \geq \discpt_i + \discmg$ if $b_i = 1$ and $h(x) \leq \discpt_i - \discmg$ if $b_i = 0$. Thus $\MH$ shatters the points $(x_1, \ldots, x_d)$ as witnessed by $(\discpt_1, \ldots, \discpt_d)$, i.e., $\fatdim(\MH) \geq d$. 
\end{proof}

An immediate consequence of \cref{lem:sample-con-real} is that given a weak ERM oracle $\Oerm$ for the class $\MH$, we immediately obtain a weak consistency oracle for the class $\Thrcls(\MH)$. Using this observation, we obtain the following corollary of \cref{thm:weak-oig}:
\begin{corollary}
  \label{cor:thrcls-weak}
  There are constants $C_1, C_2$ so that the following holds. Consider a real-valued concept class $\MH \subset [0,1]^\MX$ and $\discmg \in (0,1)$, write $d := \fatdim[\discmg](\MH)$, and suppose $m \geq C_1 d \log d$. Let $S = \{ ((x_i, \discpt_i), y_i) \} \in (\MX \times \disc \times \{0,1\})^m$ be $\Thrcls(\MH)$-realizable. Then there is an algorithm $\SA : (\MX \times \disc \times \{0,1\})^{m-1} \times (\MX \times \disc) \to \Delta(\{0,1\})$ which makes $\tilde O(m^3)$ calls to a range consistency oracle $\Orange$ for $\MH$ and which satisfies
  \begin{align}
\frac 1m \sum_{i=1}^m \E_{\hat y \sim \SA(S_{-i}, (x_i, \discpt_i))} \left[ \lbin(\hat y, y_i) \right] \leq \frac 12 - \frac{1}{C_2 m \log m}\nonumber.
  \end{align}
\end{corollary}
\begin{proof}
  The corollary follows immediately from \cref{thm:weak-oig} applied to the class $\Thrcls(\MH)$, which has VC dimension bounded by $\fatdim(\MH)$ by \cref{lem:vc-thr-fat}, together with the fact that  a single weak ERM oracle call for the class $\Thrcls(\MH)$ can be implemented by a single range consistency oracle call for the class $\MH$ of the same length, using \SampleConReal (\cref{lem:sample-con-real}). 
\end{proof}

\begin{algorithm}
  \caption{Real-valued learner}
  \label{alg:reg-boosting}
  \begin{algorithmic}[1]\onehalfspacing
    \Require Concept class $\MH \subset [0,1]^\MX$, sample $S = \{ (x_i, y_i) \}_{i \in [n]} \subset (\MX \times [0,1])^n$, point $x \in \MX$, failure probability $\delta$, discretization parameters $\discmg, \beta \in (0,1)$, weak learner $\SA$ for $\Thrcls(\MH)$ with margin $\eta$, weak ERM oracle $\Oerm$. 
    \Function{\RegRealizable}{$S,x, \SA, \eta, \delta, \discmg, \beta$}
    \State For $i \in [n]$ and $\discpt \in \disc$, define $y_{i,\discpt}' := \begin{cases} 1 &: y_i \geq \discpt + \beta \\ 0 &: y_i \leq \discpt - \beta \\ * &: |y_i - \discpt| < \beta.\end{cases}$ 
    \State Define a dataset $\tilde S \subset (\MX \times \disc \times \{0,1,*\})^{n'}$ for some $n' \leq 2n \cdot(\lfloor 1/\discmg \rfloor + 1)$, as follows:
    \begin{align}
\tilde S = \left\{ ((x_i, \discpt), y_{i,\discpt}') \ : \ i \in [n], \discpt \in \disc, y_{i,\discpt}' \neq * \right\}\label{eq:tildes-reg}.
    \end{align}
    \State\label{line:reg-call-adaboost} Call \Adaboost (\cref{alg:adaboost}) on the dataset $\tilde S$ using the weak learner $\SA$ with $T = \lceil 16 \log(4n/\delta)/\eta^2 \rceil$, and denote the resulting hypothesis by $J : \MX \times  \disc \to \{0,1\}$.
    \State \Return $H(x) := \discmg \cdot \sum_{\discpt \in \disc} J(x, \discpt)$. 
    \EndFunction

    \Function{\RegAgnostic}{$S,x, \SA, \eta, \delta, \discmg$}
    \State Let $\hat y \in [0,1]^n$ denote the output of $\SampleERMReal(S, \discmg/2, \Oerm)$. \Comment{\emph{\cref{alg:sample-erm-real}}}
    \State Define $\tilde S = \{ (x_i, \hat y_i) \ : \ i \in [n]\}$.\label{line:call-ermreal}
    \State \Return $H(x) := \RegRealizable(\tilde S,x, \SA, \eta, \delta, \discmg, 2\discmg)$.
    \EndFunction
  \end{algorithmic}
\end{algorithm}

\begin{lemma}[Training error for regression]
  \label{lem:adaboost-train-reg}
  Let $m,n \in \BN$, $x \in \MX$, and $\eta ,\delta, \discmg, \beta \in (0,1)$ be given, and suppose that algorithm $\SA$ is an $m$-sample weak learner with margin $\eta$ for the class $\Thrcls(\MH)$. Let $\bar S \in (\MX \times [0,1])^n$ be a sample so that for some $h^\st \in \MH$, each $(x_i,y_i) \in \bar S$ satisfies $|y_i - h^\st(x_i)| \leq \beta - \discmg$. Then the hypothesis $H(\cdot) := \discmg \cdot \sum_{\discpt \in \disc} J(\cdot,\discpt)$, where $J$ is defined on \cref{line:reg-call-adaboost} of $\RegRealizable(\bar S,x, \SA, \eta, \delta, \discmg, \beta)$, satisfies $\eerr_{\bar S}(H) \leq 3\beta$ with probability $1-\delta$. 
\end{lemma}
\begin{proof}
Since there is some $h^\st$ so that each $(x,y) \in \bar S$ satisfies $|y - h^\st(x)| \leq \beta-\discmg$,   the dataset $\tilde S$ defined in \cref{eq:tildes-reg} is $\Thrcls(\MH)$-realizable. Then by \cref{lem:adaboost-sample} together with the assumed guarantee on $\SA$, the output hypothesis $J : \MX \times \disc \to\{0,1\}$ of \Adaboost satisfies the following with probability $1-\delta$:
  \begin{align}
\forall i \in [n], \ \discpt \in \{ \discpt' \in \disc \ : \ |y_i - \discpt'| \geq \beta \}, \quad J(x_i, \discpt) = \One{\discpt > y_i}\label{eq:threshold-class-correct}.
  \end{align}
  Thus, for each $i \in [n]$, letting $\hat y_i := \discpt \cdot \lfloor y_i / \discpt \rfloor$, we have
  \begin{align}
    \labs(H(x_i), y_i) = |H(x_i) - y_i| \leq & |H(x_i) - \hat y_i| + |\hat y_i - y_i|\nonumber\\
    \leq & \discmg + \discmg \cdot \sum_{\discpt \in \disc} \One{J(x_i, \discpt) \neq \One{\hat y_i > \discpt}}\nonumber\\
    \leq & \discmg + \discmg \lceil \beta / \discmg \rceil + \discmg \cdot \sum_{\discpt \in \disc,\ |y_i - \discpt| \geq \beta} \One{J(x_i, \discpt) \neq \One{\hat y_i > \discpt}}\leq 3\beta\nonumber,
  \end{align}
  where the final equalitty uses \cref{eq:threshold-class-correct}. It follows that $\eerr_{\bar S}(H) = \frac 1n \sum_{i=1}^n \labs(H(x_i), y_i) \leq 3\beta$, as desired.
\end{proof}

\begin{proof}[Proof of \cref{thm:reg-main}]
  The proof closely follows that of \cref{thm:multiclass-main}. 
Let a mapping $\discmg \mapsto \fatdim$ be given and consider a concept class $\MH \subset [0,1]^\MX$ with $\fatdim(\MH) \leq \fatdim$ for all $\discmg \in [0,1]$, together a weak range oracle $\Orange$ and a weak ERM oracle $\Oerm$ for $\MH$. Let $C_1, C_2$ be the constants in the statement of \cref{cor:thrcls-weak}, and fix $\discmg \in (0,1)$. Define $m := C_1 \fatdim \cdot \log(\fatdim)$ and $\eta : = \frac{1}{C_2 m \log m}$. Write $T = \lceil 16 \log(4n/\delta)/\eta^2 \rceil$ as in \cref{alg:reg-boosting}.  Let $\SA : (\MX \times \disc \times \{0,1\})^m \times (\MX \times \disc) \to \{0,1\}$ denote the learner of \cref{cor:thrcls-weak} (in particular, $\SA(S, (x, \discpt))$ calls $\WeakRealizable$ with inputs $S, (x, \discpt)$, and an appropriate choice of the parameters). Then by \cref{cor:thrcls-weak} and \cref{lem:loo}, $\SA$ is an $m$-sample weak learner for the partial concept class $\Thrcls(\MH)$ with margin $\eta$. Moreover, as guaranteed by \cref{cor:thrcls-weak}, a single call of $\SA$ can be completed using $\tilde O(m^3)$ calls to either $\Orange$ or $\Oerm$. 

  \paragraph{Realizable case.} Given $n, \delta \in (0,1)$ and an $\MH$-realizable distribution $P \in \Delta(\MX \times [K])$, consider a sample $S \sim P^n$. We will show that the algorithm which, given $S$ and some $x \in \MX$ as input, returns the output $H(x)$ of $\RegRealizable(S, x,\SA, \eta, \delta, \discmg, \discmg)$ satisfies the requirements of the theorem. Let $H : \MX \to [0,1]$ be defined by $H(\cdot) = \discmg \cdot \sum_{\discpt \in \disc} J(\cdot, \discpt)$, where $J$ is the hypothesis defined in \cref{line:reg-call-adaboost} of \RegRealizable. We want to show that with probability $1-\delta$ over $S$ and the random bits of \RegRealizable, $\err_P(H) \leq \ep$ for an appropriate choice of $\ep$. To do so, we define a distribution $Q$ over sample compression schemes $(\rho, \kappa)$ on datasets of size at most $2n \cdot (\lfloor 1/\discmg \rfloor + 1)$, so that the distribution of $\rho(\kappa(S))$ is the same as the distribution of $H$. The call to \Adaboost (\cref{alg:adaboost}) in \cref{line:reg-call-adaboost} of \cref{alg:reg-boosting} generates a sequence of i.i.d.~datasets $\tilde S_1, \ldots, \tilde S_T \in (\MX \times \disc \times \{0,1\})^m$, each consisting only of examples in $\tilde S$ (defined in \cref{eq:tildes-reg}), together with a sequence of parameters $\alpha_1, \ldots, \alpha_T \in \BR$ and a sequence of random bitstrings $R_1, \ldots, R_T$. The values of $\tilde S_t, \alpha_t, R_t$ satisfy the following: the output hypothesis $J : \MX \times \disc \to \{0,1\}$ of $\Adaboost$ is given by $J(x, \discpt) := \frac 12 + \frac 12 \sign\left( \sum_{t=1}^T \alpha_t \cdot (2\SA_{R_t}(\tilde S_t, x, \discpt) - 1) \right)$. For $j \in [m]$, the $j$th example in $\tilde S_t$ may be written as $(x_{i_{t,j}}, \discpt_{t,j}, b_{t,j})$, for some random variables $i_{t,j} \in [n], \discpt_{t,j} \in \disc, b_{t,j} \in \{0,1\}$. We then define $S_t := \{ (x_{i_{t,j}}, y_{i_{t,j}})_{j \in [m]}$.

  We define $(\rho, \kappa) \sim Q$ to be distributed as follows: %
  \begin{itemize}
  \item $\kappa$ maps the dataset $S \in (\MX \times [0,1])^n$ to $\kappa(S) = (S_1, \ldots, S_T), (\alpha_t, (\discpt_{t,j}, b_{t,j})_{j \in [m]})_{t \in [T]}$, where $S_t, \alpha_t, \discpt_{t,j}, b_{t,j}$ are defined as a (deterministic) function of $(R_t)_{t \in [T]}$ and $S$ as described above.
  \item $\rho$ proceeds as follows, given an input of the form $(S_1', \ldots, S_T'), \left( \alpha_t', (\discpt_{t,j}', b_{t,j}')_{j \in [m]}\right)_{t \in [T]}$. Denoting $S_t' = \{ (x_{t,j}', y_{t,j}') \}_{j \in [m]}$, let us define $\tilde S_t' := \{ (x_{t_j}', \discpt_{t,j}', b_{t,j}') \}_{j \in [m]}$; then $\rho$ outputs the hypothesis $x \mapsto \discmg \cdot \sum_{\discpt \in \disc}\left( \frac 12 + \frac 12 \sign \left( \sum_{t=1}^T \alpha_t' \cdot (2\SA_{R_t}(\tilde S_t', x, \discpt) - 1 ) \right)\right)$. 
  \end{itemize}

  Since the values $\alpha_t$ defined in \Adaboost can be encoded with $O(\log n)$ bits and the list $(\discpt_{t,j}, b_{t,j})_{j \in [m]}$ can be encoded with $O(m \log 1/\discmg)$ bits, with probability 1 over $(\kappa, \rho) \sim Q$, the size of $\kappa$ for input samples $S$ of size $n$ is bounded by $|\kappa| \leq O(T \cdot (m \log 1/\discmg + \log n))$. \cref{lem:adaboost-train-reg} gives that there is a subset $\ME$ of compression schemes for which $Q(\ME) \geq 1-\delta$ so that for all $(\kappa, \rho) \in \ME$, $\eerr_S(\rho(\kappa(S))) \leq 3\beta$. By \cref{lem:gen-compression} applied to the absolute loss function, it follows that with probability $1-2\delta$ over the draw of $S \sim P^n$ and the draw of $(\rho, \kappa) \sim Q$, we have that, for some constant $C$, 
  \begin{align}
    \err_P(\rho(\kappa(S))) \leq &  C \cdot \sqrt{\discmg \Delta} + \Delta \leq \discmg + 2C^2 \Delta \label{eq:err-am-gm}\\
    \mbox{ for } & \Delta :=\frac{\log(n) \log(n/\delta) \cdot \frac{1}{\eta^2} \cdot (m \log(1/\discmg) + \log n)}{n}\nonumber,
  \end{align}
  where the second inequality in \cref{eq:err-am-gm} uses the AM-GM inequality. Given $S$, the distribution of $\rho(\kappa(S))$ is the same as the distribution of the output hypothesis $H$ of \RegRealizable. Thus, by making an optimal choice of the discretization parameter $\discmg$, our choices of $m, \eta$ and by rescaling $\delta$, we see that $\err_P(H) \leq \ep$ with probability $1-\delta$ for
  \begin{align}
\ep := \inf_{\discmg \in [0,1]} \left\{ O(\discmg) + \frac{\fatdim^3 \cdot \log(1/\delta)}{n} \cdot \poly\log(\fatdim, n, \log\log 1/\delta)  \right\} \label{eq:ep-alpha-optimize}.
  \end{align}

  Note that the dataset $\tilde S$ passed to \Adaboost in \RegRealizable is of size $|\tilde S| = O(n/\discmg) \leq O(n^2)$, since the optimal choice of $\discmg$ in \cref{eq:ep-alpha-optimize} is always bounded below by $1/n$. Since the number of rounds of \Adaboost is $T \leq \poly(n)$ and the weak learner $\SA$ makes $\poly(m)$ calls to $\Orange$ (\cref{cor:thrcls-weak}), the same argument as in the proof of \cref{thm:partial-main} establishes that the cumulative query cost of \RegRealizable for the oracle $\Orange$ is $\poly(n)$.

  \paragraph{Agnostic case.} Fix a distribution $P \in \Delta(\MX \times [0,1])$. Given $n \in \BN$, we will show that the algorithm which takes as input an i.i.d.~sample $S \sim P^n$ and a point $x \in \MX$, and returns the output  $H(x)$ of $\RegAgnostic(S,x, \SA, \eta, \delta, \discmg, 2\discmg)$ (\cref{alg:reg-boosting}) satisfies the requirements of the theorem. Let $H : \MX \to [K]$ denote the (random) hypothesis defined by $x \mapsto \RegAgnostic(S, x, \SA, \eta, \delta, \discmg)$ (in particular, this hypothesis is exactly the one given by $x\mapsto \discmg \sum_{\discpt \in \disc} J(x, \discpt)$, where $J$ is the hypothesis defined on \cref{line:reg-call-adaboost} in the call to \RegRealizable from \RegAgnostic). We want to show that, with probability $1-\delta$ over $S$ and the random bits of \RegAgnostic, $\err_P(H) \leq \inf_{h \in \MH} \err_P(h) + \ep$, for an appropriate value of $\ep$.  
  To do so, we consider the  distribution $Q'$ on compression schemes which is defined similarly to $Q$ as above, with the exception that $(\rho, \kappa) \sim Q'$ are defined so as to simulate the execution of $\RegRealizable(S, \SA, \eta, \delta, \discmg, 2\discmg)$, as opposed to $\RegRealizable(S, \SA, \eta, \delta, \discmg, \discmg)$. (In particular, this changes the definition of the i.i.d.~datasets $\tilde S_t$ as defined in \cref{eq:tildes-reg}.) 

  Let $\Sigma$ denote the mapping which, takes as input $S \in (\MX \times [0,1])^n$ and outputs the dataset $\tilde S \in (\MX \times [0,1])^n$ as defined in \cref{line:call-ermreal} of $\RegAgnostic$. For any compression scheme $(\kappa, \rho) \in \supp(Q')$, $(\kappa \circ \Sigma, \rho)$ is a compression scheme of size $|\kappa \circ \Sigma| \leq |\kappa| \leq O(T \cdot (m \log(1/\discmg) + \log n))$. Thus, by \cref{lem:gen-compression}, there is a constant $C > 0$ so that, for any fixed $(\kappa, \rho) \in \supp(Q')$, with probability $1-\delta$ over the draw of $S \sim P^n$,
  \begin{align}
\err_P(\rho(\kappa(\Sigma(S)))) \leq \eerr_S(\rho(\kappa(\Sigma(S)))) + C \sqrt{\frac{1}{n} (T \log(n) \cdot (m \log (1/\discmg) + \log n) + \log(1/\delta))}\label{eq:use-compression-agnostic-reg}.
  \end{align}
  By \cref{lem:weak-to-sample-real} and the definition of $\Sigma$, for any $S = \{ (x_i, y_i) \}_{i\in [n]}$, the dataset $\tilde S = \Sigma(S)$, which can be written as $\tilde S = \{ (x_i, \hat y_i) \}_{i \in [n]}$, satisfies the following: there is some $h^\st \in \MH$ so that $\eerr_S(h^\st) = \inf_{h \in \MH} \eerr_S(h)$ and $\labs(h^\st(x_i), \hat y_i) \leq \discmg$ for all $i \in [n]$. Thus, \cref{lem:adaboost-train-reg} with $\beta = 2\discmg$ yields that, for any $S$, the output hypothesis $H$ of $\RegRealizable(\tilde S, \SA, \eta, \delta, \discmg, \beta)$ satisfies $\eerr_{\tilde S}(H) \leq 6\discmg$ with probability $1-\delta$. Since the hypothesis $H$ has the same distribution as $\rho(\kappa(\Sigma(S)))$ for $(\rho, \kappa) \sim Q'$, we see that for any $S$, there is a set $\ME'$ of compression schemes satisfying $Q'(\ME') \geq 1-\delta$ so that, for all $(\kappa, \rho) \in \ME'$, $\eerr_{\tilde S}(\rho(\kappa(\Sigma(S)))) \leq 6\discmg$. 
  Moreover, note that, for any $H : \MX \to [0,1]$,
  \begin{align}
    \eerr_S(H) = \frac 1n \sum_{i=1}^n \labs(y_i, H(x_i)) \leq & \frac 1n \sum_{i=1}^N \labs(H(x_i), \hat y_i) + \frac 1n \sum_{i=1}^n \labs(\hat y_i, h^\st(x_i)) + \frac 1n \sum_{i=1}^n \labs(h^\st(x_i), y_i) \nonumber\\
    \leq & \eerr_{\tilde S}(H) + \discmg + \inf_{h \in \MH} \eerr_S(h)\label{eq:S-H-decompose}.
  \end{align}
  Combining \cref{eq:S-H-decompose} with $H = \rho(\kappa(\Sigma(S)))$ and \cref{eq:use-compression-agnostic-reg} gives that with probability $1-\delta$ over the draw of $S \sim P^n$ and $(\rho,\kappa) \sim Q'$,
  \begin{align}
    \err_P(\rho(\kappa(\Sigma(S)))) \leq &  \eerr_{\tilde S}(\rho(\kappa(\Sigma(S)))) + \discmg + \inf_{h \in \MH} \eerr_S(h) + C \sqrt{\frac{1}{n} (T \log(n) \cdot (m \log (1/\discmg) + \log n) + \log(1/\delta))}\nonumber\\
    \leq & 7\discmg +  \inf_{h \in \MH} \eerr_S(h) + C \sqrt{\frac{1}{n} \left(\frac{1}{\eta^2} \cdot \log(n/\delta) \log(n) \cdot (m \log (1/\discmg) + \log n) + \log(1/\delta)\right)}\nonumber.
  \end{align}
  McDiarmid's inequality yields that with probability $1-\delta$ over the choice of $S \sim P^n$, $\inf_{h \in \MH} \eerr_S(h) \leq \inf_{h \in \MH} \err_P(h) + C \sqrt{\log(1/\delta)/n}$, and combining this fact with the above display and the choice of $Q'$ gives that, with probability $1-\delta$ over the choice of $S \sim P^n$ and the execution of $\RegAgnostic$, its output hypothesis $H$ satisfies
  \begin{align}
\err_P(H) \leq & \inf_{h \in \MH} \err_P(h) + 7\discmg + C \sqrt{\frac{\fatdim^3 \cdot \log(1/\delta)}{n}} \cdot \poly \log(\fatdim, n, \log \log 1/\delta)\nonumber.
  \end{align}
  Infimizing over $\discmg \in (0,1)$ gives the claimed result.

  To analyze the oracle complexity of \RegAgnostic, we first note that the call to \SampleERMReal makes $O(n/\discmg) \leq O(n^2)$ calls to $\Oerm$ (\cref{lem:weak-to-sample-real}). The remainder of the analysis of oracle complexity follows the realizable case exactly. 
\end{proof}

\section{Lower bounds}
\label{sec:lb}
In this section, we present lower bounds for oracle-efficient PAC learning with a (strong) ERM oracle. \cref{sec:lb-mc} treats the setting of multiclass classification with bounded DS dimension, and \cref{sec:reg-lbs} treats the setting of realizable regression with bounded one-inclusion graph dimension. Notice that in both of these settings, uniform convergence does not hold (otherwise, a single ERM call on the i.i.d.~sample  would suffice).
\subsection{Lower bound for multiclass classification}
\label{sec:lb-mc}
A recent breakthrough result \cite{brukhim2022characterization} established that the \emph{DS dimension} of a multiclass concept class $\MH \subset [K]^\MX$ characterizes the sample complexity of PAC learning up to a polynomial factor, in both the realizable and agnostic settings. Our main result of this section, \cref{thm:mc-lb}, shows that even when the DS dimension is $1$ a strong ERM oracle is insufficient for PAC learnability.

We begin by introducing the DS dimension. For simplicity, we restrict to the setting that the number of classes $K$ is finite (as such will be the case in our lower bound). Given $n \in \BN$ and a nonempty subset  $\MS \subset [K]^n$, $\MS$ is called a \emph{$n$-dimensional pseudocube} if for each $y \in \MS$, there is some $y' \in \MS$ so that $y'_i \neq y_i$ and $y_j = y_j'$ for all $j \neq i$. The \emph{DS dimension} of $\MH \subset [K]^\MX$, denoted $\DSdim(\MH)$ \cite{daniely2014optimal}, is defined to be the largest positive integer $d$ so that, for some $X = (x_1, \ldots, x_d) \in \MX^d$, $\MH|_X$ contains a $d$-dimensional pseudocube.

Next, we introducing the concept class which will be used to prove \cref{thm:mc-lb}. Given positive integers $N, K \in \BN$, we set $\MX_N := [N]$ and let the label space be $\MY_{K} := [K]$.   For simplicity, we will often write $\MX = \MX_N= [N]$ and $\MY = \MY_{K}$ when the values of $N, K$ are clear.

Let $\MX^{\leq q} = [N]^{\leq q}$ denote the set of subsequences of $[N]$ of length at most $q$ (including the empty sequence). %
Note that $|\MX^{\leq q}| \leq 2qN^q$. 

For mappings $h^\st : \MX \to [K]$ and $\phi^\st : \MX^{\leq q} \to [K]$ 
and an element $z \in \MX^{\leq q}$, we define the \emph{merged hypothesis} $\merge(h^\st, \phi^\st, z) \in \MY^\MX$, as follows:
\begin{align}
  \merge(h^\st, \phi^\st, z)(x) := \begin{cases}
    \phi^\st(z) &: x \not\in z \\
    h^\st(x) &: x  \in z,
  \end{cases}\nonumber
\end{align}
where $x \in z$ denotes the event that $x$ is one of the elements of $z$.

Given mappings $h^\st : \MX \to [K]$ and $\phi^\st : \MX^{\leq q} \to [K]$, we define the class $\Hmc(h^\st, \phi^\st) \subset \MY^\MX$ as follows:
\begin{align}
\Hmc(h^\st, \phi^\st) = \{ h^\st \} \cup \{ \merge(h^\st, \phi^\st, z) \ : \ z \in \MX^{\leq q} \}\nonumber.
\end{align}

\begin{lemma}
  For any $h^\st : \MX \to [K]$ and $\phi^\st : \MX^{\leq q} \to [K]$ for which $\phi^\st$ is injective, %
  the DS dimension of $\Hmc(h^\st, \phi^\st)$ is bounded as $\DSdim(\Hmc(h^\st, \phi^\st)) \leq 1$. 
\end{lemma}
\begin{proof}
Suppose for the purpose of contradiction that there were distinct $x_1, x_2 \in \MX$ so that $\Hmc(h^\st, \phi^\st)|_{\{x_1, x_2 \}}$ contains a 2-dimensional pseudo-cube. Certainly this pseudo-cube must have some element $(h(x_1), h(x_2))$ (indexed by $h \in \Hmc(h^\st, \phi^\st)$) so that $h(x_1) \neq h^\st(x_1)$ and $h(x_2) \neq h^\st(x_2)$. (Indeed, take some element in this pseudo-cube, and for each $i \in \{1,2\}$ for which the $i$th coordinate is $h^\st(x_i)$, move to an $i$-neighbor.) But by definition of $\Hmc(h^\st, \phi^\st)$, we must have $h(x_1) = h(x_2) = \phi^\st(z) \not \in \{h^\st(x_1), h^\st(x_2) \}$ for some $z \in \MX^{\leq q}$, and thus, since  $\phi^\st$ is injective, $h = \merge(h^\st,\phi^\st, z)$, meaning that $x_1, x_2 \not\in z$. But it is also clear that $(h(x_1), h(x_2))$ cannot have any neighbor in $\Hmc(h^\st, \phi^\st)|_{\{x_1, x_2 \}}$: any neighbor $(h'(x_1), h'(x_2))$ must satisfy $h'(x_1) = h^\st(x_1)$ or $h'(x_2) = h^\st(x_2)$, but there is no function $h' \in \Hmc(h^\st)$ with $h'(x_i) = h^\st(x_i)$ and $h'(x_{3-i}) = \phi^\st(z)$ for some $i \in \{1,2\}$ (since $x_1, x_2 \not\in z$). 
\end{proof}

\begin{theorem}
  \label{thm:mc-lb}
  For any $q \in \BN$, there are domains $\MX, \MY$ satisfying $|\MY| \leq q^{O(q)}$ so that the following holds. 
  There is no algorithm $\Alg$ which satisfies the following guarantee: for any class $\MH \subset \MY^\MX$ with $\DSdim(\MH) = 1$ together with a strong ERM oracle $\Oerms$ for $\MH$, $\Alg$ is a $(\Oerms; 1/4, 1/4)$-PAC learner for $\MH$ with sample complexity and oracle complexity at most $q$. 
\end{theorem}
As a consequence of the bound $|\MY| \leq q^{O(q)}$ in \cref{thm:mc-lb} and the fact that $\Natdim(\MH) \leq \DSdim(\MH)$, we observe the following: even with a strong ERM oracle, there is no realizable PAC learning algorithm with sample complexity and ERM query cost $o(\log(K)/\log \log(K))$, even for classes with Natarajan dimension 1. This lower bound is nearly tight, as simply returning an ERM on the sample yields error at most $\ep$ with  $q = O(\Natdim(\MH) \cdot \log(K) / \ep)$ samples and overall query cost. \cite{daniely2014optimal}
We additionally note that in the hard instance used to prove \cref{thm:mc-lb}, the distribution $P$ has the additional property that its marginal over $\MX$ is uniform. 
\begin{proof}[Proof of \cref{thm:mc-lb}]
  Fix $q \in \BN$, and set $N = 16q, K = 96q^2 N^q$. We take $\MX = \MX_N = [N]$ and $\MY = \MY_K = [K]$, so that $|\MY| \leq q^{O(q)}$ holds. Let $\MH_0 := \MY^\MX = [K]^{[N]}$ denote the space of all mappings $h : \MX \to \MY$ and $\Phi := [K]^{[N]^{\leq q}}$ denote the space of all mappings $\phi : \MX^{\leq q} \to [K]$. Let $\SU := \Unif(\MH_0 \times \Phi)$: in particular, for a tuple $(h^\st, \phi^\st) \sim \SU$, the values $h^\st(x) \in \MY$ and $\phi^\st(z) \in \MY$ are all independent and uniform for all $x \in \MX, z \in [N]^{\leq q}$. Given $h^\st: \MX \to \MY$, $\phi^\st : \MX^{\leq q} \to \MY$, and a subset $\MX' \subset \MX$, let $\Oerms_{h^\st, \phi^\st, \MX'}$ denote an arbitrary strong ERM oracle for the class $\Hmc(h^\st, \phi^\st)$ satisfying the following condition: if there is an empirical risk minimizer of the sample passed to $\Oerms_{h^\st, \phi^\st, \MX'}$ of the form $\merge(h^\st, \phi^\st, z)$ for some $z \subset \MX'$, then the oracle returns $\merge(h^\st, \phi^\st, z)$. %
Moreover, let $P_{h^\st} \in \Delta(\MX \times \MY)$ denote the uniform distribution over tuples $(x, h^\st(x))$, for $x \in \MX$. Note that the distribution $P_{h^\st}$ is realizable with respect to the class $\Hmc(h^\st, \phi^\st)$, for any $\phi^\st \in \Phi$.
  
Let us consider the execution of $\Alg$ with the class $\Hmc(h^\st, \phi^\st)$ for a choice of $(h^\st, \phi^\st) \sim \SU$. Let $(x_1, y_1), \ldots, (x_q, y_q) \in \MX \times \MY$ denote the i.i.d.~realizable sequence sampled with respect to $P_{h^\st}$. Let $\MX' = \{ x_1, \ldots, x_q\}$ -- we will consider the interaction of $\Alg$ with the oracle $\Oerms_{h^\st, \phi^\st, \MX'}$.  For $1 \leq t \leq q$, let $(\hat x_t, \hat y_t) \in \MX \times \MY$ denote the $t$th tuple in $\MX \times \MY$ queried in the course of the oracle calls of $\Alg$. (In particular, we concatenate the tuples of each oracle call and let $(\hat x_t, \hat y_t)$ denote the $t$th element in this concatenated list. As such, the ERM oracle will, in general, only return a hypothesis after receiving certain examples $(\hat x_t, \hat y_t)$ corresponding to the last datapoint in each dataset passed to it.) Let $\SF_t$ denote the sigma-algebra generated by $(x_1, h^\st(x_1)), \ldots, (x_q, h^\st(x_q))$, $(\hat x_1, \hat y_1), \ldots, (\hat x_t, \hat y_t)$, the values of $\phi^\st(z)$ for $z \subset \{x_1, \ldots, x_q \}$, the internal randomness of $\Alg$, 
and the results of all oracle calls (to $\Oerms_{h^\st, \phi^\st,\MX'}$) which terminated at some step $s \leq t$. Let $\ME_t$ denote the event that all oracle calls terminating at some step $s \leq t$ return an element of $\merge(h^\st, \phi^\st, z)$ for some $z \subset \{ x_1, \ldots, x_q \}$. Since $\merge(h^\st, \phi^\st, z)$ is $\SF_t$-measurable for each $z \subset \{ x_1, \ldots, x_q \}$, the event $\ME_t$ is $\SF_t$-measurable.

Note that, conditioned on $\SF_t$, $(\hat x_{t+1}, \hat y_{t+1})$ is independent of $\One{\ME_t} \cdot  \One{\hat x_{t+1} \not \in \{x_1, \ldots, x_q \}} \cdot h^\st(\hat x_{t+1})$ and $\One{\ME_t} \cdot \phi^\st(z)$ for all $z \not \subset \{x_1, \ldots, x_q \}$: this holds because $(\hat x_{t+1}, \hat y_{t+1})$ is measurable with respect to $\SF_t$, and conditioned on any instantiation of the random variables generating $\SF_t$ for which $\ME_t$ occurs, the values of $\One{\ME_t} \cdot \phi^\st(z)$ for $z \not \subset \{x_1, \ldots, x_q \}$ and of $\One{\ME_t}\cdot h^\st(x)$ for $x \not \in \{x_1, \ldots, x_q \}$ are all independently and uniformly distributed in $[K]$. Thus, conditioned on $\SF_t$, with probability at least $1-(1 + 2qN^q)/K \geq 1-3qN^q/K$ (over $\SU$ and the execution of $\Alg$), we have that
\begin{align}
  \hat y_{t+1} \not \in \left(\{\One{\ME_t}\cdot \One{\hat x_{t+1} \not \in \{x_1, \ldots, x_q \}} \cdot h^\st(\hat x_{t+1}) \} \cup \{ \One{\ME_t} \cdot \phi^\st(z) \ : \ z \not\subset \{x_1, \ldots, x_q\}\} \right) \label{eq:tilyt-gt}.
\end{align}
Let the event that \cref{eq:tilyt-gt} holds be denoted $\MG_{t+1}$. %
Thus, under $\ME_t \cap \bigcap_{s \leq t+1} \MG_s$, for all $s \leq t+1$, all pairs $(\hat x_s, \hat y_s)$ for which $\hat y_s = h(\hat x_s)$ for some $h \in \Hmc(h^\st, \phi^\st)$ must satisfy either $\hat y_s \in \{ \phi^\st(z) \ : \ z \subset \{x_1, \ldots, x_q \}\}$ or $\hat x_s \in \{x_1, \ldots, x_q \}$. In particular, for any subset of the pairs $(\hat x_s, \hat y_s)$, there must be some empirical risk minimizer for this subset which belongs to $\{ \merge(h^\st, \phi^\st, z) \ : \ z \subset \{x_1, \ldots, x_q \} \}$. 
Thus, by definition of $\Oerms_{h^\st, \phi^\st, \MX'}$, under $\ME_t \cap \bigcap_{s \leq t+1} \MG_s$, the empirical risk minimizer returned by $\Oerms_{h^\st, \phi^\st, \MX'}$ at step $t+1$ (if any) must belong to $\{ \merge(h^\st, \phi^\st, z)  \ : \ z \subset \{x_1, \ldots, x_q \} \}$, i.e., $\ME_{t+1}$ holds. %

Since \cref{eq:tilyt-gt} gives that $\Pr_{\SU, \Alg}(\MG_t) \geq 1-3qN^q/K$, a union bound gives that $\Pr\left( \bigcap_{t \leq q} \MG_t \right) \geq 1-3q^2 N^q/K$. %
Since $\ME_0$ holds with probability 1 and $\ME_{t+1}$ holds under $\ME_t \cap \bigcap_{s \leq t+1} \MG_s$, we have that $\ME_q$ holds under the event $\bigcap_{t \leq q} \MG_t$, i.e., with probability at least $1-3q^2 N^q/K$. Write $\ME^\st := \ME_q \cap \bigcap_{s \leq t+1} \MG_s$.

Let $H : \MX \to \MY$ denote the output hypothesis of $\Alg$; note that $H$ is $\SF_q$-measurable. We argued above that for any $x \not \in \{x_1, \ldots, x_q \}$, conditioned on any instantiation of the random variables generating $\SF_q$ for which $\ME^\st$ (and thus $\ME_q$) occurs, $h^\st(x)$ is uniformly distributed in $[K]$. Thus, $\Pr_{\SU, \Alg}(H(x) = \One{\ME^\st} \cdot h^\st(x)) \leq 1/K$. Taking expectation over $(x, h^\st(x)) \sim P_{h^\st}$, we have that $\E_{\SU, \Alg} \E_{(x,y) \sim P_{h^\st}}[\One{H(x) = \One{\ME^\st} \cdot y}] \leq 1/K + q/N$. Thus, by Markov's inequality, there is some subset $\MJ \subset \MH_0 \times \Phi$ of measure $\SU(\MJ) \geq 3/4$ so that for all $(h^\st, \phi^\st) \in \MJ$, we have $\E_{\Alg} \E_{(x,y) \sim P_{h^\st}}[\One{H(x) = \One{\ME^\st} \cdot y}] \leq 2/K + 2q/N$. Since also $\ME^\st$ occurs with probability at least $1-3q^2 N^q/K\geq 31/32$  (over the choice of $(h^\st, \phi^\st) \sim \SU$ and the execution of $\Alg$), by Markov's inequality there is a subset $\MJ' \subset \MH_0 \times \Phi$ of measure $\SU(\MJ') \geq 3/4$ so that, for all $(h^\st, \phi^\st) \in \MJ'$, we have $\Pr_{\Alg}(\ME^\st) \geq 7/8$. 

Thus, for any $(h^\st, \phi^\st) \in \MJ \cap \MJ'$, we have $\Pr_{\Alg}(\ME^\st) \geq 7/8$ and $\E_{\Alg} \E_{(x,y) \sim P_{h^\st}}[\One{H(x) = \One{\ME^\st} \cdot y}] \leq 1/K + 2q/N$. Then
\begin{align}
\E_{\Alg} \E_{(x,y) \sim P_{h^\st}} [\One{H(x) = y}] \leq \E_{\Alg} \E_{(x,y) \sim P_{h^\st}}[(1-\One{\ME^\st}) + \One{H(x) = \One{\ME^\st} \cdot y}] \leq 1/8 + 2/K + 2q/N < 1/2\nonumber.
\end{align}
In particular, $\E_{\Alg} \E_{(x,y) \sim P_{h^\st}}[\lmc(H(x), y)] > 1/2$, which implies that $\Alg$ cannot be a $(\Oerms_{h^\st, \phi^\st, \MX'}; 1/4, 1/4)$-PAC learner for for $\Hmc(h^\st, \phi^\st)$. 
\end{proof}

\subsection{Lower bounds for realizable regression}
\label{sec:reg-lbs}
In this section, we prove a lower bound for oracle-efficient learning in the setting of realizable regression. A recent paper \cite{attias2023optimal} identified a combinatorial dimension depending on a scale parameter $\gamma$, called the \emph{one-inclusion graph dimension at scale $\gamma$} of a class $\MH \subset [0,1]^\MX$ (denoted $\OIGdim(\MH)$), so that finiteness of $\OIGdim(\MH)$ at all scales $\gamma$ characterizes real-valued learnability in the realizable setting. (In contrast, in the agnostic setting, recall that fat-shattering dimension is known to characterize learnability \cite{alon1997scale}.) Moreover, \cite{attias2023optimal} showed, roughly speaking, that the optimal sample complexity of PAC learning $\MH$ scales nearly linearly with $\OIGdim(\MH)$ for an appropriate choice of $\gamma$. In \cref{thm:reg-lb} below, we will show that classes whose one-inclusion graph dimension is constant at all scales may not be learnable with respect to a strong ERM oracle. To do so, we need to consider a slightly different notion of strong ERM oracle, since a minimizer of empirical risk may not exist if the class $\MH$ is infinite. For simplicity, we assume that $\MX$ is finite (as it will be in our lower bound).
\begin{definition}[Strong limiting ERM oracle]
  \label{def:slerm}
Consider a real-valued concept class $\MH \subset [0,1]^\MX$. A \emph{limiting strong ERM oracle} $\Oerm$ for $\MH$ is a mapping which takes as input a dataset $S \in (\MX \times [0,1])^n$ and outputs a concept $h : \MX \to [0,1]$ so that, for some sequence of hypotheses $h_n \in \MH$, $\| h_n - h \|_\infty \to 0$ as $n \to \infty$ and $\lim_{n \to \infty} \eerr_{S, \labs}(h_n) = \inf_{h \in \MH} \eerr_{S, \labs}(h)$. 
\end{definition}
Note that an $h$ as required in \cref{def:slerm} always exists since $\MX$ is finite, and thus $[0,1]^\MX$ is compact. 
\begin{theorem}
  \label{thm:reg-lb}
 Consider any $q \in \BN$ and algorithm $\Alg$ with access to a limiting strong ERM oracle $\Oerms$ for a real-valued concept class. Suppose that $\Alg$ has sample complexity and oracle complexity at most $q/200$. Then there is a set $\MX$ and a concept class $\MH \subset [0,1]^\MX$ which satisfies the following two conditions:
 \begin{enumerate}
  \item %
  $\Alg$ is not a $(\Oerms; 1/400, 1/400)$-PAC learner for $\MH$. 
  \item For any $\MH$-realizable distribution $P \in \Delta(\MX \times [0,1])$, there is a PAC learning algorithm that achieves 0 error with 1 sample with probability 1.
  \end{enumerate}
\end{theorem}
To avoid needing to formally define the one-inclusion graph dimension $\OIGdim(\cdot)$, we have not explicitly stated an upper bound on $\OIGdim(\MH)$ in the statement of \cref{thm:reg-lb}, instead directly stating its PAC learnability. Using the second item of the above theorem, the results of \cite{attias2023optimal} imply that the class $\MH$ satisfies $\OIGdim(\MH) = O(1)$ for all $\gamma \in (0,1)$; alternatively, it can be verified directly that $\OIGdim(\MH) = 1$ for all $\gamma$. 
\begin{proof}[Proof of \cref{thm:reg-lb}]
  For $n \in \BN$ and $q \in \BN$, write $N_{n,q} := n^q$, and let $p_n$ denote a prime number which is greater than $4n \cdot N_{n,q}$ and distinct from $p_{n-1}, \ldots, p_1$. We set $\MX = [q]$. For any sequence $\bar \sigma = (\sigma_1, \sigma_2, \ldots )$, where $\sigma_n : [N_{n,q}] \to [N_{n,q}]$ is a permutation, we define a function class $\MH_{\bar\sigma} \subset [0,1]^\MX$, as follows. 
  For $n \in \BN$ and $1 \leq i \leq N_{n,q}$, let $g_{n,i} : \MX \to \{0, 1/n, \ldots, (n-1)/n \}$ be the $i$th function (lexicographically) in $\{ 0, 1/n, \ldots, (n-1)/n \}^\MX$. Now define
  \begin{align}
h_{n,i, \bar \sigma}(x) := \frac{\lceil g_{n,i}(x) \cdot p_n \rceil }{p_n} + \frac{\sigma_n(i)}{p_n}\nonumber.
  \end{align}
  By our choice of $p_n$, we have that $\| h_{n,i}(x) - g_{n,i}\|_{\infty} \leq (\sigma(i) + 1)/p_n \leq 2N_{n,q} / (4n \cdot N_{n,q}) < 1/n$. We now define $\MH_{\bar \sigma} := \bigcup_{n \in \BN} \{ h_{n,i,\bar \sigma} \ : \ n \in \BN, i \in [N_{n,q}] \}$. Note that for any function $f : \MX \to [0,1]$ and $n \in \BN$, there is some $i \in [N_{n,q}]$ so that $\| f - g_{n,i} \|_{\infty} \leq 1/n$, meaning that $\| f - h_{n,i,\bar\sigma} \|_{\infty} \leq 2/n$. Hence, for any $\bar \sigma$, $\MH_{\bar \sigma}$ is dense in the space of functions $[0,1]^\MX$ (with respect to $\| \cdot \|_\infty$). Thus, we may choose the following limiting strong ERM oracle $\Oerms$ for the class $\MH_{\bar \sigma}$: given a sample $(x_1, y_1), \ldots, (x_n, y_n)$ $\Oerms$ returns the function which is 0 on all points $x \not \in \{x_1, \ldots, x_n \}$, and which maps each $x_i$ to $\mathrm{median}(\{ y_j \ : \ x_j = x_i \} )$.  Note that $\Oerms$ does not depend on $\bar\sigma$. 

  We next prove that for any $\bar \sigma$, the class $\MH_{\bar \sigma}$ is learnable with a single sample: indeed, note that for any $x \in \MX$, $n \in \BN$, and $i \in [N_{n,q}]$, there is no hypothesis $h' \in \MH_{\bar \sigma}$, $h' \neq h_{n,i,\bar \sigma}$, with $h'(x) = h_{n,i,\bar\sigma}(x)$, since $p_n$ is prime for each $n$. Thus, the learning algorithm which sees a sample $(x, h^\st(x))$, for any $x \in \MX$, can determine $h^\st$ from the value of $h^\st(x)$, and return $h^\st$.

  Finally, we lower bound the performance of $\Alg$, for some choice of $\bar \sigma$. Fix $n = 100$, and choose $\sigma_{n'}$ for $n' \neq n$ arbitrarily.  Moreover, we let $\sigma_n : [N_{n,q}] \to [N_{n,q}]$ be a uniformly random permutation, $i^\st \sim \Unif([N_{n,q}])$, and $h^\st = h_{n, i^\st, \bar \sigma}$. We let $P$ be the distribution $\Unif( \{ (x, h^\st(x))  \ : \ x \in \MX \})$.

  We consider the performance of $\Alg$ in expectation over the distribution of $\bar \sigma$ and $h^\st$ that we have defined.
  Let $S = \{(x_1, y_1), \ldots, (x_{q/200}, y_{q/200})\}$ denote the i.i.d.~sample from $P$ that $\Alg$ receives. Note that, conditioned on $S$, the value of $i^\st$ is uniformly random over the set $\MI_S := \{ i \in [N_{n,q}] \ : \ g_{n,i,\bar\sigma}(x_j) = g_{n,i^\st}(x_j) \ \forall j \in [q/200] \}$. This uses the fact that the value of $\sigma_n^{-1}(i^\st) \in [N_{n,q}]$ is uniform and independent of $i^\st$. Since the responses to the queries to $\Oerms$ do not depend on $\bar\sigma$, the output hypothesis $H$ of $\Alg$ and $i^\st$ are conditionally independent, conditioned on $S$. Thus, conditioned on $H$, $i^\st$ is uniformly random over the set $\MI_S$, which in particular means that the function $g_{n,i^\st}$ is distributed uniformly among all functions $g \in \{0, 1/n, \ldots, (n-1)/n\}^\MX$ satisfying $g(x_j) = g_{n,i^\st}(x_j)$ for $j \in [q/200]$. Hence for any $x \not \in \{x_1, \ldots, x_{q/200} \}$, $\E_{i^\st, \bar \sigma}[|g_{n, i^\st}(x) - H(x) | \mid H ] \geq 1/10$, which yields $\E_{i^\st, \bar \sigma}[|h_{n,i^\st,\sigma}(x) - H(x) | \mid H] \geq 1/10 - 2/n > 1/100$. Averaging over the distribution of $H$ and of $x \sim \Unif(\MX)$ and using that $h^\st = h_{n, i^\st, \bar \sigma}$, we see that
  \begin{align}
\E_{i^\st, \bar \sigma} \E_{\Alg} \E_{(x,y) \sim P_{h^\st}}[| h^\st(x) -  H(x)|] > 1/100 - (q/200)/q = 1/200\nonumber.
  \end{align}
Thus, there is some $\bar \sigma$ and $i^\st \in [N_{n,q}]$ so that, letting $h^\st = h_{n,i^\st}$, $\E_{\Alg} \E_{(x,y) \sim P_{h^\st}}[\labs(H(x), y)] > 1/200$. Hence $\Alg$ cannot be a $(\Oerms; 1/400, 1/400)$-PAC learner for $\MH_{\bar \sigma}$, for some choice of $\bar\sigma$. %
\end{proof}

\section{Computational separation between weak and strong consistency oracle}
\label{sec:oracle-separation}
We let \factoring denote the following computational problem: given a positive integer $n$, represented in binary, in the event that $n  =pq$ for primes $p,q$, then output $p$ and $q$; otherwise, the output can be arbitrary.\footnote{If one wishes to define a total problem, one can require that the output be $0$ if $n$ is not the product of 2 primes.} It is widely believed that there is no polynomial-time algorithm for \factoring \cite{lenstra2011integer}.

Below we define a hypothesis class $\Hprime$ for which a weak consistency oracle $\Ocon$ (\cref{def:weak-con}) can be implemented in polynomial time, yet there is no polynomial-time algorithm which can implement a strong ERM oracle $\Oerms$ (\cref{def:strong-erm}), assuming that there is no polynomial-time algorithm for \factoring. In fact, our proof shows that it is computationally hard to implement a strong ERM oracle even if it must only succeed on realizable datasets (such an oracle corresponds to the standard notion of consistency oracle, which returns a hypothesis consistent with the input dataset, if one exists).

We set $\MX := \BN$ and define the class $\Hprime \subset \{0,1\}^\MX$, as follows. For each pair of primes $p,q$, we define
\begin{align}
  h_{p,q}(x) := \begin{cases}
    1 &: x \in \{p, q, pq \} \\
    0 &: \mbox{otherwise}. 
  \end{cases}\nonumber
\end{align}
Also write, for each $n \in \BN$, $g_n(x) := \One{x=n}$. We then set $\Hprime := \{ h_{p,q} \ : \ p,q \mbox{ are prime} \} \cup \{ g_n \ : \ n \in \BN \mbox{ is not a product of two primes}\}$. Note that $\VCdim(\Hprime) \leq 3$. 
\begin{proposition}
  \label{prop:hprimes}
  The class $\Hprime$ satisfies the following:
  \begin{enumerate}
  \item A weak consistency oracle $\Ocon$ for $\Hprime$ can be implemented in polynomial time.
  \item \factoring reduces to the problem of implementing a strong ERM oracle $\Oerms$ for $\Hprime$.
  \end{enumerate}
\end{proposition}
\begin{proof}
  We first describe how a weak consistency oracle can be efficiently implemented. Given a dataset $S := \{ (x_i, y_i) \}_{i \in [n]}$, we perform the following steps:
  \begin{itemize}
  \item Set $T := \{ x_i \ : \ y_i = 1 \}$.
  \item If $|T| = 0$, return $\True$.
  \item If $|T| = 1$, let $n$ denote the unique element of $T$. For each point in $S$ of the form $(x_i, 0)$, check if $n$ is a multiple of $x_i$. If so, and the ratio $n/x_i$ is prime, then return $\False$. At the end of the loop, if we have not yet returned, then return $\True$. 
  \item If at least two values in $T$ are composite, return $\False$.
  \item If at least three values in $T$ are prime, return $\False$.
  \item If $T$ consists of two prime values whose product is $m$, then output $\True$ if and only if the point $(m,0)$ does not appear in $S$.
  \item Otherwise, $T$ consists of a composite value $m$ and at least one prime value $p$; then output $\True$ if and only the point $(m/p, 0)$ does not appear in $S$. 
  \end{itemize}
  The above steps may be efficiently implemented since there is a polynomial-time algorithm for determining whether a given natural number is prime \cite{agrawal2004primes}.

  Now consider a strong ERM oracle $\Oerms$ for $\Hprime$. Given a positive integer $n$, consider the hypothesis $\hat h$ returned by $\Oerms( \{ (n,1)\})$. If $n$ is a product of two primes $p,q$ then the unique empirical risk minimizer is the hypothesis $h_{p,q}$, meaning that $\Oerms$ must return this hypothesis and thereby yield the prime factors $p,q$. %
\end{proof}

\end{document}